\documentclass{article}
\pdfoutput=1

\PassOptionsToPackage{sort&compress,numbers}{natbib}
\usepackage[preprint]{neurips_2021}
\makeatletter
\renewcommand{\@noticestring}{}
\makeatother

\usepackage[utf8]{inputenc} 
\usepackage[T1]{fontenc}    
\usepackage[unicode]{hyperref}       
\usepackage{url}            
\usepackage{booktabs}       
\usepackage{amsfonts}       
\usepackage{nicefrac}       
\usepackage{microtype}      
\usepackage{xcolor}         
\usepackage{amsmath,amsfonts,amssymb,amsthm}

\usepackage{todonotes}
\usepackage{IEEEtrantools}
\usepackage[linesnumbered,ruled]{algorithm2e}

\usepackage{bbm}
\usepackage{enumitem}
\usepackage{mathtools}
\usepackage{epstopdf,wrapfig}

\author{%
  Prateek Jain \\
  Google AI Research Lab,\\
  Bengaluru, India 560016 \\
  \texttt{prajain@google.com} \\
    \And
  Suhas S Kowshik \\
  Department of EECS\\
  MIT,\\
  Cambridge, MA 02139 \\
  \texttt{suhask@mit.edu} \\
  \AND
  Dheeraj Nagaraj \\
  Department of EECS\\
  MIT,\\
  Cambridge, MA 02139 \\
  \texttt{dheeraj@mit.edu} \\
  \And
  Praneeth Netrapalli \\
  Google AI Research Lab,\\
  Bengaluru, India 560016 \\
  \texttt{pnetrapalli@google.com} \\
  }

\newcommand{\distas}[1]{\mathbin{\overset{#1}{\sim}}}%
\newsavebox{\mybox}\newsavebox{\mysim}
\newcommand{\distras}[1]{%
  \savebox{\mybox}{\hbox{\kern3pt$\scriptstyle#1$\kern3pt}}%
  \savebox{\mysim}{\hbox{$\sim$}}%
  \mathbin{\overset{#1}{\kern\z@\resizebox{\wd\mybox}{\ht\mysim}{$\sim$}}}%
}
\newtheorem{theorem}{Theorem}
\newtheorem{lemma}[theorem]{Lemma}
\newtheorem{claim}{Claim}
\newtheorem{define}{Definition}
\newtheorem{proposition}{Proposition}
\newtheorem{corollary}{Corollary}
\newtheorem{assumption}{Assumption}
\newtheorem*{remark}{Remark}

\newcommand{\norm}[1]{\left\| #1 \right\|}
\newcommand{\norms}[1]{\| #1 \|}
\newcommand{\lmin}[1]{\sigma_{\min}(#1)}
\newcommand{\lmax}[1]{\sigma_{\max}(#1)}
\newcommand{\Pb}[1]{\mathbb{P}\left[#1 \right]}

\newcommand{\Ex}[1]{\mathbb{E}\left[#1 \right]}

\newcommand{\At}[2]{\mbox{${A^{#1}_{#2}}$}}
\newcommand{\Av}[2]{\mbox{${A^{#1,v}_{#2}}$}}

\newcommand{\Att}[2]{\mbox{${\tilde{A}^{#1}_{#2}}$}}
\newcommand{\Avt}[2]{\mbox{${\tilde{A}^{#1,v}_{#2}}$}}
\newcommand{\Abt}[2]{\mbox{${\tilde{A}^{#1,b}_{#2}}$}}
\newcommand{\Atdiff}[2]{\left(\mbox{${A^{#1,v}_{#2}}$}-\A\right)}
\newcommand{\Attdiff}[2]{\left(\mbox{${\tilde{A}^{#1,v}_{#2}}$}\right)}

\newcommand{\Atto}[2]{\left(\mbox{${\tilde{A}^{#1,b}_{#2}}$}-\A\right)}

\newcommand{\Anoa}[2]{\left(\mbox{${\hat{A}^{b}_{#1,#2}}$}-\A\right)}
\newcommand{\Anoat}[2]{\left(\mbox{${\hat{\tilde A}^{b}_{#1,#2}}$}-\A\right)}
\newcommand{\Anva}[2]{\left(\mbox{${\hat{A}^{v}_{#1,#2}}$}\right)}
\newcommand{\Anvat}[2]{\left(\mbox{${\hat{\tilde A}^{v}_{#1,#2}}$}\right)}
\def\Ana{\hat A_{a,N} }
\def\Anav{\hat A^v_{a,N} }
\def\Anvt{\hat {\tilde A}^v_{N} }
\def\Anab{\hat A^b_{a,N} }
\def\Anat{\hat{\tilde {A}}_{a,N} }
\def\Anavt{\hat{\tilde {A}}^v_{a,N} }
\def\Anovt{\hat{\tilde {A}}^v_{0,N} }
\def\Anoavt{\hat{\tilde {A}}^v_{0,a} }
\def\Anabt{\hat{\tilde {A}}^b_{a,N} }

\newcommand{\Vt}[1]{\tilde V_{#1}}

\newcommand{\kl}[2]{\mathsf{KL}( #1 \| #2 )}
\newcommand{\Xt}[2]{\mbox{${X^{#1}_{#2}}$}}
\newcommand{\Xttr}[2]{\mbox{${X^{#1,\top}_{#2}}$}}
\newcommand{\Xtt}[2]{\mbox{${\tilde{X}}^{#1}_{#2}$}}
\newcommand{\Xtttr}[2]{\mbox{${\tilde{X}}^{#1,\top}_{#2}$}}

\newcommand{\Nt}[2]{\mbox{${\eta^{#1}_{#2}}$}}

\newcommand{\Htt}[3]{\mbox{${\tilde{H}^{#1}_{#2,#3}}$}}
\newcommand{\Htttr}[3]{\mbox{${\tilde{H}^{#1,\top}_{#2,#3}}$}}
\newcommand{\Hth}[3]{\mbox{${\hat{H}^{#1}_{#2,#3}}$}}
\newcommand{\Hthtr}[3]{\mbox{${\hat{H}^{#1,\top}_{#2,#3}}$}}

\newcommand{\prodHtt}[2]{\left(\prod_{s=#2}^{1}\Htt{#1-s}{0}{B-1}\right)}
\newcommand{\prodHtttr}[2]{\left(\prod_{s=1}^{#2}\Htttr{#1-s}{0}{B-1}\right)}
\newcommand{\prodHth}[2]{\left(\prod_{s=#2}^{1}\Hth{#1-s}{0}{B-1}\right)}
\newcommand{\prodHthtr}[2]{\left(\prod_{s=1}^{#2}\Hthtr{#1-s}{0}{B-1}\right)}

\newcommand{\Ptt}[2]{\left(I-2\gamma \Xtt{#1}{#2}\Xtttr{#1}{#2}\right)}

\newcommand{\Ppt}[2]{\mbox{${\tilde{P}^{#1}_{#2}}$}}

\newcommand{\var}{\mathsf{VAR}}
\newcommand{\sgdber}{\mathsf{SGD}-\mathsf{RER}}
\newcommand{\sgder}{\mathsf{SGD}-\mathsf{ER}}
\newcommand{\sgd}{\mathsf{SGD}}
\newcommand{\ols}{\mathsf{OLS}}
\newcommand{\gram}[1]{#1^{\top} #1}

\newcommand{\lossop}{\mathcal{L}_{\mathsf{op}}}
\newcommand{\losspred}{\mathcal{L}_{\mathsf{pred}}}
\newcommand{\lossmm}{\mathcal{L}_{\mathsf{minmax}}}
\newcommand{\ahat}{\hat{A}}
\newcommand{\sym}[1]{\operatorname{Sym}\left(#1\right)}

\def\ch{\mathcal{H}}
\def\cc{\mathcal{C}}
\def\cd{\mathcal{D}}

\def\cl{\mathcal{L}}
\def\clt{\tilde{\cl}}
\newcommand{\ind}[2]{1\left[\cd^{#1,#2}\right]}
\newcommand{\indc}[2]{1\left[\cd^{#1,#2,C}\right]}
\def\cdt{\tilde\cd}
\def\cct{\tilde \cc}

\newcommand{\indt}[2]{1\left[\tilde \cd^{#1,#2}\right]}
\newcommand{\indtc}[2]{1\left[\tilde \cd^{#1,#2,C}\right]}
\def\cdh{\hat\cd}
\newcommand{\indh}[2]{1\left[\cdh^{#1,#2}\right]}

\def\nn{\mathcal{N}}
\def\nx{\norm{X}}
\def\nxx{{\left(2\norm{X}\right)}}
\def\A{{{A^*}}}
\def\Aa{A^*}
\def\a{a^*}
\def\atr{a^{*,\top}}
\def\gammah{\hat{\gamma}}
\DeclareMathOperator{\tr}{Tr}
\DeclareMathOperator{\cro}{Cr}
\DeclareMathOperator{\poly}{Poly}
\def\crot{\widetilde {\cro}}
\DeclareMathOperator{\dg}{Dg}
\def\dgt{\widetilde{\dg}}
\def\prbnd{\frac{1}{T^{\alpha}}}
\def\prbndsq{\frac{1}{T^{\alpha/2}}}

\def\prbndsqinv{T^{\alpha/2}}

\def\Ieee{IEEEeqnarray*}
\def\Ieeen{\IEEEyesnumber}

\title{Streaming Linear System Identification with Reverse Experience Replay}

\begin{document}

\maketitle

\begin{abstract}%
We consider the problem of estimating a linear time-invariant (LTI) dynamical system from a single trajectory via streaming algorithms, which is encountered in several applications including reinforcement learning (RL) and time-series analysis. 
While the LTI system estimation problem is well-studied in the {\em offline} setting, the practically important streaming/online setting has received little attention. Standard streaming methods like stochastic gradient descent (SGD) are unlikely to work since streaming points can be highly correlated. 
In this work, we propose a novel streaming algorithm, SGD with Reverse Experience Replay ($\sgdber$), that is inspired by the experience replay (ER)  technique popular in the RL literature. $\sgdber$ divides data into small buffers and runs SGD backwards on the data stored in the individual buffers. We show that this algorithm exactly deconstructs the dependency structure and obtains information theoretically optimal guarantees for both parameter error and prediction error. Thus, we provide the first -- to the best of our knowledge -- optimal SGD-style algorithm for the classical problem of linear system identification with a first order oracle. 
Furthermore, $\sgdber$ can be applied to more general settings like sparse LTI identification with known sparsity pattern, and  non-linear dynamical systems. Our work demonstrates that the knowledge of data dependency structure can aid us in designing statistically and computationally efficient algorithms which can ``decorrelate'' streaming samples. 
\end{abstract}


\section{Introduction}
In this paper, we study the problem of learning linear-time invariant (LTI) systems, where the goal is to estimate the matrix $\A\in \mathbb{R}^{d\times d}$ from the given samples $(X_0,\dots,X_T)$ that obey: 
\begin{equation} \label{eq:state_evolution}
	X_{\tau+1} = \A X_\tau + \eta_\tau,\ \ \ X_\tau\in \mathbb{R}^d,\ \ \ \eta_\tau \stackrel{i.i.d.}{\sim} \mu,
\end{equation}
where $\mu$ is an unbiased noise distribution. 
The problem is central in control theory and reinforcement learning (RL) literature \citep{kumar2015stochastic, accikmecse2013lossless}. It is also equivalent to estimating Vector Autoregressive (VAR) model popular in the time-series analysis literature \citep{hamilton2020time}, where it has been used in several applications like finding gene regulatory information network \citep{gene}. 

Despite a long line of classical literature for the problem, most of the existing results focus on the {\em offline} setting, where all the samples $(X_0,\dots,X_T)$ are available apriori. In this setting,  ordinary least squares (OLS) method that estimates $A$ as,   $\hat A=\arg\min_A\sum_{\tau=0}^{T-1}\|X_{\tau+1}-AX_\tau\|^2$ is known to be nearly optimal \citep{simchowitz2018learning, sarkar2019near}. However, such offline solutions do not apply to the streaming setting -- where $\A$ needs to be estimated online -- that has applications in several domains like RL, large-scale forecasting systems, recommendation systems \citep{hanck2019introduction, zheng2016neural}.


In this paper, we study the above mentioned problem of learning LTI systems via first order gradient oracle with streaming data. The goal is to design an estimator that provides accurate estimation while ensuring nearly optimal time complexity and space complexity that is nearly {\em independent} of $T$. Note that due to specific form arising in linear regression,  the optimal solution to OLS can be estimated in online fashion using Sherman-Morrison-Woodbury formula. But such a solution is limited and does not apply to practically important settings like  {\em generalized  non-linear dynamical} system or when $\A$ is high-dimensional and has special structure like low-rank or sparsity \citep{basu2015regularized,Basu_2019}. 

So, in this work, we focus on designing  Stochastic Gradient Descent (SGD) style methods that can work directly with first order gradient oracle, and hence is more widely applicable to the settings mentioned above. 
In fact, after the first appearance of this manuscript, the algorithm ($\sgdber$)  and the techniques introduced in this paper were used to obtain near-optimal guarantees for learning certain classes of {\em non-linear dynamical systems} \cite{jain2021nonlinear} as well as in Q-learning tabular MDPs in RL \citep{agarwal2021online}. We note that prior to \cite{jain2021nonlinear}, even optimal \emph{offline} algorithms were unknown for such non-linear systems.

SGD is a popular method for general streaming settings, and has been shown to be {\em optimal} for problems like streaming linear regression \citep{jain2017parallelizing}. 
However, when the data has temporal dependencies, as in the estimation of linear dynamical systems, such a naive implementation of SGD may not perform well as observed in \citep{nagaraj2020least,gyorfi1996averaged}.
 In fact, for linear system identification, our experiments suggest that SGD suffers from a non-zero bias (Section~\ref{sec:experiments}). In order to address temporal dependencies in data, practitioners use a heuristic called \emph{experience replay},
which maintains a {\em buffer} of points, and samples points {\em randomly} from the buffer. However, for linear system identification, experience replay does not seem to provide an accurate unbiased estimator for reasonable buffer sizes (see Section~\ref{sec:experiments}).

In this work, we propose {\em reverse experience replay} for linear system identification. Our method maintains a small {\em buffer} of points, but instead of random ordering, we replay the points in a {\em reverse} order. We show that this algorithm exactly unravels the temporal correlations to obtain a consistent estimator for $\A$. Similar to the standard linear regression problem with {\em i.i.d.} samples, we can break the error in two parts: a) bias: that depends on the initial error $\|A_0-\A\|$, b) variance: the steady state error due to noise $\eta$. We show that our proposed method, under fairly standard assumptions and with a small buffer size, is able to decrease the bias at fast rate, while the variance error is nearly optimal (see Theorem~\ref{thm:op_informal}), matching the information theoretic lower bounds \citep[Theorem 2.3]{simchowitz2018learning}. 
To the best of our knowledge, we provide first non-trivial analysis for a purely streaming SGD-style algorithm with optimal computation complexity and nearly bounded space complexity that is dependent logarithmically on $T$. We note here that the idea of reverse experience replay was independently discovered in experimental reinforcement learning by \citep{rotinov2019reverse} based on reverse replay observed in Hippocampal place cells \citep{ambrose2016reverse} in Neurobiology. We also refer to \citep{whelan2021robotic} for more on this connection.  

In addition to the transition matrix estimation error $\|A-\A\|$, we also provide analysis of prediction error, i.e., $E[\|AX-\A X\|^2]$ (see Theorem~\ref{thm:pred_informal}). Here again, we bound the {\em bias} and the {\em variance} part of the error separately.  We further derive new lower bounds for prediction error (see Theorem~\ref{thm:main_lower_bound}) and show that our algorithm is minimax optimal, under standard assumptions on the model. As mentioned earlier, our method work with general first order oracles, hence applies to more general problems like {\em sparse LTI estimation} with known sparsity structure and unlike online OLS methods, $\sgdber$ has nearly optimal time complexity. Finally, we also provide empirical validation of our method on simulated data, and demonstrate that the proposed method is indeed able to provide error rate similar to the OLS method while methods like SGD and standard experience replay, lead to biased estimates.

\paragraph{Related Work.}  Due to applications in RL, recently LTI system identification has been widely studied. In particular,  \citep{oymak2019non} studied the problem in offline setting under the ``stability" condition, i.e., the spectral radius ($\rho(\A)$) of $\A$   is a constant bounded away from $1$. 
The sequence of papers \citep{sarkar2019near,faradonbeh2018finite,simchowitz2018learning,jedra2020finite} provide optimal analyses of the offline OLS estimator beyond assumptions of stability. That is, they show that OLS recovers $\A$ near optimally even the process defined by ~\eqref{eq:state_evolution} is stable but does not mix within time $T$ (when $\rho(\A)$ is $1-O(1/T)$) or is unstable (when $\rho(\A)$ is larger than $1$). Further \cite{simchowitz2018learning,jedra2019sample} provide information theoretic lower bounds for the LTI system identification problem.   \cite{sattar2020non,foster2020learning,jain2021nonlinear} consider the problem of identifying non-linear dynamical systems of the form $X_{t+1} = \phi (\A X_t) + \eta_t$ where $\phi$ is a one dimensional link function which acts co-ordinate wise. In this setting, however, there is no closed for expressions for the estimator of $\A$.  \cite{sattar2020non,foster2020learning} give offline algorithms whose error guarantees are worse off by factors of mixing time whereas  \cite{jain2021nonlinear} obtains near optimal offline and streaming algorithms for this setting. In fact, \cite{jain2021nonlinear} uses $\sgdber$ which was first introduced in this work in order to obtain the streaming algorithm.

LTI identification problem has been studied in time series forecasting literature as well. For example, \citep{lai1983asymptotic} obtains asymptotic consistency results for system identification problem and  \citep{campi2002finite,vidyasagar2006learning} consider the problem of finite time recovery. Both consider a certain parameterized predictor for a linear system with empirical risk minimization for the parameter and analyzes the deviation from population risk. Similarly, \citep{kuznetsov2018theory} also studies generalization error guarantees. In contrast, our work is able to provide precise bias and variance (similar to generalization error) of the estimator in the streaming setting, and show that the asymptotic error is minimax optimal. 

\citep{hardt2018gradient} studied SISO systems with observations $(x_\tau,y_\tau) \in \mathbb{R}^2$ and a hidden state $h_\tau$ which is high dimensional, thus their model and applications are significantly different than the LTI system we study. For the SISO system, \citep{hardt2018gradient} analyzes SGD to provide error bounds contain (a large) polynomial in the hidden state dimension. Here, the hidden state has an evolution similar to Equation~\ref{eq:state_evolution} whereas $x_1,\dots,x_T$ are drawn i.i.d from some distribution. 

System identification has been studied in the context of partially observed LTI systems as well. Recent works \citep{oymak2019non,tsiamis2019finite,sarkar2021finite,lale2020logarithmic, lee2020improved,lee2020non} focus on identifying a certain Hankel-like matrix of the system. These are not directly comparable to the fully observed setting in this work since the model parameters are identifiable only upto a similarity transformation in the partially observed setting. 

Recently, there has been an exciting line of work in the related domain of online control (see  \citep{cohen2018online,agarwal2019online,hazan2020nonstochastic,chen2020black} and references therein). The state equation studied in these papers also contain an additive term of $B u_\tau$ for some unknown matrix $B$ and a control signal $u_\tau$ and the noise $\eta_\tau$ is either stochastic (as in \citep{cohen2018online}) or adversarial (as in \citep{agarwal2019online,hazan2020nonstochastic,chen2020black}). The goal is to output control signals $u_\tau$ after observing $X_1,\dots,X_\tau$, such that the cost $\sum_\tau c_\tau(X_\tau,u_\tau)$ is minimized  for some sequence of convex costs $c_\tau$. We focus on the LTI system identification(or estimation) problem  while the goal of the above mentioned line of work is to design an online controller. 

We also note here another line of works \citep{ghai2020no,rashidinejad2020slip,kozdoba2019line,hazan2017learning, tsiamis2020sample,tsiamis2020online,lale2020logarithmic} focused on online prediction of both fully observed and partially observed LTI systems, and the similar problem of time series forecasting by regret minimization \citep{kuznetsov2016time,kuznetsov2018theory}.  In particular, the main goal there is to design online prediction algorithms minimizing regret against a certain class (for instance, against a Kalman filter with knowledge of the system parameters in the case of partially observed LTI systems). The situation considered in our work is different in atleast two aspects: 1) we focus significantly on parameter recovery or system identification and 2) our notion of prediction is \emph{prediction at stationarity} which can be thought of as one-step regret (compared to $T$--step regret for instance in \citep{ghai2020no,rashidinejad2020slip}).  

Finally, \citep{basu2015regularized} considers {\em offline} sparse linear regression with $\ell_1$ penalty where the feature vector is derived from an auto regressive model. Similarly, \citep{nagaraj2020least} considers the problem of linear regression where the feature vectors come from a Markov chain. This line of work is different from ours in that we try to estimate the parameters of the Markov process itself. 

\paragraph{Paper Organization.} 
We provide the problem definition and introduce the notations in the next section. We then present our algorithm and the key intuition behind it in Section~\ref{sec:alg}. We then present our main result in Section~\ref{sec:main_results} and provide a proof sketch in Section~\ref{sec:proofsketch}. Finally, we present simulation results in Section~\ref{sec:experiments}. 

\section{Problem Setting and Notation}\label{sec:prob}
In this section, we first introduce the data generation model, the required assumptions and then provide the precision problem definition. Throughout the paper, we use $\|A\|$ to denote the operator norm of $A$ unless otherwise specified. $\|A\|_F$ denotes the Frobenius norm of $A$. $\sigma_i(A)$ denotes the $i$-th largest singular value of $A$, i.e., $\sigma_{\max}(A)=\sigma_1(A)$. $\kappa(A):=\sigma_{\max}(A)/\sigma_{\min}(A)$ denotes the condition number of $A$. $\rho(A)$ denotes the spectral radius of $A$. For two symmetric matrices $A,B\in\mathbb{R}^{d\times d}$ we say $A\preceq B$ if $B-A$ is positive semidefinite (psd). For notational simplicity, we use $C$ to denote a constant, and it's value can be different in different equations.
\paragraph{Linear Dynamical System/VAR(1) model.} 
 Given an initial (possibly random) data point $X_0$ which is independent of the noise sequence, we generate the  $(X_0,\dots,X_T)$ from the $\var$ model as:
\begin{\Ieee}{LLL}
\label{eq:var1}
X_{\tau+1}=\A X_{\tau} +\eta_{\tau},\,\ \ 0\leq \tau\leq T-1\Ieeen, 
\end{\Ieee}
where $\A\in\mathbb{R}^{d\times d}$ be the transition matrix. Let $\eta_1,\dots,\eta_T \in \mathbb{R}^d$ be an i.i.d noise sequence with $0$ mean and finite second moment with probability measure $\mu$. 
We will denote this model by $\var(\A,\mu)$. We also make the following assumptions about $\A$, $\mu$, and $X_0$: 
\begin{assumption}\label{as:norm_condition}
	\textbf{External Stability.}
$\|\A\| < 1$
 \end{assumption}
\begin{assumption}\label{as:noise_concentration}
	\textbf{Sub-Gaussian Noise.} $\mu$ has co-variance $\Sigma$ and for all $x\in \mathbb{R}^d$,  $\langle x,\eta_{\tau}\rangle $ is    $C_{\mu}\langle x , \Sigma \cdot x\rangle$  sub-Gaussian. Further, $\Sigma$ is full rank. Also, let $\mu_4\coloneqq\Ex{\norm{\eta_{\tau}}^4}$ be the fourth moment of the noise. 
\end{assumption}
\begin{assumption}\label{as:stationarity}
	\textbf{Stationarity.} 
	$X_0 \sim \pi$, the stationary distribution corresponding to $(\A,\mu)$. Let $M_4\coloneqq \Ex{\norm{X_0}^4}$.
\end{assumption}


Due to Assumption~\ref{as:norm_condition}, we can show that the law of  the iterate $X_T$ from the $\var$ model defined above converges to a stationary distribution $\pi$ as $T \to \infty$ for arbitrary choice of $X_0$ and has a mixing time of the order $\tau_{\mathsf{mix}} = O\left(\tfrac{1}{1-\|\A\|}\right)$. For simplicity, we will absorb $C_{\mu}$ into other constants. Finally, we will use $(Z_0,\dots,Z_T) \sim \var(\A,\mu)$ to mean that $Z_0,\dots,Z_T$ is a stationary sequence corresponding to $\var(\A,\mu)$. We also note that the covariance matrix under stationarity, $G:= \mathbb{E}_{X\sim \pi}XX^{\top}=\sum_{s=0}^{\infty}\A^s \Sigma(\A^{\top})^s \succeq \Sigma$.

\begin{remark}
It is indeed possible to replace Assumption~\ref{as:norm_condition} with the weaker condition on the spectral radius of $\A$: $\rho(\A) < 1$. While our results still hold in this case, the bound might have additional condition number factors. See Section~\ref{subsec:spectral_gap} for more details.
\end{remark}
\begin{remark}
The full rank assumption on $\Sigma$ is needed for polynomial sample complexity \citep{tsiamis2021linear}.
\end{remark}


\paragraph{Problem Statement.} 
Let $(X_0, X_1, \cdots, X_T)$ be sampled from $\var(\A, \mu)$ model for a fixed horizon $T$. Then, the goal is to design and analyze an online algorithm that uses only first order gradient oracle to estimate the system matrix $\A$. That is, at each time-step $\tau$, we obtain gradient for the transition $(X_\tau,X_{\tau+1})$ and output estimate $A_\tau$. The goal is to ensure that each $A_{\tau}$ has small estimation error wrt $\A$; naturally, we would expect better estimation error with increasing $\tau$. We quantify estimation error using the following two loss functions: \vspace*{-5pt}
\begin{enumerate}[leftmargin=*]
\item Parameter error: $\lossop(A;\A,\mu) = \|A-\A\|$
\item Prediction error at stationarity: $\losspred(A;\A,\mu) := \mathbb{E}_{X_{\tau} \sim \pi}\|X_{\tau+1}-AX_{\tau}\|^2$
\end{enumerate}
Note that the problem is equivalent to $d$ linear regression problems, but with {\em dependent} samples, making it significantly more challenging. Whenever Assumption~\ref{as:norm_condition} holds, stationary distribution $\pi$ exists, so the prediction error $\losspred$ is meaningful. Furthermore: $\losspred(A) - \losspred(\A) = \tr\left[(A-\A)^{\top}(A-\A)G\right]$ where $G := \mathbb{E}_{X\sim\pi}XX^{\top}$. 


\section{Algorithm}\label{sec:alg}
As mentioned in related works, the standard OLS estimator that minimizes the empirical loss is known to be nearly optimal in the {\em offline setting} \citep{simchowitz2018learning}:
\begin{equation}
\label{eq:OLS}
\hat{A}_{OLS}=\arg\min_A \sum_{\tau=0}^{T-1}\norm{A X_{\tau}-X_{\tau+1}}^2.
\end{equation}
 Note that for least squares loss, one can indeed maintain covariance matrix and residual vector to compute the OLS solution {\em online}. But such a solution does not work if we have access to only gradients and breaks down even for generalized linear models, whereas as the techniques introduced in this work has been extended to non-linear systems \cite{jain2021nonlinear}.

On the other hand,  using standard SGD we can obtain update to $A$ efficiently by using gradient at the current point. That is, assuming $A_0=0$, we get the following SGD update (for all $\tau\geq 0$): 
\begin{equation}
\label{eq:SGD}
A_{\tau+1}=A_{\tau}-2\gamma (A_\tau X_\tau-X_{\tau+1})X_\tau^{\top},
\end{equation}
where $\gamma$ is the stepsize. While SGD is known to be an optimal estimator in certain streaming problems with i.i.d. data, for the $\var(\A,\mu)$ problem the standard SGD  does not apply, as samples $(X_\tau, X_{\tau+1})$ and $(X_{\tau+1}, X_{\tau+2})$ are highly correlated. To see why this is the case, let us unroll the recursion for two steps and using Equation~\eqref{eq:var1}:
\begin{align*}
A_{2}-\A=(A_{0}-\A)(I-2\gamma X_{0}X_{0}^{\top})(I-2\gamma X_{1}X_{1}^{\top}) + 2\gamma \eta_{1}X^{\top}_{1}+ 2\gamma \eta_{0}X^{\top}_{0}(I-2\gamma X_{1}X_{1}^{\top}). 
\end{align*}

Note that the last term does not have $0$ mean because $X_{1}$ depends on $\eta_{0}$ by Equation~\eqref{eq:var1}. Even in the case when $A_{0} = \A$, this means that $\mathbb{E} A_{2} \neq \A$ in general. In fact, in Section~\ref{sec:experiments}, we show empirically that SGD with constant step-size converges to a significantly larger error than OLS, even when $T$ is very large. This shows that we cannot naively treat this problem as a collection of $d$ linear regressions.  This is consistent with the results in \citep{nagaraj2020least,gyorfi1996averaged} which show a similar behavior for constant step-size SGD with dependent data. Now, one can use techniques like {\em data drop} that drops a large fraction of points (either explicitly or during the mathematical analysis) from the stream to obtain nearly independent samples \citep{DuchiAJJ12, nagaraj2020least}, but such methods waste a lot of samples and have significantly suboptimal error rate than OLS. 

So, the goal is to design a streaming method for the problem of learning dynamical systems that at each time-step $t$  provides an accurate estimate of $\A$, while also ensuring small space+time complexity.We now present a novel algorithm that addresses the above mentioned problem. 
\subsection{SGD with Reverse Experience Replay}
\label{subsec:sgdber_define}

\begin{figure}

  \centering
  \includegraphics[width = 0.7\linewidth]{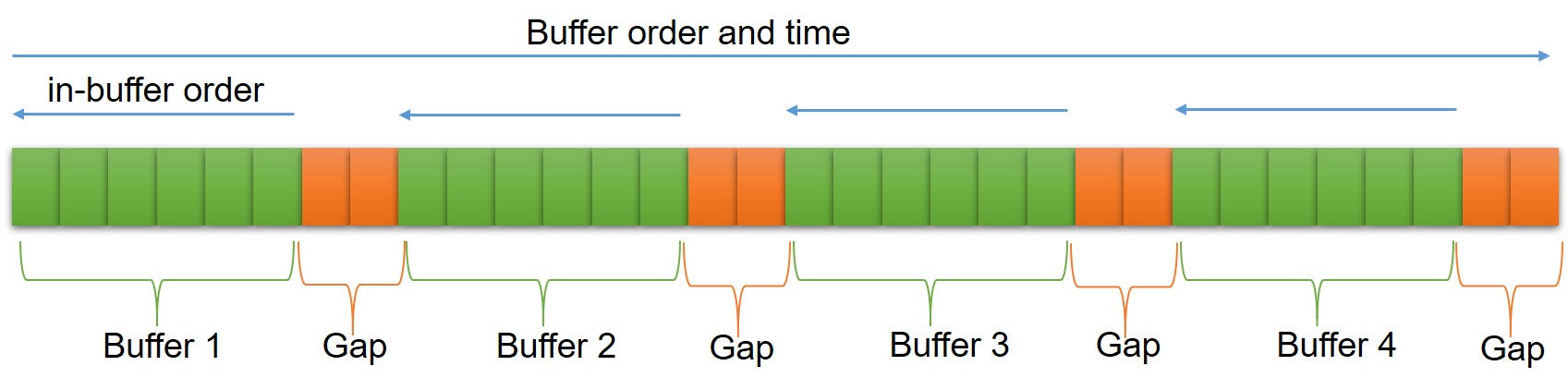}\vspace*{-10pt}
  \caption{ Data Processing Order in $\sgdber$. A cell represents a data point. Time goes from left to right, buffers are also considered from left to right. Within each buffer, the data is processed in the reverse order. Gaps ensure that data in successive buffers are approximately independent.}
  \label{fig:data_order_sgd_rer}
\end{figure}
We now discuss a novel algorithm called SGD with Reverse Experience Replay ($\sgdber$) that addresses the problem of learning stationary auto-regressive models (or linear dynamical systems) in the streaming setting. Our method is inspired by the experience replay technique \citep{lin1992self}, used extensively in RL to break temporal correlations between dependent data. We make the following crucial observation. Suppose in Equation~\eqref{eq:SGD}, instead of processing the samples in the order $(X_1,X_2) \to (X_2,X_3) \to \dots \to (X_{T-1},X_T)$, we process it in the reverse order. That is: $(X_{T-1},X_T) \to (X_{T-2},X_{T-1}) \to \dots \to (X_1,X_2)$. Then, 
\begin{align}\label{eq:two_step_unroll_rev}
A_{2}-\A=(A_{0}-\A)(I-2\gamma X_{T-1}X_{T-1}^{\top})(I-2\gamma X_{T-2}X_{T-2}^{\top}) + 2\gamma \eta_{T-2}X^{\top}_{T-2} \nonumber \\+ 2\gamma \eta_{T-1}X^{\top}_{T-1}(I-2\gamma X_{T-2}X_{T-2}^{\top}) 
\end{align}
Now, observe that $(X_{T-2},X_{T-1})$ are \emph{independent} of $\eta_{T-1}$. Therefore the problematic last term, $2\gamma \eta_{T-1}X^{\top}_{T-1}(I-2\gamma X_{T-2}X_{T-2}^{\top})$, now has expectation $0$. So the updates for \emph{reverse} order SGD would be {\em unbiased}. This, however, requires us to know all the data points beforehand which is infeasible in the streaming setting. We alleviate this issue by designing $\sgdber$, which is the online variant of the above algorithm. $\sgdber$ uses a buffer of large enough size to store values of consecutive data points and then performs reverse SGD in each of these buffers and then discards this buffer. Experience replay methods also use such (small) buffers of data, but typically samples point randomly from the buffer instead of the reverse order that we propose. We refer to Figure~\ref{fig:data_order_sgd_rer} for an illustration of the proposed data processing order. 

We present a pseudocode of $\sgdber$ in Algorithm~\ref{alg:1}. Note that the algorithm forms non-overlapping buffers of size $S=B+u$. Here $B$ is the actual size of the buffer while $u$ samples are used to interleave between two buffers so that the buffers are {\em almost independent} of each other. Now within a buffer, we perform the usual SGD but with samples read in reverse order. Formally, suppose we index our buffers by $t=0,1,2,\cdots$ and let $S=B+u$ be the total samples (including those that were dropped) in the buffers. Let $N$ denote the total number of buffers in horizon $T$. Within each buffer $t$, we index the samples as $X^{t}_i$ where $i=0,1,2,\cdots, S-1$. That is $X^t_i\equiv X_{tS+i}$ is the $i$-th sample in buffer $t$. Similarly $\Nt{t}{i}\equiv \eta_{tS+i}$. Further let $\Xt{t}{-i}\equiv \Xt{t}{(S-1)-i}$. Similarly we set $\Nt{t}{-i}\equiv \Nt{t}{(S-1)-i}$ Then, the algorithm performs the recursion stated in Line 1 of Algorithm~\ref{alg:1}. Note that the recursion can also be written as, 
\begin{\Ieee}{LLL}
\label{eq:sgd_expreplay2}
\At{t-1}{i+1}-\A=\left(\At{t-1}{i}-\A\right)\left(I-2\gamma \Xt{t-1}{-i}\Xt{t-1}{-i}^{\top}\right)+2\gamma\Nt{t-1}{-i}\Xt{t-1}{-i}.\Ieeen
\end{\Ieee}
for $1\leq t\leq N$ and $0\leq i\leq B-1$ with $A^t_0=A^{t-1}_{B}$ and $A^0_0=A_0$.

We then ignore the iterates corresponding to first $a$ buffers as part of the {\em burn-in period}, and output average of the remaining iterates ($t>a$) at each step as that step's estimator (see Line 2 of Algorithm~\ref{alg:1}). That is, we have the tail-averaged iterate: 
\begin{\Ieee}{LLL}
\label{eq:tail_avg_defn}
\hat A_{a,t}=\frac{1}{t-a}\sum_{\tau=a+1}^t \At{\tau-1}{B}.\Ieeen 
\end{\Ieee}
We output the new iterate $\hat A_{a,t}$ only at the end of each buffer $t$. At intermediate steps, $(t-1)B+1\leq \tau\leq tB$, we  output $\hat A_{a,t-1}$. 
 Also, note that the tail average can be computed in small space and time complexity, by using a running sum of the tail iterates. The update for each point is rank-one, so can be computed in time linear in number of parameters ($O(d^2)$). In the next section, we show that despite using small buffer size $S=B+u$ (that depends logarithmically on $T$), and by throwing away a small constant--independent of {\em any} problem parameter--fraction of points $u$ in each buffer, we are still able to provide error bound similar to that of OLS.


\begin{algorithm}[t!]
	\label{alg:1}
	\DontPrintSemicolon
	\SetKwInOut{Input}{Input}
	\SetKwInOut{Output}{Output}
	\SetKwFunction{RN}{ReadNext}\SetKwFunction{LN}{LeaveNext}
	\Input{Streaming data $\{X_\tau\}$, horizon $T$, buffer size $B$, buffer gap $u$, bound $R$, tail average start: $a$}
	\Output{Estimate $\hat A_{a,t}$, for all $a< t\leq N-1$; $N=T/(B+u)$}
	\Begin{
		Step-size: $\gamma\leftarrow \frac{1}{8RB}$, Total buffer size: $S\leftarrow B+u$, Number of buffers: $N\leftarrow T/S$\;
		$A^0_0=0$ \textsf{/*Initialization*/}\;
		\For{$t\leftarrow 1$ \KwTo $N$}{
			Form buffer $\textsf{Buf}^{t-1}=\{X^{t-1}_0, \dots, X^{t-1}_{S-1}\}$, where, $X^{t-1}_{i}\leftarrow X_{(t-1)\cdot S+i}$\;
			If $\exists i,\ s.t.,\ \norm{X^{t-1}_i}^2>R$, then \textbf{return} $\hat A_{a,t}=0$\;
			\For{$i\leftarrow 0$ \KwTo $B-1$}{
				\nl $A^{t-1}_{i+1}\leftarrow A^{t-1}_i - 2 \gamma (\At{t-1}{i}\Xt{t-1}{S-1-i}-\Xt{t-1}{S-i})\left(\Xt{t-1}{S-1-i}\right)^{\top}$
			}
			$A^{t}_0=A^{t-1}_B$\;
			If $t> a$, then $\hat A_{a,t}\leftarrow \frac{1}{t-a}\sum_{\tau=a+1}^{t} A^{\tau-1}_B$
		}
	}
	\caption{$\sgdber$}
\end{algorithm}


\section{Main Results}
\label{sec:main_results}
We now state our main results with leading order terms. For simplicity, we only state the results for the tail average $\hat{A}_{\tfrac{N}{2},N}$ but a similar result holds for any $\hat{A}_{a,t}$ when $a = \Omega(dB\kappa(G)\log^2 T)$. We refer to Section~\ref{sec:formal_proof_sketch} for complete statements.
Recall the problem setting, and the covariance matrix $G := \mathbb{E}_{X\sim\pi}XX^{\intercal}$. Before stating the results, we choose the parameters $B,R,\alpha$ and $u$ as follows, which can be estimated using upper bounds on $\|\A\|$:
\setlist{nolistsep}
\begin{enumerate}[leftmargin=*, noitemsep]
\item $d\leq \poly(T)$. We use this to bound the norm of covariates in the next item.
\item $\alpha \geq 22$ ;\quad $R \geq C(\alpha)\frac{\tr(\Sigma)\log T}{1-\norm{\A}^2} = O(d\tau_{\mathsf{mix}}\log T)$ s.t. $\Pb{\norm{X_\tau}^2\leq R,\, \tau\leq T}\geq 1-\prbnd$. See lemma~\ref{lem:2} in appendix. 
\item $u\geq \alpha \tfrac{\log T}{\log\left(\tfrac{1}{\|A^{*}\|}\right)} = O(\tau_{\mathsf{mix}}\log T)$;\quad $B=10 u$
\end{enumerate}
 

For all the results below, we suppose that Assumptions~\ref{as:norm_condition},~\ref{as:noise_concentration} and~\ref{as:stationarity} hold, the stream of samples ${X_\tau}$ is sampled from $\var(\A, \mu)$ model described in Section~\ref{sec:prob} and that $R,B,\alpha$ and $u$ are chosen as above. Further we hide some mild conditions on $N$ and $T$.  
\begin{theorem}[Informal version of Theorem~\ref{thm:main_upper_bound_op}]\label{thm:op_informal}
	 Let the step size $\gamma < \min \left(\tfrac{C}{B\lmin{G}},\tfrac{1}{8BR}\right)$ for some constant $C$ depending only on $C_{\mu}$. 
Then, with probability at least $1-\frac{1}{T^{100}}$, we have: 
{\small
$$
	 \lossop(\hat{A}_{\tfrac{N}{2},N},\A,\mu) \leq C\sqrt{\frac{(d+\log T) \lmax{\Sigma}}{T \lmin{G}} } +
 \text{ Lower Order Terms} \,.$$
}
\end{theorem}

\begin{theorem}[Informal version of Theorem~\ref{thm:main_upper_bound_pred}]\label{thm:pred_informal}
	Consider the setting of Theorem~\ref{thm:op_informal} but where the  step size $\gamma = \min \left(\frac{1}{2R},\tfrac{c}{BR}\right)$ for some constant $0<c<1$. Then, the following holds: 
{\small\begin{\Ieee}{LLL}
\Ex{\losspred(\hat{A}_{\frac{N}{2},N};\A,\mu)} -\tr(\Sigma)\leq C\frac{d\tr(\Sigma)}{T} +\text{ Lower Order Terms}
\end{\Ieee}}
where ``lower order'' is with respect to $\frac{d}{T}$.
\end{theorem}
See Section~\ref{subsec:variance_last_iterate}, Section~\ref{subsec:bias_last_iterate} for a detailed proof of the parameter error bound and see  Section~\ref{sec:pred_var}, Section~\ref{sec:pred_bias} for a detailed proof of the prediction error bound.

We now make the following observations: 
\noindent
\begin{enumerate}[label=(\arabic*),leftmargin=*]
\item  The dominant term in our bound on $\lossop$ (Theorem~\ref{thm:op_informal}) matches the information theoretically optimal bound (up to logarithmic factors) for the $\var(\A,\mu)$ estimation problem \citep{simchowitz2018learning} as long as $\|\A\| \leq 1-\frac{1}{T^{\xi}}$ for $\xi \in (0,1/2)$. Note that despite working with dependent data, leading term in our error bound is nearly independent of mixing time $\tau_{\mathsf{mix}}$. In contrast, most of the existing streaming/SGD style methods for dependent data have strong dependence on $\tau_{\mathsf{mix}}$ \cite{nagaraj2020least}. 
\item SGD for linear regression with {\em independent} data \citep{jain2017parallelizing,defossez2015averaged}, but with similar problem setting incurs error  $O(\frac{d \tr(\Sigma)}{T})$ for $\losspred$. So  our bound for $\sgdber$ matches the independent data setting bound in the minimax sense.
\item The space complexity of our method is $O(Bd+d^2)$ where $B=O(\tau_{\mathsf{mix}} \log T)$ is independent of $d$ and only logarithmically dependent on $T$. 
\item {\bf Sparse matrices with known support}: Suppose $\A$ is known to be sparse and \emph{we know the support} (say by running $L_1$ regularized OLS on a small set of samples). Let $s_j$ denote the sparsity of row $j$ of $\A$. Then the $\sgdber$ algorithm can be modified to run row by row such that it operates only on the support of row $j$. That is the covariates can be projected onto the support of each row. Then it can be shown that the prediction error is bounded as $O\left(\sum_{j=1}^d \sigma_j^2 s_j/T\right)$ where $\sigma_j^2$ is the $j$-th diagonal entry of $\Sigma$. Note that $\sgdber$ requires only $O(|\mathsf{supp}(\A)|)$ operations per iteration while applying online version of standard OLS would require $O(d^2)$ operations. In the simple case of $\Sigma = \sigma^2 I$, we note that $G \succeq \sigma^2 I$ and hence the bound for $\losspred$ becomes $O\left(\tfrac{|\mathsf{supp}(\A)|}{T}\right)$. We refer to Section~\ref{sec:sparse_system} for a sketch of this extension. 
\end{enumerate}
Next, we show that our error bounds are nearly information theoretically optimal. For the lower bound on $\lossop$ we directly use ~\citep[Theorem 2.3]{simchowitz2018learning}. 
\begin{theorem}
\label{thm:main_lower_bound_lossop}
Let $\rho<1$ and $\delta\in(0,1/4)$. Let $\mu$ be the distribution $\mathcal{N}(0,\sigma^2 I)$. For  any estimator $\hat A\in\mathcal{F}$, there exists an matrix $\A\in\mathbb{R}^{d\times d}$ where $\A = \rho O$ for some orthogonal matrix $O$ such that $|\sigma_{\max}(\A)|=\rho$ and we have that with probability at least $\delta$:
\begin{equation}
\label{eq:main_lower_bound_lossop}
\|\hat A -\A\| = \Omega\sqrt{\frac{(d+\log(1/\delta))(1-\rho)}{T}}.
\end{equation}
\end{theorem}
Notice that in the setting of Theorem~\ref{thm:main_lower_bound_lossop}, we have $G = \sum_{i=0}^{\infty}\sigma^2 (\A)^{i} (\A)^{i,\top} = \frac{\sigma^2}{1-\rho^2}I$. Therefore, $\lmin{G} = \frac{1}{1-\rho^2} \sim \frac{1}{1-\rho} $. The bound in Theorem~\ref{thm:op_informal} matches the above minimax bound up to logarithmic factors. 

Next we consider the prediction loss. We fix dimension $d$ and horizon $T$ and consider the class of $\var$ models $\mathcal{M}$ such that Assumptions~\ref{as:norm_condition},~\ref{as:noise_concentration}, and~\ref{as:stationarity} hold such that $\mathsf{Tr}(\Sigma(\mu)) = \beta \in \mathbb{R}^{+}$ be fixed. Let $\mathcal{F}$ be the class of all estimators for parameter $\A$ given data $(Z_0,\dots,Z_T)$.	 We want to lower bound the minimax error:
$$\lossmm(\mathcal{M}) :=  \inf_{f\in \mathcal{F}}\sup_{\left(\A,\mu\right)\in \mathcal{M}} \mathbb{E}_{(Z_t)\sim \var(\A,\mu)} \losspred(f(Z_0,\dots,Z_T);\A,\mu) - \losspred(\A;\A,\mu).$$
\begin{theorem}\label{thm:main_lower_bound}
For some universal constant $c$, we have:
$$\lossmm(\mathcal{M}) \geq c\beta (d-1) \min\left(\frac{1}{T},\frac{1}{d^2}\right), \text{ where } \beta = \tr(\Sigma(\mu)).$$
\end{theorem} 
Note that the theorem shows that our algorithm is minimax optimal with respect to the prediction loss at stationarity, $\losspred$. See Section~\ref{sec:lower_bounds} for a detailed proof of the above lower bound. 

\section{Idea Behind Proofs}\label{sec:proofsketch}
In this section, we provide an overview of the key techniques to prove our results. As observed in the discussion following Equation~\eqref{eq:two_step_unroll_rev}, when the data is processed in the reverse order within a buffer, it behaves similar to SGD for linear regression with i.i.d. data. Due to the gaps of size $u$, we can take the buffers to be approximately independent. Therefore, we analyze the algorithm as follows:
\begin{enumerate}
\item Analyze reverse order \emph{within} a buffer using the property noted in Equation~\eqref{eq:two_step_unroll_rev}.
\item Treat \emph{different} buffers to be i.i.d. due to gap and present an i.i.d data type analysis.
\end{enumerate}

To execute the proposed proof strategy, we introduce the following technical notions:\vspace*{-3pt}
\paragraph{Coupled Process.} 
For the real data points $(X_\tau)$, the points in different buffers are \emph{weakly} dependent. In order to make the analysis straight forward, we introduce the \emph{fictitious} coupled process $\tilde{X}_{\tau}$ such that $\norm{\tilde{X}_{\tau} -X_{\tau}} \lesssim \frac{1}{T^{\alpha}}$ for large enough $\alpha$, for every data point $X_{\tau}$ used by $\sgdber$. We have the additional property that the successive buffers are actually independent for this coupled process. We refer to Definition~\ref{def:1} in the appendix for the construction of the coupled process $\tilde{X}_{\tau}$.

Suppose we run $\sgdber$ with the coupled process $\tilde{X}_{\tau}$ instead of $X_{\tau}$ to obtain the coupled iterates $\tilde{A}^{t}_i$. We can then show that $\tilde{A}_i^{t} \approx A^{t}_i$. Thus it suffices analyze the coupled iterates $\tilde{A}^{t}_i$. We refer to Sections~\ref{subsec:basic_lemmas} and~\ref{sec:initial_coupling} for the details. \vspace*{-3pt}
\paragraph{Bias Variance Decomposition.} 
We consider the standard bias variance decomposition with individual buffers as the basic unit as opposed to individual data points. We refer to Section~\ref{sec:bias_variance} for the details. We decompose the error in the iterates into the bias part $\Atto{t-1}{B}=(A_0-\A)\prod_{s=0}^{t-1}\Htt{s}{0}{B-1}$ and the variance part $\Attdiff{t-1}{B}=2\gamma\sum_{r=1}^{t}\sum_{j=0}^{B-1}\Nt{t-r}{-j}\Xtttr{t-r}{-j}\Htt{t-r}{j+1}{B-1}{\prod_{s=r-1}^{1}\Htt{t-s}{0}{B-1}} $ where the matrices $\Htt{s}{0}{B-1}=\prod_{i=0}^{B-1}\Ptt{s}{-i}$ are the independent 'contraction' matrices associated with each buffer $s$. This result in the geometric decay of the initial distance between $(A_0-\A)$. The variance part is due to the inherent noise present in the data. In Section~\ref{subsec:variance_last_iterate} we first establish the exponential decay of the `bias'. We then consider the second moment of the variance term. Observe that the distinct terms in the expression for $\Attdiff{t-1}{B}$ are uncorrelated either due to reverse order \emph{within} a buffer as noted in Equation~\eqref{eq:two_step_unroll_rev} or due to independence between the data in distinct buffers (due to coupling). This allows us to split the second moment into diagonal terms with non-zero mean and cross terms with zero mean. Diagonal terms are analyzed via a recursive argument in Claim~\ref{claim:single_buffer_isometry} and the following discussion in order to remove dependence on mixing time factors. The analysis for parameter recovery (the result of Theorem~\ref{thm:pred_informal}) is similar but we bound the relevant exponential moments using sub-Gaussianity of the noise sequence $\eta_t$ to obtain high-probability bounds which when combined with standard $\epsilon$-net arguments give us guarantees for the operator norm error $\lossop$.\vspace*{-3pt}
\paragraph{Averaged Iterates.} 
We then combine the bias and variance bounds obtained for individual iterates in Section~\ref{subsec:variance_last_iterate} to analyze the tail averaged output. Using techniques standard in the analysis of SGD for linear regression, we finally show that this averaging leads error rates of the order $\frac{d^2}{T}$. We refer to Sections~\ref{sec:main_op_bound} (for parameter recover) and ~\ref{sec:pred_loss} (for prediction error) for the detailed results.  \vspace*{-3pt}
\paragraph{Picking the Step Sizes and Conditioning.} 
Due to the auto-regressive nature of the data generation, the iterates can grow to be of the size $O(\frac{d}{1-\rho})$. The step sizes need to be set small enough so that the $\gamma \|X_{\tau}X_{\tau}^{\top}\| \leq 1$ in order for the $\sgdber$ iterations to not diverge to infinity. In the statement of Theorem~\ref{thm:pred_informal}, we condition on the event where $\|X_{\tau}\|^2$ are all bounded by a sufficiently large number $R$ for every $\tau$ in order to ensure this property. The relevant events where the norm is bounded are defined in Section~\ref{subsec:basic_lemmas}. Conditioning on these events results in previously zero mean terms to be not zero mean. Routine calculations using triangle inequality and Cauchy-Schwarz inequality ensure that the means are still of the order $\frac{1}{T^{\alpha}}$ for any fixed constant $\alpha > 0$. Furthermore, we actually require step sizes such that $\gamma \norm{\sum_{\tau\in\mathrm{Buffer}}X_{\tau} X^{\top}_{\tau}}\leq 1$ to show exponential contraction of $\Htt{s}{0}{B-1}$ matrices due to the Grammian $G$ as described next.\vspace*{-3pt}
\paragraph{Probabilistic Results.} 
We establish some properties of $\Htt{s}{0}{B-1}$, which are products of dependent random matrices in Section~\ref{sec:op_norm}. Specifically we refer to Lemmas~\ref{lem:contraction},~\ref{lem:almost_sure_contraction},~\ref{lem:probable_contraction}, and~\ref{lem:operator_norm_bound_1} which establish that $\norm{\prod_{s=0}^{t-1}\Htt{s}{0}{B-1}} \lesssim (1-\gamma B \lmin{G})^{t}$ with high probability. 

\section{Experiments}
\label{sec:experiments}
\begin{figure}[t!]
\begin{center}
	\vspace*{-10pt}
	\includegraphics[width=\columnwidth, height=6cm]{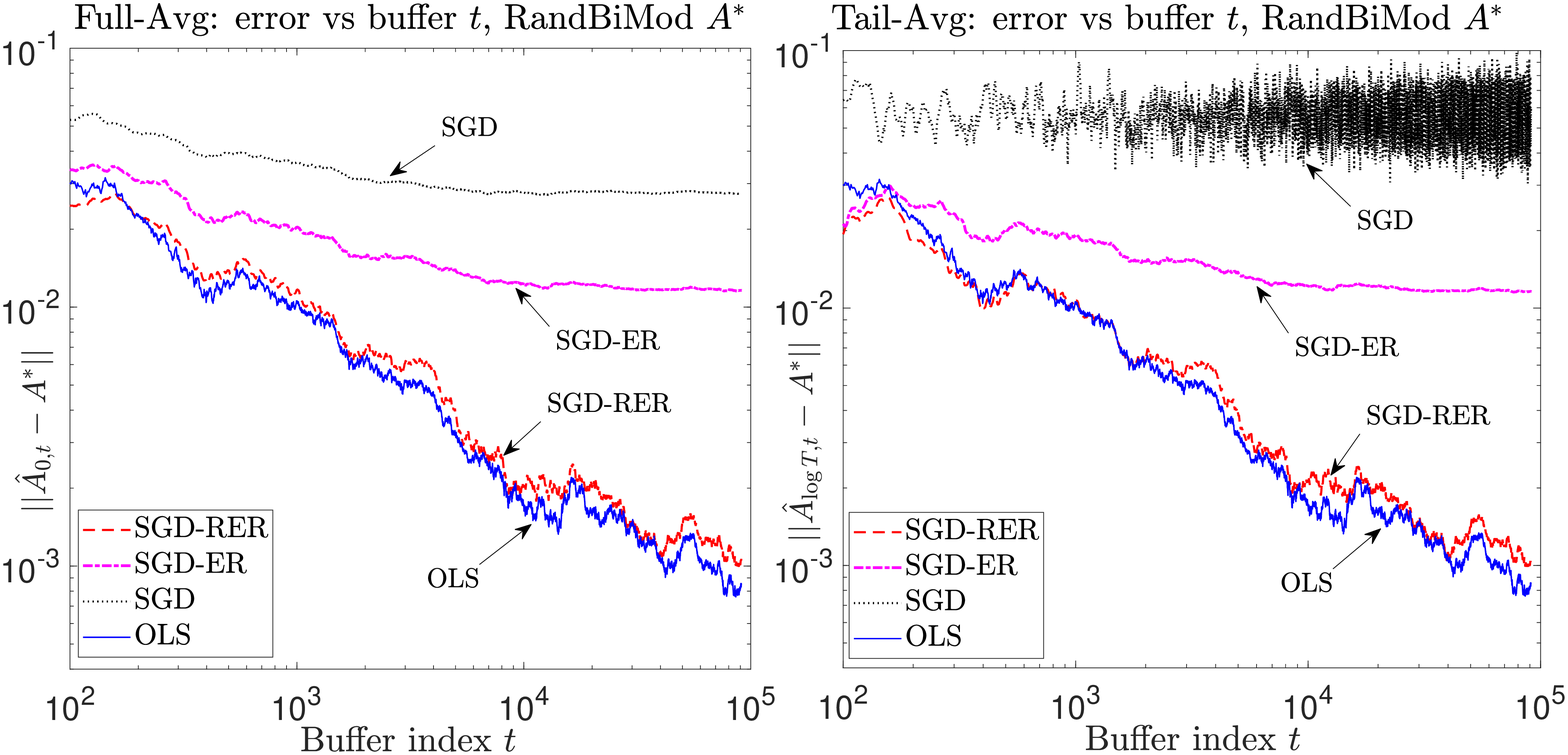}\vspace*{-10pt}
	\caption {Gaussian $\var(\A,\mu)$:  Parameter error for  tail averaged and full average iterates of $\sgdber$ and baselines. $\sgdber$ and $\ols$ incur similar parameter error, while error incurred by $\sgd$ and $\sgder$ saturate at significantly higher level, indicating non-zero bias. The parameters used are $\rho=0.9$, $d=5$, $T=10^7$, $B=100$, $u=10$. $R$ is estimated and $\gamma=1/2R$. }
	\label{fig:1}
\end{center} 
\end{figure}

In this section, we compare performance of our $\sgdber$ method on synthetic data against the performance of standard baselines $\ols$ and $\sgd$, along with $\sgder$ method that applies standard experience replay technique, but where points from a buffer are sampled {\em randomly}. \\
\textbf{Synthetic data}: We sample data from $\var(\A,\mu)$  with $X_0=0$, $\mu\distas{}\mathcal{N}(0,\sigma^2 I)$ and $\A\in \mathbb{R}^{d\times d}$ is generated from the "RandBiMod" distribution. That is, $\A=U\Lambda U^{\top}$ with random orthogonal $U$, and $\Lambda$ is diagonal with $\lceil d/2\rceil$ entries on diagonal being $\rho$ and the remaining diagonal entries are set to $\rho/3$. We set $d=5$, $\rho=0.9$ and $\sigma^2=1$. 
We fix a horizon $T=10^7$ and set the buffer size as $B=100$ and $u=10$. To estimate $R$ from the data, we use the first $\lfloor 2\log T\rfloor=32$ samples and set $R$ as the sum of the norms of these samples. We let the stepsize to be $\gamma=\frac{1}{2R}$ which is \emph{aggressive} compared to our theorems. We start the $\sgdber$ and other $SGD$-like algorithms from the second buffer onward.

For tail averaging, as described in algorithm~\ref{alg:1}, we ignore the first $\lfloor \log T \rfloor=16$ buffers, and maintain a running tail average at the end of each of the subsequent buffers. In figure~\ref{fig:1}, we plot the parameter errors $\norm{\hat A_{\log T,t}-\A}$ and $\norm{\hat A_{0,t}-\A}$ versus the buffer index $t$ as the algorithm runs for horizon $T$. For $\ols$, we include samples in the first buffer as well (which were used for estimating $R$). Clearly,  $\sgdber$ has very similar performance as that of $\ols$ whereas $\sgder$ and $\sgd$ seem to display residual bias for the chosen step-size (which is logarithmic in the horizon $T$) and buffer lengths. We also observe a similar behavior when we choose $\A=\rho I$. 






\section{Conclusion}
In this paper, we studied the problem of linear system identification in streaming setting and provided an efficient algorithm ($\sgdber$). We proved that $\sgdber$ achieves nearly minimax optimal error rate, both in terms of parameter error as well as prediction error. Furthermore, using experiments, we validated that standard SGD as well as SGD with experience replay can have large bias error. Our algorithm and analysis demonstrates that the knowledge of dependency structure can aid us in designing accurate algorithms for dependent data.

This work opens up a myriad of open questions about learning from dependent data in general and Markov processes in particular. Our work currently assumes a specific Markovian dependency structure -- extending the intuition and techniques to handle more general data dependencies is an interesting open question. Further, our work does not address the question of recovering a sparse system matrix with unknown sparsity pattern. So online learning of  such linear dynamical systems with (unknown) sparsity pattern or low-rank structure is an exciting question with applications to domains like bioinformatics. Moreover, even in our linear setting, extending $\sgdber$ to the situation of partially observed states with or without control inputs would be another direction to pursue. Finally, it would be interesting to understand how the techniques introduced in this work perform in practical RL settings where learning with data from Markov processes is essential. 



\begin{ack}
D.N. was supported in part by NSF grant DMS-2022448.\\
S.S.K was supported in part by Teaching Assistantship (TA) from EECS, MIT.\\
Part of this work was done when S.S.K was visiting Microsoft Research Lab India Pvt Ltd during summer 2020.
\end{ack}

\clearpage
\bibliographystyle{unsrtnat}
\bibliography{refs}

\begin{thebibliography}{55}
\providecommand{\natexlab}[1]{#1}
\providecommand{\url}[1]{\texttt{#1}}
\expandafter\ifx\csname urlstyle\endcsname\relax
  \providecommand{\doi}[1]{doi: #1}\else
  \providecommand{\doi}{doi: \begingroup \urlstyle{rm}\Url}\fi

\bibitem[Kumar and Varaiya(2015)]{kumar2015stochastic}
Panqanamala~Ramana Kumar and Pravin Varaiya.
\newblock \emph{{Stochastic systems: Estimation, identification, and adaptive
  control}}.
\newblock SIAM, 2015.

\bibitem[A{\c{c}}{\i}kme{\c{s}}e et~al.(2013)A{\c{c}}{\i}kme{\c{s}}e, Carson,
  and Blackmore]{accikmecse2013lossless}
Beh{\c{c}}et A{\c{c}}{\i}kme{\c{s}}e, John~M Carson, and Lars Blackmore.
\newblock {Lossless convexification of nonconvex control bound and pointing
  constraints of the soft landing optimal control problem}.
\newblock \emph{IEEE Transactions on Control Systems Technology}, 21\penalty0
  (6):\penalty0 2104--2113, 2013.

\bibitem[Hamilton(2020)]{hamilton2020time}
James~Douglas Hamilton.
\newblock \emph{Time series analysis}.
\newblock Princeton university press, 2020.

\bibitem[Fujita et~al.(2007)Fujita, Sato, Garay-Malpartida, Yamaguchi, Miyano,
  Sogayar, and Ferreira]{gene}
André Fujita, João~R Sato, Humberto~M Garay-Malpartida, Rui Yamaguchi, Satoru
  Miyano, Mari~C Sogayar, and Carlos~E Ferreira.
\newblock Modeling gene expression regulatory networks with the sparse vector
  autoregressive model.
\newblock \emph{BMC Systems Biology}, 1:\penalty0 39, 2007.

\bibitem[Simchowitz et~al.(2018)Simchowitz, Mania, Tu, Jordan, and
  Recht]{simchowitz2018learning}
Max Simchowitz, Horia Mania, Stephen Tu, Michael~I Jordan, and Benjamin Recht.
\newblock {Learning without mixing: Towards a sharp analysis of linear system
  identification}.
\newblock \emph{arXiv preprint arXiv:1802.08334}, 2018.

\bibitem[Sarkar and Rakhlin(2019)]{sarkar2019near}
Tuhin Sarkar and Alexander Rakhlin.
\newblock Near optimal finite time identification of arbitrary linear dynamical
  systems.
\newblock In \emph{International Conference on Machine Learning}, pages
  5610--5618. PMLR, 2019.

\bibitem[Hanck et~al.(2019)Hanck, Arnold, Gerber, and
  Schmelzer]{hanck2019introduction}
Christoph Hanck, Martin Arnold, Alexander Gerber, and Martin Schmelzer.
\newblock Introduction to econometrics with r.
\newblock \emph{University of Duisburg-Essen}, 2019.

\bibitem[Zheng et~al.(2016)Zheng, Tang, Ding, and Zhou]{zheng2016neural}
Yin Zheng, Bangsheng Tang, Wenkui Ding, and Hanning Zhou.
\newblock A neural autoregressive approach to collaborative filtering.
\newblock In \emph{International Conference on Machine Learning}, pages
  764--773. PMLR, 2016.

\bibitem[Basu et~al.(2015)Basu, Michailidis, et~al.]{basu2015regularized}
Sumanta Basu, George Michailidis, et~al.
\newblock Regularized estimation in sparse high-dimensional time series models.
\newblock \emph{The Annals of Statistics}, 43\penalty0 (4):\penalty0
  1535--1567, 2015.

\bibitem[Basu et~al.(2019)Basu, Li, and Michailidis]{Basu_2019}
Sumanta Basu, Xianqi Li, and George Michailidis.
\newblock {Low Rank and Structured Modeling of High-Dimensional Vector
  Autoregressions}.
\newblock \emph{IEEE Transactions on Signal Processing}, 67\penalty0
  (5):\penalty0 1207–1222, Mar 2019.
\newblock ISSN 1941-0476.
\newblock \doi{10.1109/tsp.2018.2887401}.

\bibitem[Jain et~al.(2021)Jain, Kowshik, Nagaraj, and
  Netrapalli]{jain2021nonlinear}
Prateek Jain, Suhas~S Kowshik, Dheeraj Nagaraj, and Praneeth Netrapalli.
\newblock {Near-optimal Offline and Streaming Algorithms for Learning
  Non-Linear Dynamical Systems}.
\newblock \emph{arXiv preprint arXiv:2105.11558}, 2021.

\bibitem[Agarwal et~al.(2021)Agarwal, Chaudhuri, Jain, Nagaraj, and
  Netrapalli]{agarwal2021online}
Naman Agarwal, Syomantak Chaudhuri, Prateek Jain, Dheeraj Nagaraj, and Praneeth
  Netrapalli.
\newblock Online target q-learning with reverse experience replay: Efficiently
  finding the optimal policy for linear mdps.
\newblock \emph{arXiv preprint arXiv:2110.08440}, 2021.

\bibitem[Jain et~al.(2017)Jain, Netrapalli, Kakade, Kidambi, and
  Sidford]{jain2017parallelizing}
Prateek Jain, Praneeth Netrapalli, Sham~M Kakade, Rahul Kidambi, and Aaron
  Sidford.
\newblock Parallelizing stochastic gradient descent for least squares
  regression: mini-batching, averaging, and model misspecification.
\newblock \emph{The Journal of Machine Learning Research}, 18\penalty0
  (1):\penalty0 8258--8299, 2017.

\bibitem[Nagaraj et~al.(2020)Nagaraj, Wu, Bresler, Jain, and
  Netrapalli]{nagaraj2020least}
Dheeraj Nagaraj, Xian Wu, Guy Bresler, Prateek Jain, and Praneeth Netrapalli.
\newblock {Least Squares Regression with Markovian Data: Fundamental Limits and
  Algorithms}.
\newblock \emph{Advances in Neural Information Processing Systems}, 33, 2020.

\bibitem[Gy{\"o}rfi and Walk(1996)]{gyorfi1996averaged}
L{\'a}szl{\'o} Gy{\"o}rfi and Harro Walk.
\newblock On the averaged stochastic approximation for linear regression.
\newblock \emph{SIAM Journal on Control and Optimization}, 34\penalty0
  (1):\penalty0 31--61, 1996.

\bibitem[Rotinov(2019)]{rotinov2019reverse}
Egor Rotinov.
\newblock {Reverse Experience Replay}.
\newblock \emph{arXiv preprint arXiv:1910.08780}, 2019.

\bibitem[Ambrose et~al.(2016)Ambrose, Pfeiffer, and Foster]{ambrose2016reverse}
R~Ellen Ambrose, Brad~E Pfeiffer, and David~J Foster.
\newblock Reverse replay of hippocampal place cells is uniquely modulated by
  changing reward.
\newblock \emph{Neuron}, 91\penalty0 (5):\penalty0 1124--1136, 2016.

\bibitem[Whelan et~al.(2021)Whelan, Prescott, and Vasilaki]{whelan2021robotic}
Matthew~T Whelan, Tony~J Prescott, and Eleni Vasilaki.
\newblock A robotic model of hippocampal reverse replay for reinforcement
  learning.
\newblock \emph{arXiv preprint arXiv:2102.11914}, 2021.

\bibitem[Oymak and Ozay(2019)]{oymak2019non}
Samet Oymak and Necmiye Ozay.
\newblock Non-asymptotic identification of lti systems from a single
  trajectory.
\newblock In \emph{2019 American Control Conference (ACC)}, pages 5655--5661.
  IEEE, 2019.

\bibitem[Faradonbeh et~al.(2018)Faradonbeh, Tewari, and
  Michailidis]{faradonbeh2018finite}
Mohamad Kazem~Shirani Faradonbeh, Ambuj Tewari, and George Michailidis.
\newblock Finite time identification in unstable linear systems.
\newblock \emph{Automatica}, 96:\penalty0 342--353, 2018.

\bibitem[Jedra and Proutiere(2020)]{jedra2020finite}
Yassir Jedra and Alexandre Proutiere.
\newblock {Finite-time Identification of Stable Linear Systems Optimality of
  the Least-Squares Estimator}.
\newblock In \emph{2020 59th IEEE Conference on Decision and Control (CDC)},
  pages 996--1001. IEEE, 2020.

\bibitem[Jedra and Proutiere(2019)]{jedra2019sample}
Yassir Jedra and Alexandre Proutiere.
\newblock Sample complexity lower bounds for linear system identification.
\newblock In \emph{2019 IEEE 58th Conference on Decision and Control (CDC)},
  pages 2676--2681. IEEE, 2019.

\bibitem[Sattar and Oymak(2020)]{sattar2020non}
Yahya Sattar and Samet Oymak.
\newblock Non-asymptotic and accurate learning of nonlinear dynamical systems.
\newblock \emph{arXiv preprint arXiv:2002.08538}, 2020.

\bibitem[Foster et~al.(2020)Foster, Sarkar, and Rakhlin]{foster2020learning}
Dylan Foster, Tuhin Sarkar, and Alexander Rakhlin.
\newblock Learning nonlinear dynamical systems from a single trajectory.
\newblock In \emph{Learning for Dynamics and Control}, pages 851--861. PMLR,
  2020.

\bibitem[Lai and Wei(1983)]{lai1983asymptotic}
TL~Lai and CZ~Wei.
\newblock Asymptotic properties of general autoregressive models and strong
  consistency of least-squares estimates of their parameters.
\newblock \emph{Journal of multivariate analysis}, 13\penalty0 (1):\penalty0
  1--23, 1983.

\bibitem[Campi and Weyer(2002)]{campi2002finite}
Marco~C Campi and Erik Weyer.
\newblock Finite sample properties of system identification methods.
\newblock \emph{IEEE Transactions on Automatic Control}, 47\penalty0
  (8):\penalty0 1329--1334, 2002.

\bibitem[Vidyasagar and Karandikar(2006)]{vidyasagar2006learning}
Mathukumalli Vidyasagar and Rajeeva~L Karandikar.
\newblock A learning theory approach to system identification and stochastic
  adaptive control.
\newblock In \emph{Probabilistic and randomized methods for design under
  uncertainty}, pages 265--302. Springer, 2006.

\bibitem[Kuznetsov and Mohri(2018)]{kuznetsov2018theory}
Vitaly Kuznetsov and Mehryar Mohri.
\newblock Theory and algorithms for forecasting time series.
\newblock \emph{arXiv preprint arXiv:1803.05814}, 2018.

\bibitem[Hardt et~al.(2018)Hardt, Ma, and Recht]{hardt2018gradient}
Moritz Hardt, Tengyu Ma, and Benjamin Recht.
\newblock Gradient descent learns linear dynamical systems.
\newblock \emph{The Journal of Machine Learning Research}, 19\penalty0
  (1):\penalty0 1025--1068, 2018.

\bibitem[Tsiamis and Pappas(2019)]{tsiamis2019finite}
Anastasios Tsiamis and George~J Pappas.
\newblock Finite sample analysis of stochastic system identification.
\newblock In \emph{2019 IEEE 58th Conference on Decision and Control (CDC)},
  pages 3648--3654. IEEE, 2019.

\bibitem[Sarkar et~al.(2021)Sarkar, Rakhlin, and Dahleh]{sarkar2021finite}
Tuhin Sarkar, Alexander Rakhlin, and Munther~A Dahleh.
\newblock {Finite Time LTI System Identification.}
\newblock \emph{J. Mach. Learn. Res.}, 22:\penalty0 26--1, 2021.

\bibitem[Lale et~al.(2020)Lale, Azizzadenesheli, Hassibi, and
  Anandkumar]{lale2020logarithmic}
Sahin Lale, Kamyar Azizzadenesheli, Babak Hassibi, and Anima Anandkumar.
\newblock Logarithmic regret bound in partially observable linear dynamical
  systems.
\newblock \emph{arXiv preprint arXiv:2003.11227}, 2020.

\bibitem[Lee(2020)]{lee2020improved}
Holden Lee.
\newblock Improved rates for identification of partially observed linear
  dynamical systems.
\newblock \emph{arXiv preprint arXiv:2011.10006}, 2020.

\bibitem[Lee and Lamperski(2020)]{lee2020non}
Bruce Lee and Andrew Lamperski.
\newblock Non-asymptotic closed-loop system identification using autoregressive
  processes and hankel model reduction.
\newblock In \emph{2020 59th IEEE Conference on Decision and Control (CDC)},
  pages 3419--3424. IEEE, 2020.

\bibitem[Cohen et~al.(2018)Cohen, Hasidim, Koren, Lazic, Mansour, and
  Talwar]{cohen2018online}
Alon Cohen, Avinatan Hasidim, Tomer Koren, Nevena Lazic, Yishay Mansour, and
  Kunal Talwar.
\newblock Online linear quadratic control.
\newblock In \emph{International Conference on Machine Learning}, pages
  1029--1038. PMLR, 2018.

\bibitem[Agarwal et~al.(2019)Agarwal, Bullins, Hazan, Kakade, and
  Singh]{agarwal2019online}
Naman Agarwal, Brian Bullins, Elad Hazan, Sham Kakade, and Karan Singh.
\newblock Online control with adversarial disturbances.
\newblock In \emph{International Conference on Machine Learning}, pages
  111--119. PMLR, 2019.

\bibitem[Hazan et~al.(2020)Hazan, Kakade, and Singh]{hazan2020nonstochastic}
Elad Hazan, Sham Kakade, and Karan Singh.
\newblock The nonstochastic control problem.
\newblock In \emph{Algorithmic Learning Theory}, pages 408--421. PMLR, 2020.

\bibitem[Chen and Hazan(2020)]{chen2020black}
Xinyi Chen and Elad Hazan.
\newblock Black-box control for linear dynamical systems.
\newblock \emph{arXiv preprint arXiv:2007.06650}, 2020.

\bibitem[Ghai et~al.(2020)Ghai, Lee, Singh, Zhang, and Zhang]{ghai2020no}
Udaya Ghai, Holden Lee, Karan Singh, Cyril Zhang, and Yi~Zhang.
\newblock No-regret prediction in marginally stable systems.
\newblock In \emph{Conference on Learning Theory}, pages 1714--1757. PMLR,
  2020.

\bibitem[Rashidinejad et~al.(2020)Rashidinejad, Jiao, and
  Russell]{rashidinejad2020slip}
Paria Rashidinejad, Jiantao Jiao, and Stuart Russell.
\newblock {SLIP: Learning to predict in unknown dynamical systems with
  long-term memory}.
\newblock \emph{arXiv preprint arXiv:2010.05899}, 2020.

\bibitem[Kozdoba et~al.(2019)Kozdoba, Marecek, Tchrakian, and
  Mannor]{kozdoba2019line}
Mark Kozdoba, Jakub Marecek, Tigran Tchrakian, and Shie Mannor.
\newblock {On-line learning of linear dynamical systems: Exponential forgetting
  in kalman filters}.
\newblock In \emph{Proceedings of the AAAI Conference on Artificial
  Intelligence}, volume~33, pages 4098--4105, 2019.

\bibitem[Hazan et~al.(2017)Hazan, Singh, and Zhang]{hazan2017learning}
Elad Hazan, Karan Singh, and Cyril Zhang.
\newblock Learning linear dynamical systems via spectral filtering.
\newblock \emph{Advances in Neural Information Processing Systems},
  30:\penalty0 6702--6712, 2017.

\bibitem[Tsiamis et~al.(2020)Tsiamis, Matni, and Pappas]{tsiamis2020sample}
Anastasios Tsiamis, Nikolai Matni, and George Pappas.
\newblock Sample complexity of kalman filtering for unknown systems.
\newblock In \emph{Learning for Dynamics and Control}, pages 435--444. PMLR,
  2020.

\bibitem[Tsiamis and Pappas(2020)]{tsiamis2020online}
Anastasios Tsiamis and George Pappas.
\newblock Online learning of the kalman filter with logarithmic regret.
\newblock \emph{arXiv preprint arXiv:2002.05141}, 2020.

\bibitem[Kuznetsov and Mohri(2016)]{kuznetsov2016time}
Vitaly Kuznetsov and Mehryar Mohri.
\newblock Time series prediction and online learning.
\newblock In \emph{Conference on Learning Theory}, pages 1190--1213. PMLR,
  2016.

\bibitem[Tsiamis and Pappas(2021)]{tsiamis2021linear}
Anastasios Tsiamis and George~J Pappas.
\newblock Linear systems can be hard to learn.
\newblock \emph{arXiv preprint arXiv:2104.01120}, 2021.

\bibitem[Duchi et~al.(2012)Duchi, Agarwal, Johansson, and Jordan]{DuchiAJJ12}
John~C. Duchi, Alekh Agarwal, Mikael Johansson, and Michael~I. Jordan.
\newblock {Ergodic Mirror Descent}.
\newblock \emph{{SIAM} Journal on Optimization}, 22\penalty0 (4):\penalty0
  1549--1578, 2012.
\newblock \doi{10.1137/110836043}.
\newblock URL \url{https://doi.org/10.1137/110836043}.

\bibitem[Lin(1992)]{lin1992self}
Long-Ji Lin.
\newblock Self-improving reactive agents based on reinforcement learning,
  planning and teaching.
\newblock \emph{Machine learning}, 8\penalty0 (3-4):\penalty0 293--321, 1992.

\bibitem[D{\'e}fossez and Bach(2015)]{defossez2015averaged}
Alexandre D{\'e}fossez and Francis Bach.
\newblock {Averaged least-mean-squares: Bias-variance trade-offs and optimal
  sampling distributions}.
\newblock In \emph{Artificial Intelligence and Statistics}, pages 205--213.
  PMLR, 2015.

\bibitem[Petrov(2016)]{petrovgelfand}
Fedor Petrov.
\newblock {Non-asympototic version of Gelfand's formula}.
\newblock MathOverflow, 2016.
\newblock URL \url{https://mathoverflow.net/q/228561}.

\bibitem[Boucheron et~al.(2013)Boucheron, Lugosi, and
  Massart]{boucheron2013concentration}
St{\'e}phane Boucheron, G{\'a}bor Lugosi, and Pascal Massart.
\newblock \emph{{Concentration inequalities: A nonasymptotic theory of
  independence}}.
\newblock Oxford university press, 2013.

\bibitem[Vershynin(2018)]{vershynin2018high}
Roman Vershynin.
\newblock \emph{{High-dimensional probability: An introduction with
  applications in data science}}, volume~47.
\newblock Cambridge university press, 2018.

\bibitem[Lindvall(2002)]{lindvall2002lectures}
Torgny Lindvall.
\newblock \emph{Lectures on the coupling method}.
\newblock Courier Corporation, 2002.

\bibitem[Szarek(1982)]{szarek1982nets}
Stanislaw~J Szarek.
\newblock {Nets of Grassmann manifold and orthogonal group}.
\newblock In \emph{Proceedings of research workshop on Banach space theory
  (Iowa City, Iowa, 1981)}, volume 169, page 185, 1982.

\bibitem[Cai et~al.(2013)Cai, Ma, Wu, et~al.]{cai2013sparse}
T~Tony Cai, Zongming Ma, Yihong Wu, et~al.
\newblock {Sparse PCA: Optimal rates and adaptive estimation}.
\newblock \emph{The Annals of Statistics}, 41\penalty0 (6):\penalty0
  3074--3110, 2013.

\end{thebibliography}

\clearpage

\clearpage
\appendix


\section*{Organization of the appendix}
We provide a map of the results in the appendix. 
\begin{enumerate}
\item In section~\ref{sec:formal_proof_sketch} we provide formal statements of theorems~\ref{thm:op_informal} and \ref{thm:pred_informal}. We also discuss the more general spectral gap condition $\max_i |\lambda_i(A)|<1$ instead of the stronger condition $\norm{A}<1$ and its impact on the results. 
\item In section~\ref{subsec:basic_lemmas} we construct the coupled process $\tilde{X}_t$ and setup notations used in the rest of the paper. The coupled process has the additional property that the successive buffers are independent.
\item In section~\ref{sec:initial_coupling} we show that the $\sgdber$ iterates generated using the coupled process are close to ones generated by the actual data. After this, we only deal with the coupled iterates.
\item In section~\ref{sec:bias_variance} we provide the bias-variance decomposition
\item In section~\ref{sec:main_op_bound} we provide the proof of the parameter error bound of theorem~\ref{thm:op_informal}. Required intermediary results are discussed in section~\ref{sec:op_norm}.
\item In section~\ref{sec:bias_var_analysis} we present the bounds on the bias and variance terms separately (for last and average iterates), which are necessary to prove theorem~\ref{thm:main_upper_bound_pred}. Most of the proofs are relegated to sections \ref{sec:proof_prop_1}, \ref{sec:prop_avg_var1}, \ref{sec:proof_thm_last_bias}, \ref{sec:proof_prop_avg_bias} and \ref{sec:technical_proofs}.
\item In section~\ref{sec:pred_loss}  we prove theorem~\ref{thm:pred_informal}. 
\item In section~\ref{sec:lower_bounds}, we prove the lower bounds for the prediction error given in theorem~\ref{thm:main_lower_bound}.
\item In section~\ref{sec:sparse_system} we discuss the scenario of $\var(\A,\mu)$ where $\A$ is sparse with known sparsity pattern. We provide a proof sketch of the bound on prediction error in terms of sparsity. 
\end{enumerate}
\section{Formal Results and Proof Sketch}

\label{sec:formal_proof_sketch}
In this Section, we formally state the full results and sketch the outline of our proof. Recall the definitions of  $\lossop$ and $\losspred$ from section~\ref{sec:prob}. For all the theorems below, we suppose that Assumptions~\ref{as:norm_condition},~\ref{as:noise_concentration} and~\ref{as:stationarity} hold.  Assume that $u,\gamma,B,\alpha$ and $R$ are as chosen in section~\ref{sec:main_results}. 


Let $t>a$ and let $\hat{A}_{a,t}$ be the tail averaged output of $\sgdber$ after buffer $t-1$. Further let $T^{\alpha/2}>cd\kappa(G)$.

%
%
%

\begin{theorem}\label{thm:main_upper_bound_op}
Suppose we pick the step size $\gamma = \min \left(\tfrac{C}{B\lmin{G}},\tfrac{1}{8BR}\right)$ for some constant $C$ depending only on $C_{\mu}$.  
Then, there are constants $C,c_i>0,\,0\leq i\leq 4$ such that if $a>c_0\left(d+\alpha\log T\right)$  then with probability at least $1-\frac{C}{T^{\alpha}}$, we have: 

\begin{\Ieee}{LLL}
\label{eq:main_upper_bound_op}
\lossop(\hat{A}_{a,t},\A,\mu)\leq c_1\sqrt{\frac{(d+\alpha\log T) \sigma_{\max}(\Sigma)}{(t-a) B  \lmin{G}} }+\beta_b\norm{A_0-\A}+c_4\frac{T^2}{B^2} \norm{\A^u}\Ieeen
\end{\Ieee}
where 
\begin{equation}
\label{eq:beta_b_defn}
\beta_b=c_3 \frac{d\kappa(G)\log T}{t-a}e^{-c_2\frac{a}{d\kappa(G)\log T}}
\end{equation}

\end{theorem}

The techniques for the proof is developed in Section~\ref{sec:op_norm} and the Theorem~\ref{thm:main_upper_bound_op} is proved in Section~\ref{sec:main_op_bound}.
\begin{theorem}
\label{thm:main_upper_bound_pred}
Let $R,B,u,\alpha$ be chosen as in section~\ref{sec:main_results}. Let $\gamma = \frac{c}{4RB}\leq \frac{1}{2R}$ for $0<c<1$. Then there are constants $c_1,c_2,c_3,c_4>0$ such that for $\prbndsqinv>c_1\frac{\sqrt{M_4}}{\sigma_{\min}(G)}$ the expected prediction loss $\losspred$ is bounded as
\begin{\Ieee}{LLL}
\label{eq:main_upper_bound_pred}
\Ex{\losspred(\hat{A}_{a,t};\A,\mu)}-\tr(\Sigma)&\leq & c_2\left[ \frac{d\tr(\Sigma)}{B(t-a)}+ \frac{d^2\sigma_{\max}(\Sigma)}{B(t-a)}\frac{\sqrt{\kappa(G)}}{B}\right]+\\
&& c_3\left[\frac{d^2\sigma_{\max}(\Sigma)}{B^2(t-a)^2}(\kappa(G))^{3/2}dB\log T+\right.\\
&&\left. \beta_b\tr(G)\norm{A_0-\A}^2+\right.\\
&& \left.\left(  \frac{T^3}{B^3}\norm{\A^u}+\frac{ d\sigma_{\max}(\Sigma)}{R} \frac{T^2}{B^2}\prbndsq\right)\tr(G)\right]\\
\Ieeen
\end{\Ieee}

where $\beta_b$ is defined in \eqref{eq:beta_b_defn}.


\end{theorem}

The above theorem is proven only for the case $t=N$. The proof for general $t$ is almost the same. The proof follows by first considering $\Ex{\losspred(\hat{A}_{a,N};\A,\mu)\ind{0}{N-1}}$ ($\cd^{0,N-1}$ is defined in \ref{subsec:notations}) and using theorem~\ref{thm:predloss_vari} and theorem~\ref{thm:pred_bias} along with lemma~\ref{lem:coupling_AA^t} in the appendix sections~\ref{sec:pred_var}, \ref{sec:pred_bias} and \ref{sec:initial_coupling}. Then noting that if the norm of any of the covariates $X_t$ exceed $\sqrt{R}$ the algorithm returns the zero matrix we have that $\Ex{\losspred(\hat{A}_{a,N};\A,\mu)\indc{0}{N-1}}\leq c\norm{\A}\tr(G)\prbnd$.

\begin{remark}\rm
\noindent
\begin{enumerate}[label=(\arabic*)]
\item In theorem~\ref{thm:main_upper_bound_pred} the term $\frac{d^2\sigma_{\max}(\Sigma)}{B(t-a)}\frac{\sqrt{\kappa(G)}}{B}$ is strictly a lower order term compared to $\frac{d\tr(\Sigma)}{B(t-a)}$ when $\norm{\A}<c_0<1$. To see this note that $\sigma_{\max}(G)\leq \frac{\sigma_{\max}(\Sigma)}{1-\norm{\Aa}^2}$ and $\sigma_{\min}(G)\geq\sigma_{\min}(\Sigma)$. Hence $\kappa(G)\leq \frac{\kappa(\Sigma)}{1-\norm{\Aa}^2}=O(\tau_{\mathsf{mix}}\kappa(\Sigma))$. By the choice of $B$ in the section~\ref{sec:main_results} we see that $\frac{\sqrt{\kappa(G)}}{B}=o(1)$ and it \emph{does not depend on condition number of $\A$}.
\item If $a=\Omega\left(d\kappa(G)\left(\log T\right)^2\right)$ the $\beta_b$ is a lower order term. Further choosing $u$ and $\alpha$ as in section~\ref{sec:main_results} we see that the terms depending on $\norm{\A^u}$ and $\prbndsq$ are strictly lower order.
\item Thus for the choice of $a$ as in the previous remark such that $a<(1+c)t$ (for some $c>0$), we get minimax optimal rates: $\frac{d\tr(\Sigma)}{Bt}$ for $\losspred$ and up to log factors, $\sqrt{\frac{d\sigma_{\max}(\Sigma)}{T\sigma_{\min}(G)}}$ for $\lossop$
\end{enumerate}
\end{remark}

\subsection{Spectral Gap Condition}\label{subsec:spectral_gap}
In Assumption~\ref{as:norm_condition}, we could have used the more general spectral radius condition $\rho(\A) = \sup_i |\lambda_i(\A)| < 1$ rather than the one on the operator norm.  We have the Gelfand formula for spectral radius which shows that $\lim_{k\to \infty}\|\A^k\|^{1/k} = \rho(\A)$. Now, if $\A$ is such that $\rho(\A)<1$ but $\|\A\|>1$ (a case studied by \citep{simchowitz2018learning}), then we need to make $u$ as large as $C d\log T$ which would lead to a relatively large buffer size $B$ of $d\log T$. To see this, we verify the proof by \citep{petrovgelfand} (by replacing $A$ with $\frac{A}{\|A\|}$ and $\rho(A)$ with $\frac{\rho}{\|A\| }$ in the proof) to show that $\|\A^{k}\| \leq \left(2k\|\A\|\right)^d\rho^{k-d}  $ whenever $k \geq d$. Therefore, in the worst case, we can pick $u = O\bigr(\left(\log \left(T\lmax{G}\right) + d\log d\|A\| \right)/\log 1/\rho\bigr)$.

In the case of $\rho<1$ but $\norm{\A}>1$, $\kappa(G)$ can grow super linearly in $d$. For instance, consider $\A$ to be nilpotent of order $d$ (i.e. $\A^{d-1}\neq 0$ but $\A^d=0$). Here $\sigma_{\max}(G)$ can grow like $\norm{\A}^d$. So we need exponentially (in $d$) many samples for bias decay. However, in many cases of interest (ex: symmetric matrices, normal matrices etc) the spectral radius is the same as the operator norm.

\section{Basic Lemmas and Notations}
\label{subsec:basic_lemmas}

Since the covariates $\{X_\tau\}_{\tau \leq T}$ are correlated, we will introduce a coupled process such that we have independence across buffers and that Euclidean distance between the covariates of the original process and the coupled process can be controlled. 

\begin{remark}
Note that the coupled process is imaginary and we do not actually run the algorithm with the coupled process. We construct it to make the analysis simple by first analyzing the algorithm with the imaginary coupled process and then showing that the output of the actual algorithm cannot deviate too much when run with the actual data. 
\end{remark}
\begin{define}[Coupled process]
\label{def:1}
Given the covariates $\{X_\tau:\tau=0,1,.\cdots T\}$ and noise $\{\eta_\tau:\tau=0,1,\cdots,T\}$, we define $\{\tilde{X}_\tau:\tau=0,1,\cdots,T\}$ as follows:
\begin{enumerate}
\item For each buffer $t$ generate, independently of everything else, $\tilde{X}^t_0\distas{}\pi$, the stationary distribution of the $\var(\A,\mu)$ model.

\item Then, each buffer has the same recursion as eq \eqref{eq:var1}:
\begin{equation}
\label{eq:coupling}
\tilde{X}^t_{i+1}=\A \tilde{X}^t_i+\eta^t_i,\, i=0,1,\cdots S-1, 
\end{equation}
where the noise vectors as same as in the actual process $\{X_\tau\}$.
\end{enumerate}
\end{define}

\noindent With this definition, we have the following lemma: 
\begin{lemma}
\label{lem:1}
For any buffer $t$, $\norms{\Xt{t}{i}-\tilde{X}^t_i}\leq \norms{\A^i}\norms{\Xt{t}{0}-\tilde{X}^t_0},\, a.s.$. That is, 
\begin{equation}
\label{eq:contraction}
\norms{\Xt{t}{i}\Xt{t}{i}^T-\tilde{X}^t_i \Xtt{t}{i}^T}\leq 2\nx\norms{\Xt{t}{i}-\Xtt{t}{i}}\leq \nxx^2\norms{\A^i}. 
\end{equation}
Here $\|X\|$ denotes $\sup_{\tau \leq T} \|X_{\tau}\|$.
\end{lemma}

%

\begin{lemma}\label{lem:data_subgaussianity}
Suppose $\mu$ obeys Assumption~\ref{as:noise_concentration} and $\A$ obeys Assumption~\ref{as:norm_condition}. Suppose $X \sim \pi$, which is the stationary distribution of $\var(\A,\mu)$. 
$\langle X, x\rangle$ has mean $0$ and is sub-Gaussian with variance proxy $C_{\mu}x^{\top} Gx$
\end{lemma}
\begin{proof}
Suppose $\eta_1,\dots,\eta_n,\dots$ is a sequence of i.i.d random vectors drawn from the noise distribution $\mu$. We consider the partial sums $\sum_{i=0}^{n}\A^{i}\eta_{i}$. Call the law of this to be $\pi_n$. Clearly $\pi_n$ converges in distribution to $\pi$ as $n\to \infty$ since $\pi_n$ is the law of the $n+1$-th iterate of $\var(\A,\mu)$ chain stated at $X_0 = 0$. By Skorokhod representation theorem, we can define the infinite sequence $X^{(1)},\dots, X^{(n)},\dots,$ and another random variable $X$ such that $X^{(i)} \sim \pi_i$, $X \sim \pi$ and $\lim_{n\to \infty } X^{(n)} = X $ a.s. Define $G_n = \sum_{i=0}^{n} \A^i \Sigma (\A^i)^{T}$. Clearly, $G_n \preceq G = \sum_{i=0}^{\infty}\A^i \Sigma (\A^i)^{T}$. A simple evaluation of Chernoff bound for $\langle X^{(n)},x \rangle$ by decomposing it into the partial sum of noises shows that:
$$\mathbb{E} \exp(\lambda \langle X^{(n)},x \rangle) \leq \exp\left(\frac{\lambda^2 C_{\mu}}{2} \langle x,G_nx\rangle\right) \leq \exp\left(\frac{\lambda^2 C_{\mu}}{2} \langle x,Gx\rangle\right)$$

We now apply Fatou's lemma, since $X^{(n)} \to X$ almost surely, to the inequality above to conclude that:
$$\mathbb{E} \exp(\lambda \langle X,x \rangle) \leq \exp\left(\frac{\lambda^2 C_{\mu}}{2} \langle x,Gx\rangle\right).	$$
\end{proof}

Hence $\langle x,X_t\rangle$ is subgaussian with mean $0$ and variance proxy $C_\mu \sigma_{\max}(G)\norm{x}^2$. This will provide uniform variance for all $x$ such that $\norm{x}^2=1$.

From subgaussianity and standard $\epsilon$-net argument we have the following lemma.

\begin{lemma}
\label{lem:2}
For any $\beta>0$ there is a constant $c>0$ such that 
\begin{\Ieee}{LLL}
\label{eq:conc_ineq_norm_X}
\Pb{\exists \tau \leq T \,:\,\norm{X_\tau}^2>c\tr{G}\log T }\leq \frac{d}{T^{\beta}}\Ieeen
\end{\Ieee}
Thus as long as $d<\poly(T)$, for every $\alpha>0$ there is a $c>0$ such that
\begin{\Ieee}{LLL}
\label{eq:conc_ineq_norm_X_1}
\Pb{\exists \tau \leq T \,:\,\norm{X_\tau}^2>c\tr{G}\log T }\leq\prbnd\Ieeen
\end{\Ieee}
\end{lemma}

%
\subsection{Notations}
\label{subsec:notations}

Before we analyze this algorithm, we define some notations. We work in a probability space $(\Omega,\mathcal{F},\mathbb{P})$ and all the random elements are defined on this space. We define the following notations:
\begin{align*}
&X^t_{-i}=X^t_{(S-1)-i},\, 0\leq i\leq S-1, \quad
G=\sum_{s=0}^{\infty}\A^s \Sigma(\A^{\top})^s, \quad 
G_t=\sum_{s=0}^{t-1}\A^s \Sigma(\A^{\top})^s,\\
&\Ppt{t}{i}=\Ptt{t}{i},\quad 
\Htt{t}{i}{j}=\begin{cases}\prod_{s=i}^{j}\Ppt{t}{-s} & i\leq j\\
	I & i>j 
\end{cases},\\
&\hat{\gamma}=4\gamma(1-\gamma R),\quad 
\cc^t_{-j}=\left\{\norms{X^{t}_{-j}}^2\leq R\right\},\quad
\cct^{t}_{-j}=\left\{\norms{\tilde X^{t}_{-j}}^2\leq R\right\},\\
&\cd^t_{-j}=\left\{\norms{X^{t}_{-i}}^2\leq R:\,j\leq i\leq B-1\right\}=\bigcap_{i=j}^{B-1}\cc^t_{-i},\\
&\cd^{s,t}=\begin{cases} \bigcap_{r=s}^t \cd^r_{-0} & s\leq t\\
\Omega & s>t 
\end{cases},\quad
\tilde\cd^t_{-j}=\left\{\norms{\tilde X^{t}_{-i}}^2\leq R:\,j\leq i\leq B-1\right\}=\bigcap_{i=j}^{B-1}\cct^t_{-i},\\
&\tilde \cd^{s,t}=\begin{cases} \bigcap_{r=s}^t \tilde \cd^r_{-0} & s\leq t\\
\Omega & s>t 
\end{cases},\quad
\cdh^t_{-j}=\cd^{t}_{-j}\cap \cdt^t_{-j},\quad
\cdh^{s,t}=\cd^{s,t}\cap \cdt^{s,t}.
\end{align*}

Lastly $c$ and $c_i$ for $i=0,1,\cdots$ denote absolute constants that can change from line to line in the proofs. 

\section{Initial Coupling}
\label{sec:initial_coupling}
We consider the coupled process introduced in Definition~\ref{def:1} and run $\sgdber$ with the fictitious coupled process $\tilde{X}_{\tau}$ instead of $X_{\tau}$ in order to obtain the iterates $\tilde{A}^{t}_i$ instead of $\At{t-1}{i}$. Using Lemma~\ref{lem:1}, we can show that $\tilde{A}^{t-1}_i \approx \At{t-1}{i}$. It is easier to analyze the iterates $\tilde{A}^{t}_i$ due to buffer independence.  

\begin{lemma}
\label{lem:bounded_iterates}

 Let $\gamma \leq \frac{1}{2R}$. Under the event $\cd^{0,N-1}$, for every $t \in [N]$ and $0\leq i\leq B-1$ we have:
$$\|A^{t-1}_i\| \leq 2\gamma R T  \,.$$
\end{lemma}
%
\begin{lemma}
\label{lem:coupled_iterate_replacement}
 Suppose $\gamma < \frac{1}{2R}$. Under the event $\cdh^{0,N-1}$ we have for every $t \in [N]$ and $0\leq i \leq B-1$.
$\|\At{t-1}{i} - \Att{t-1}{i}\| \leq (16\gamma^2R^2T^2 + 8\gamma RT)\norm{\A^u}$ 
\end{lemma}

 We can now just analyze the iterates $\Att{t-1}{i}$ and then use Lemma~\ref{lem:coupled_iterate_replacement} to infer error bounds for $\At{t-1}{i}$. Henceforth, we will only consider $\Att{t-1}{i}$. 

\begin{lemma}
\label{lem:coupling_AA^t}
Consider the algorithmic iterates obtained from the actual process and coupled process $(\At{t}{j})$ and $(\Att{t}{j})$. Then
\begin{\Ieee}{LLL}
\label{eq:coupling_AA^t}
\Ex{\gram{\left(\At{t-1}{j}-\A\right)}\ind{0}{t-1}}\preceq \Ex{\gram{\left(\Att{t-1}{j}-\A\right)}\indt{0}{t-1}}\\
+c\left(\gamma^3 R^3 T^3\norm{\A^u}+\gamma^2 d\sigma_{\max}(\Sigma) R T^2\prbndsq\right) I\Ieeen
\end{\Ieee}
for some constant $c$. 
Furthermore, the same conclusion holds for the average iterates. That is let 
\begin{\Ieee}{LLL}
\label{eq:coupling_AA^t_a}
\Ana=\frac{1}{N-a}\sum_{t=a+1}^N\At{t-1}{B}\\
\Anat=\frac{1}{N-a}\sum_{t=a+1}^N\Att{t-1}{B}
\end{\Ieee}
Then
\begin{\Ieee}{LLL}
\label{eq:coupling_AA^t_b}
&&\Ex{\gram{\left(\Ana-\A\right)}\ind{0}{N-1}}\\
&\preceq & \Ex{\gram{\left(\Anat-\A\right)}\indt{0}{N-1}}\\
&&+c\left(\gamma^3 R^3 T^3\norm{\A^u}+\gamma^2 d\sigma_{\max}(\Sigma) R T^2\prbndsq\right) I\Ieeen
\end{\Ieee}

\end{lemma}
\begin{remark}
\label{rem:coupling_AA^t}
The above lemma holds as is when $\At{t-1}{j},\Att{t-1}{j}$ is replaced by $\Av{t-1}{j},\Avt{t-1}{j}$ respectively. 
\end{remark}

We refer to Section~\ref{sec:technical_proofs} for the proofs of the three lemmas.
%
%
%

\section{Bias Variance Decomposition}
\label{sec:bias_variance}
 Now, we can unroll the recursion in \eqref{eq:sgd_expreplay2}, but for the coupled iterates $\Att{t-1}{i}$ as 
\begin{equation}
\label{eq:expreplay3}
\Att{t-1}{B}-\A=\Atto{t-1}{B}+\Attdiff{t-1}{B}, 
\end{equation}
where 
\begin{equation}
\label{eq:At0}
\Atto{t-1}{B}=(A_0-\A)\prod_{s=0}^{t-1}\Htt{s}{0}{B-1}
\end{equation}
is the {\em bias} term, and the {\em variance} term is given by:  
\begin{equation}
\label{eq:Atv}
\Attdiff{t-1}{B}=2\gamma\sum_{r=1}^{t}\sum_{j=0}^{B-1}\Nt{t-r}{-j}\Xtttr{t-r}{-j}\Htt{t-r}{j+1}{B-1}{\prod_{s=r-1}^{1}\Htt{t-s}{0}{B-1}} 
\end{equation}

Here we use the convention that whenever $r = 1$, the product $\prod_{s=r-1}^{1}$ is empty i.e, equal to $1$.  The `bias' term is obtained when the noise terms are set to $0$, and captures the movement of the algorithm towards the optimal $\A$ when we set the initial iterate far away from it. The `variance' term $\Atdiff{t}{B}$ capture the uncertainty due to the inherent noise in the data. Our main goal is to understand the performance (estimation and prediction) of the tail-averaged iterates output by $\sgdber$. Here, we consider just the last iterate, but the same technique applies to all the outputs of $\sgdber$. That is, $\Anat=\frac{1}{N-a}\sum_{t=a+1}^N \Att{t-1}{B}$, for $a=\lceil\theta N \rceil$ with $0<\theta <1$. We can decompose the above into bias and variance as: $\Anat=\Anavt+\Anabt$, with, 
\begin{\Ieee}{LLL}
\Anavt=\frac{1}{N-a}\sum_{t=a+1}^N \Avt{t-1}{B}\Ieeen\label{eq:tail_variance}\\\\
\Anabt=\frac{1}{N-a}\sum_{t=a+1}^N \Abt{t-1}{B}.\Ieeen\label{eq:tail_bias}\\
\end{\Ieee}
Similarly, we can decompose the final error into `bias' and `variance' as in Lemma~\ref{lem:3} below.
\begin{lemma}[Bias-Variance Decomposition]
\label{lem:3}
 We have the following decomposition: 
\begin{\Ieee}{LLL}
\label{eq:bias_var}
\gram{\left(\Att{t-1}{B}-\A\right)} &\preceq &  2\left[\gram{ \Atto{t-1}{B}}+\right.\\
&& \left. \gram{ \Attdiff{t-1}{B}} \right].
\end{\Ieee}
\end{lemma}

\section{Parameter Error Bound--Proof of Theorem~\ref{thm:main_upper_bound_op}}
\label{sec:main_op_bound}
In this section, we formally prove the bounds on $\lossop(;\A,\mu)$, by combining several operator norm inequalities that we prove in Section~\ref{sec:op_norm}. As mentioned previously, we will just focus on the algorithmic iterates from the coupled process $(\Att{t-1}{j})$. Recall the output $\Att{t-1}{B}$ after the $t-1$-th buffer from Equation~\eqref{eq:expreplay3}. For any initial buffer index $a \in \{0,1,\dots,N-1\}$, the tail averaged output of our algorithm is:
$$\Anat := \frac{1}{N-a}\sum_{t=a+1}^{N} \Att{t-1}{B}.$$
Recall the quantities $\Att{t-1,v}{B}$ and $\Att{t-1,b}{B}$ as defined in \eqref{eq:At0} and \eqref{eq:Atv}. We can use this decomposition to write:
$$\Anat-\A=\Anabt-\A+\Anavt.$$
Here $\Anabt-\A := \frac{1}{N-a}\sum_{t=a+1}^{N} \Atto{t-1}{B}$ denotes the bias part and $\Anavt := \frac{1}{N-a}\sum_{t=a+1}^{N} \Attdiff{t-1}{B}$ denotes the variance part. 

\subsection{Variance}
Note that
\begin{\Ieee}{LLL}
\label{eq:tail_decomposition}
\Anavt = \frac{N}{N-a} \left(\Anovt\right)- \frac{a}{N-a}\left(\Anoavt\right) \Ieeen
\end{\Ieee}

Now, we apply Theorem~\ref{thm:average_iterate_op_bound} with $\delta$ in the definition of $\tilde{\mathcal{M}}^{0,N-1}$ to be $\frac{1}{T^{\upsilon}}$ for some fixed $\upsilon \geq 1$.  We conclude that conditioned on the event $\tilde{\mathcal{M}}^{0,N-1}\cap \cdt^{0,N-1}$, with probability at least $1- \frac{1}{T^{\upsilon}}$, we have: 

$$\|\Anovt \| \leq C \sqrt{\frac{\gamma (d+\upsilon\log T)^2 \lmax{\Sigma}}{N} } + C \sqrt{\frac{(d+\upsilon\log T) \sigma_{\max}(\Sigma)}{N B  \lmin{G}} }\,.$$

Similarly, applying Theorem~\ref{thm:average_iterate_op_bound} with $N = a$ shows that with probability at least $1- \frac{1}{T^{\upsilon}}$ conditioned on the event $\tilde{\mathcal{M}}^{0,N-1}\cap \cdt^{0,N-1}$:
$$\|\Anoavt \| \leq C \sqrt{\frac{\gamma (d+\upsilon\log T)^2 \lmax{\Sigma}}{a} } + C \sqrt{\frac{(d+\upsilon\log T) \sigma_{\max}(\Sigma)}{aB  \lmin{G}} }\,.$$

Here, the constant $C$ depends only on $C_{\mu}$. We also note that when we pick $\gamma B R \leq C_0$ where $R \gtrsim \tr(G)+ \upsilon\log T$, the first term in the equations above becomes smaller than the second term. Therefore, under this assumption we can simplify the expressions to:
\begin{equation}\label{eq:first_triangle}
	\|\Anovt \| \leq  C \sqrt{\frac{(d+\upsilon\log T) \sigma_{\max}(\Sigma)}{N B  \lmin{G}} }\,.
\end{equation}

\begin{equation}\label{eq:second_triangle}
	\|\Anoavt \| \leq  C \sqrt{\frac{(d+\upsilon\log T) \sigma_{\max}(\Sigma)}{aB  \lmin{G}} } \,.
\end{equation}

Applying Equations~\eqref{eq:first_triangle} and~\eqref{eq:second_triangle} to Equation~\eqref{eq:tail_decomposition} we conclude that conditioned on the event $\tilde{\mathcal{M}}^{0,N-1}\cap \cdt^{0,N-1}$, with probability at least $1- \frac{2}{T^{\upsilon}}$, we have: 
\begin{align}
	\|\Anavt\| &\leq \frac{N}{N-a}\| \left(\Anovt\right)\|+ \frac{a}{N-a}\|\left(\Anoavt \right)\| \nonumber \\
	&\leq \frac{CN}{N-a} \sqrt{\frac{(d+\upsilon\log T) \sigma_{\max}(\Sigma)}{N B  \lmin{G}} } + \frac{C a}{N-a}\sqrt{\frac{(d+\upsilon\log T) \sigma_{\max}(\Sigma)}{aB  \lmin{G}} }.\label{eq:coupled_tail_bound}
\end{align}
Choose $a<N/2$. Since 
$$\Pb{\tilde{\mathcal{M}}^{0,N-1}\cap \cdt^{0,N-1}}\geq 1-(\frac{1}{T^{\upsilon}}+\prbnd)$$
 we have
\begin{\Ieee}{LLL}
\Pb{\|\Anavt\| > C\sqrt{\frac{(d+\upsilon\log T) \sigma_{\max}(\Sigma)}{(N-a) B  \lmin{G}} } }\\
\leq \prbnd+\frac{3}{T^{\upsilon}}\Ieeen\label{eq:coupled_tail_bound_new}
\end{\Ieee}

\subsection{Bias}
We now consider the bias term:$\Anabt -\A:= \frac{1}{N-a}\sum_{t=a+1}^{N} \Atto{t-1}{B}$. First note that, from equation~\eqref{eq:At0}, we have
\begin{\Ieee}{LLL}
\label{eq:bias_op_bnd_1}
\norm{\Anabt -\A}\leq \frac{1}{N-a}\sum_{t=a+1}^{N}\norm{A_0-\A}\norm{\prod_{s=0}^{t-1}\Htt{s}{0}{B-1}}\Ieeen
\end{\Ieee}

Now from lemma~\ref{lem:operator_norm_bound_1}, if $a>c_1\left(d+\log \frac{N}{\delta}\right)$ then conditional on $\cdt^{0,N-1}$ with probability at least $1-\delta$, for all $a+1\leq t\leq N$ we have
\begin{\Ieee}{LLL}
\label{eq:bias_op_bnd_2}
\norm{\prod_{s=0}^{t-1}\Htt{s}{0}{B-1}}\leq 2\left(1-\gamma B\sigma_{\min}(G)\right)^{c_2 t}\Ieeen
\end{\Ieee}
Note that in lemma~\ref{lem:operator_norm_bound_1} we only condition on $\cdt^{0,t-1}$ but due to buffer independence and that $\Pb{\cdt^{0,N-1}}\geq 1-\prbnd$ we can condition on $\cdt^{0,N-1}$. 

Note that in the proof of lemma~\ref{lem:operator_norm_bound_1} the constant $c_2$ is actually at most $1$ i.e., $0<c_2\leq 1$. Hence from Bernoulli's inequality, for $x<1$ 
$$(1-x)^{c_2}\leq 1-c_2 x$$

Thus conditional on $\cdt^{0,N-1}$ with probability at least $1-\delta$
\begin{\Ieee}{LLL}
\label{eq:bias_op_bnd_3}
\norm{\Anabt -\A}&\leq & \frac{\norm{A_0-\A}}{N-a}\sum_{t=a+1}^{\infty}2\left(1-\gamma B\sigma_{\min}(G)\right)^{c_2 t}\\
&=& 2\frac{\norm{A_0-\A}}{N-a}\frac{\left(1-\gamma B\sigma_{\min}(G)\right)^{c_2 a}}{c_2\gamma B\sigma_{\min}(G)}\\
&\leq & c_3 \frac{\norm{A_0-\A}}{N-a}\frac{e^{-c_2 a\gamma B\sigma_{\min}(G)}}{\gamma B\sigma_{\min}(G)}\Ieeen
\end{\Ieee}

Hence choosing $\delta=\frac{1}{T^{\upsilon}}$ we have for $a>c_1\left(d+\log \frac{N}{\delta}\right)$ 
\begin{\Ieee}{LLL}
\label{eq:bias_op_bnd_4}
\Pb{\norm{\Anabt -\A}>c_3 \frac{\norm{A_0-\A}}{N-a}\frac{e^{-c_2 a\gamma B\sigma_{\min}(G)}}{\gamma B\sigma_{\min}(G)}}\leq \prbnd+\frac{1}{T^{\upsilon}}\Ieeen
\end{\Ieee}

Define $\beta_b$ as 
\begin{\Ieee}{LLL}
\label{eq:bias_op_bnd_5}
\beta_b=c_3 \frac{1}{N-a}\frac{e^{-c_2 a\gamma B\sigma_{\min}(G)}}{\gamma B\sigma_{\min}(G)}\Ieeen
\end{\Ieee}
Thus by union bound and equations \eqref{eq:coupled_tail_bound_new} and \eqref{eq:bias_op_bnd_4} we get
\begin{\Ieee}{LLL}
\label{eq:op_bnd_main_1}
\Pb{\norm{\Anat-\A}>C\sqrt{\frac{(d+\upsilon\log T) \sigma_{\max}(\Sigma)}{(N-a) B  \lmin{G}} }+\beta_b\norm{A_0-\A}}\\
\leq \frac{2}{T^{\alpha}}+\frac{4}{T^{\upsilon}}\Ieeen
\end{\Ieee}

Now from lemma~\ref{lem:coupled_iterate_replacement} we see that on the event $\cdh^{0,N-1}$
\begin{\Ieee}{LLL}
\label{eq:avg_couple_diff}
\norm{\Ana-\Anat}\leq c\gamma^2 R^2 T^2\norm{\A^u}\Ieeen
\end{\Ieee}

Since $\Pb{\cdh^{0,N-1}}\geq 1-\prbnd$, we obtain
\begin{\Ieee}{LLL}
\label{eq:avg_couple_diff_1}
\Pb{\norm{\Ana-\Anat}\leq c\gamma^2 R^2 T^2\norm{\A^u}}\geq 1-\prbnd\Ieeen
\end{\Ieee}

Therefore choosing $\delta=\frac{1}{T^{\upsilon}}$ we have for $N/2>a>c_1\left(d+\log \frac{N}{\delta}\right)$ 
\begin{\Ieee}{LLL}
\label{eq:op_bnd_main_2}
\Pb{\norm{\Ana-\A}>C\sqrt{\frac{(d+\upsilon\log T) \sigma_{\max}(\Sigma)}{(N-a) B  \lmin{G}} }+\beta_b\norm{A_0-\A}+c_4\gamma^2 R^2 T^2 \norm{\A^u}}\\
\leq \frac{3}{T^{\alpha}}+\frac{4}{T^{\upsilon}}\Ieeen
\end{\Ieee}
where $\beta_b$ is defined in \eqref{eq:bias_op_bnd_5}. 

The theorem follows by adjusting the constants (in choosing $\delta$) such the above probability is at most $\frac{3}{T^{\alpha}}+\frac{1}{2T^{\upsilon}}$ and then choosing $\upsilon$ such that $\frac{3}{T^{\alpha}}\leq \frac{1}{2T^{\upsilon}}$.

\section{Bias Variance Analysis of Last and Average Iterate}
\label{sec:bias_var_analysis}
In this section, our goal is to provide a PSD upper bound on 
$$\Ex{\gram{\left(\Att{t-1}{B}-\A\right)}}, \Ex{\gram{\left(\Anat-\A\right)}}$$
using the bias variance decomposition in \eqref{eq:expreplay3} and \eqref{eq:tail_bias}. This bound leads to Theorem~\ref{th:4} which is critical for our parameter error proof (Theorem~\ref{thm:main_upper_bound_op}). 


\subsection{Variance of the Last Iterate}
\label{subsec:variance_last_iterate}

The goal of this section is to bound error due to $\Attdiff{t-1}{B}$. For brevity, we will introduce the following notation:

\begin{\Ieee}{LLL}
\Vt{t-1}=\Ex{\gram{\Attdiff{t-1}{B}} \indt{0}{t-1}}.\Ieeen\label{eq:Vt}
\end{\Ieee}

%

The following proposition is the main result of this section. 
\begin{proposition}
\label{prop:1}
Let $\gamma\leq \frac{1}{2R}$. Let the noise covariance be $\Ex{\eta_t \eta_t^T}=\Sigma$. Then, {\small
\begin{align*}
&\Vt{t-1}\preceq   \frac{\gamma \tr(\Sigma)}{1-\gamma R}\left[I-\Ex{\prodHtttr{t}{t}\prodHtt{t}{t}\indt{0}{t-1}}\right]  + c_1 \gamma^2 d\sigma_{\max}(\Sigma)(Bt)^2\prbndsq I,\\
&\Vt{t-1} \succeq  \gamma \tr(\Sigma)\left[I-\Ex{\prodHtttr{t}{t}\prodHtt{t}{t}\indt{0}{t-1}}\right] -c_4 \gamma^2 d\sigma_{\max}(\Sigma)(Bt)^2\prbndsq I,
\end{align*}}
for some absolute constants $c_i>0,\,1\leq i\leq 4 $.
\end{proposition}
We refer to Section~\ref{sec:proof_prop_1} in the appendix for a full proof. Note that we have, $\frac{1}{1-\gamma\nx^2}\leq 2$. 
\begin{corollary}
\label{coro:var_last_iter}
In the same setting as Proposition~\ref{prop:1}, we have: 
\begin{equation}
\label{eq:var_last_1b}
\Vt{t-1}\preceq c_1\gamma \tr(\Sigma)I+c_2 \gamma^2 d\sigma_{\max}(\Sigma)(Bt)^2\prbndsq I,
\end{equation}
for some constants $c_1,c_2>0$. If $\prbndsqinv>T^2$, then $V_{t,1}\preceq c \gamma d\sigma_{\max} I$, for some constant $c>0$.
\end{corollary}

\subsection{Variance of the Average Iterate}

\label{subsec:average_variance}
 In this section we are interested in bounding: $\Ex{\gram{\Anvat{a}{N}}\indt{0}{N-1}}$, for $a=\theta N$ with $0\leq \theta <1$, where, 
\begin{equation}
\label{eq:avg_iterate}
\Anavt=\frac{1}{N-a}\sum_{t=a+1}^{N}\Avt{t-1}{B},
\end{equation}
and further, recall that $T=N(B+u)$. The main bound in this section is given in Proposition~\ref{prop:avg_var1}. 
 Note that we have, 
\begin{\Ieee}{LLL}
\label{eq:An1}
&&\Ex{\gram{\Anvat{a}{N}}\indt{0}{N-1}}\\
&= &\frac{1}{(N-a)^2}\sum_{t=a+1}^N\Ex{\gram{\Attdiff{t-1}{B}}\indt{0}{N-1}}\\
&&+ \frac{1}{(N-a)^2}\sum_{t_1\neq t_2}\Ex{\Attdiff{t_1-1}{B}^{\top}\Attdiff{t_2-1}{B}\indt{0}{N-1}}\Ieeen
\end{\Ieee}

\begin{proposition}
\label{prop:avg_var1}
Let $\gamma\leq \min\{\frac{c}{6RB}\frac{1}{2R}\}$ for $0<c<1$. Then for $\Anavt$ defined in \eqref{eq:avg_iterate}, there are constants $c_1,c_2>0$ such that if $\prbndsqinv>c_1\frac{\sqrt{M_4}}{\sigma_{\min}(G)}$, then: 
\begin{\Ieee}{LLL}
&&\Ex{\gram{\Anvat{a}{N}}\indt{0}{N-1}}\\
&\preceq & \frac{1}{(N-a)^2} \sum_{t=a+1}^N \left[\Vt{t-1} \left(\sum_{s=0}^{N-t}\ch^s\right) + \left(\sum_{s=0}^{N-t}\ch^s\right)^{\top} \Vt{t-1}\right] + c_2\delta I\Ieeen\label{eq:avg_var_new1}\\
 &=&\frac{1}{(N-a)^2} \sum_{t=a+1}^N \left[\Vt{t-1} \left(I-\ch\right)^{-1}+ \left(I-\ch^{\top}\right)^{-1}\Vt{t-1} \right]+c_2\delta I +\\
 && \frac{1}{(N-a)^2} \sum_{t=a+1}^N \left[\Vt{t-1} \left(I-\ch\right)^{-1}\ch^{N-t+1}+ \left(\ch^{\top}\right)^{N-t+1}\left(I-\ch^{\top}\right)^{-1}\Vt{t-1} \right] \Ieeen\label{eq:avg_var_new1a}
\end{\Ieee}
and, 
\begin{\Ieee}{LLL}
\label{eq:avg_var_new2}
\delta\equiv \delta(N,B,R)= \gamma^2 T^2 R d\sigma_{\max}(\Sigma)\prbndsq \Ieeen
\end{\Ieee}
and $\ch$ is given by, 
\begin{\Ieee}{LLL}
\label{eq:avg_var_new3}
\ch=\Ex{\prod_{j=0}^{B-1}\left(I-2\gamma \Xtt{0}{-j} \Xtttr{0}{-j}\right)1\left[\cap_{j=0}^{B-1}\left\{\norms{\Xtt{0}{-j}}^2\leq R\right\}\right]},\Ieeen
\end{\Ieee}
with $\tilde X_0$ sampled from the stationary distribution $\pi$ and $\tilde X_t$ follows the $\var(\A,\mu)$. 
\end{proposition}
See section~\ref{sec:prop_avg_var1} in the appendix for the proof.

\subsection{Bias of the Last Iterate}
\label{subsec:bias_last_iterate}
In this we will analyze the bias term of the last iterate. That is we want to bound: $$\Ex{\gram{\Atto{t-1}{B}}\indt{0}{t-1}}\,.$$
 Where $\Atto{t-1}{B}$ is defined in \eqref{eq:At0}.

\begin{theorem}
\label{th:3}
Let $\gamma RB\leq \frac{c}{6}$ for some $0<c<1$ with $B$ such that $\gamma R\leq \frac{1}{2}$. Then there are constants $c_1,c_2,c_3>0$ such that if $\prbndsqinv>c_1\frac{\sqrt{M_4}}{\sigma_{\min}(G)}$ (where $M_4=\Ex{\norms{\Xtt{0}{-0}}^4}$) then
\begin{\Ieee}{LLL}
\label{eq:th3}
\Ex{\gram{\Atto{t-1}{B}}\indt{0}{t-1}}\preceq \norm{A_0-\A}^2 \left(1-c_2\gamma B\sigma_{\min}(G)\right)^{t}I\Ieeen
\end{\Ieee}
\end{theorem}

See section~\ref{sec:proof_thm_last_bias} for the proof. 

\subsection{Bias of the Tail-Averaged Iterate}
\label{subsec:tail_bias}
 We define the tail averaged bias as 
 \begin{\Ieee}{LLL}
 \label{eq:bias_tail}
 \Anabt=\frac{1}{N-a}\sum_{t=a+1}^N\Abt{t-1}{B} \Ieeen
 \end{\Ieee}

%
%
%
%
%

\begin{theorem}
\label{th:4}
Let $\gamma R B\leq \frac{c}{6}$ for some $0<c<1$ and $B$ such that $\gamma R\leq \frac{1}{2}$. There exist constants $c_1,c_2>0$ such that if $T=N(B+u)$ satisfies $\prbndsqinv>c_1\frac{\sqrt{M_4}}{\sigma_{\min}(G)}$ then for $a=\theta N$ with $0<\theta<1$ we have
\begin{\Ieee}{LLL}
\label{eq:th4}
\norm{\Ex{\gram{\Anoat{a}{N}}\indt{0}{N-1}}}\leq \\
 c_2\frac{1}{B(N-a)}\frac{e^{-c_3 B\gamma \sigma_{\min}(G)a}}{\gamma \sigma_{\min}(G)}\norm{A_0-\A}^2 \Ieeen
\end{\Ieee}
\end{theorem}

See section~\ref{sec:proof_prop_avg_bias} for the proof.

\section{Prediction Error}
\label{sec:pred_loss}
Recall the definition of the prediction error at stationarity. 
\begin{equation}
\label{eq:pred_loss_def1}
\losspred(\hat A;\A,\mu) := \mathbb{E}_{X_t \sim \pi}\|X_{t+1}-\hat A X_t\|^2
\end{equation}
where $\pi$ is the stationary distribution. 

Note that the prediction loss is a function of possibly random estimator $\hat A$. Hence the expectation in \eqref{eq:pred_loss_def1} is only with respect to the process $(X_t)$ (which is considered independent of $\hat A$).  Letting $G=\Ex{X_t X_t^{\top}} $ as the covariance matrix of the process at stationarity, we can write
\begin{\Ieee}{LLL}
\label{eq:pred_loss_def2}
\losspred(\hat A;\A,\mu)= \tr(G\gram{(\hat A-\A)})+\tr(\Sigma)\Ieeen
\end{\Ieee}

We are interested in bounding the expected prediction loss of the estimator which is the average iterate $\Ana$ of our algorithm $\sgdber$ (with $a=\theta N$). Note that $\Ana=\Anab+\Anav$ where the superscripts $b$ and $v$ correspond to bias and variance respectively (c.f.~\eqref{eq:tail_bias})

Hence
\begin{\Ieee}{LLL}
\label{eq:pred_bias_var_decomp}
\Ex{\losspred(\Ana;\A,\mu)}&=&\tr(\Sigma)+\tr\left(G^{1/2}\Ex{\gram{\left(\Ana-\A\right)}}G^{1/2}\right)\\
&\leq & \tr(\Sigma)+2\tr\left(G^{1/2}\Ex{\gram{\Anva{a}{N}}}G^{1/2}\right)\\
&&+2\tr\left(G^{1/2}\Ex{\gram{\Anoa{a}{N}}}G^{1/2}\right)\Ieeen
\end{\Ieee}

But we will only bound $\Ex{\losspred(\Ana;\A,\mu)\ind{0}{N-1}}$ so that we have a tight upper bound on the conditional expectation of $\losspred$ over a high probability event. 

As before we will just focus on the prediction error obtained using the algorithmic iterates from the coupled process, i.e., we will bound $\Ex{\losspred(\Anat;\A,\mu)\indt{0}{N-1}}$
\subsection{Variance of prediction error}
\label{sec:pred_var}

In this section we will focus on analyzing the variance part of the expected prediction loss under the coupled process
\begin{\Ieee}{LLL}
\label{eq:pred_var}
\clt^v=\tr\left(G^{1/2}\Ex{\gram{\Anvat{a}{N}}\indt{0}{N-1}}G^{1/2}\right)\Ieeen
\end{\Ieee}
where $T=N(B+u)$.

We begin with few lemmata which would be useful in bounding $\clt^v$. Recall  the definition of $\ch$
\begin{\Ieee}{LLL}
\label{eq:cal_H}
\ch=\Ex{\prod_{j=0}^{B-1}\left(I-2\gamma \Xtt{0}{-j} \Xtttr{0}{-j}\right)1[\cdt^{0}_{-0}]}\Ieeen
\end{\Ieee}
with $\tilde X_0$ sampled from the stationary distribution $\pi$.

\begin{lemma}
\label{lem:H_plus_HT_bound}
Let $\gamma\leq \frac{1}{8RB}$. Then 
\begin{\Ieee}{LLL}
\label{eq:H_plus_HT_bound}
\ch +\ch^{\top} \preceq 2\left(I-\frac{4}{3}\gamma B G\right)+\frac{8}{3}\gamma B\sqrt{M_4}\prbndsq I \Ieeen
\end{\Ieee} 
where $M_4=\Ex{\norms{\Xtt{0}{-0}}^4}$. 
For simplicity, we just say that for $\gamma RB<\frac{c}{4}$ with $0<c<1$ then
\begin{\Ieee}{LLL}
\label{eq:H_plus_HT_bound_0}
\ch +\ch^{\top} \preceq 2\left(I-c_1\gamma B G\right)+c_2\gamma B\sqrt{M_4}\prbndsq I\Ieeen
\end{\Ieee}
for some absolute constants $c_1,c_2>0$.
\end{lemma}
The proof is similar to the combined proofs of Lemmas~\ref{lem:contraction} and~\ref{lem:almost_sure_contraction}. We therefore skip it.

Next we will bound $\tr(G(I-\ch)^{-1})$. 

\begin{lemma}
\label{lem:predloss_iden_vari}
Let $\gamma RB<\frac{c_1}{4}$ with $0<c_1<1$. Then for $T$ such that $\prbndsqinv>c_2\frac{\sqrt{M_4}}{\sigma_{\min}(G)}$ we have
\begin{\Ieee}{LLL}
\label{eq:predloss_iden_vari}
\tr\left(G(I-\ch)^{-1}\right)\leq c\frac{d}{\gamma B}\Ieeen
\end{\Ieee}
for some absolute constant $c>0$.
\end{lemma}
\begin{proof}
First note that
\begin{\Ieee}{LLL}
\label{eq:predloss_iden_vari_1}
\tr\left(G(I-\ch)^{-1}\right)) &=& \tr\left(G^{1/2}(I-\ch)^{-1}G^{1/2}\right))\\
&=&\tr\left(\left(G^{-1}-G^{-1/2}\ch G^{-1/2}\right)^{-1}\right)\\
&\leq & d\norm{\left(G^{-1}-G^{-1/2}\ch G^{-1/2}\right)^{-1}}\\
&=&\frac{d}{\sigma_{\min}\left(G^{-1}-G^{-1/2}\ch G^{-1/2}\right)}\Ieeen
\end{\Ieee}

Let $Q=\left(G^{-1}-G^{-1/2}\ch G^{-1/2}\right)$. Let $\sym{Q}=Q+Q^{\top}$. We will relate $\sigma_{\min}(Q)$ with $\sigma_{\min}\left(\frac{\sym{Q}}{2}\right)$. From AM-GM inequality, for any $\theta>0$, we have
\begin{\Ieee}{LLL}
\label{eq:AM_GM_1}
\frac{Q^{\top} Q}{\theta}+\theta I\succeq \sym{Q}\Ieeen
\end{\Ieee}

Also
\begin{equation}
\label{eq:sigma_min}
\sigma_{\min}^2(Q)=\inf_{x:\norm{x}=1}x^{\top} Q^{\top} Qx
\end{equation}

Further, from lemma~\ref{lem:H_plus_HT_bound} we have
\begin{\Ieee}{LLL}
\label{eq:predloss_iden_vari_2}
\sym{Q}&= & G^{-1}-G^{-1/2}\frac{\ch+\ch^T}{2}G^{-1/2}\\
&\succeq & c_1\gamma B I -c_2\gamma B\sqrt{M_4}\prbndsq G^{-1}\\
&\succeq & c_1\gamma B I -c_2\gamma B\sqrt{M_4}\prbndsq \frac{1}{\sigma_{\min}(G)}I\Ieeen
\end{\Ieee}

Hence combining equations \eqref{eq:AM_GM_1}, \eqref{eq:sigma_min} and \eqref{eq:predloss_iden_vari_2} we have: 
\begin{\Ieee}{LLL}
\label{eq:predloss_iden_vari_3}
\frac{\sigma_{\min}^2(Q)}{\theta}+\theta \succeq c_1\gamma B  -c_2\gamma B\sqrt{M_4}\prbndsq \frac{1}{\sigma_{\min}(G)}.\Ieeen
\end{\Ieee}

Now choosing $\theta=\frac{1}{2}c_1\gamma B$ we get: 
\begin{\Ieee}{LLL}
\label{eq:predloss_iden_vari_4}
\sigma_{\min}^2(Q) \geq \frac{c_1^2}{4}\gamma^2 B^2  -\frac{c_2 c_1}{2}\gamma^2 B^2\sqrt{M_4}\prbndsq \frac{1}{\sigma_{\min}(G)}.\Ieeen
\end{\Ieee}
Now choose $T$ large enough such that $\frac{c_2 c_1}{2}\sqrt{M_4}\prbndsq \frac{1}{\sigma_{\min}(G)}\leq \frac{c_1^2}{8}$. 
Then, $\sigma_{\min}^2(Q) \geq c_3\gamma^2 B^2$, for some constant $c_3>0$. Hence from \eqref{eq:predloss_iden_vari_1},
\begin{\Ieee}{LLL}
\label{eq:predloss_iden_vari_6}
\tr\left(G(I-\ch)^{-1}\right) \leq c_4\frac{d}{\gamma B}.
\end{\Ieee}
\end{proof}

Next we bound $\tr(\Delta (I-\ch)^{-1} G)$ for any symmetric matrix $\Delta$. Let $\kappa(G)=\frac{\sigma_{\max(G)}}{\sigma_{\min}(G)}$ denote the condition number of $G$. 

\begin{lemma}
\label{lem:pred_loss_delta_vari}
Let $\gamma R B\leq \frac{c_1}{4}$ with $0<c_1<1$. Then for $T$ such that $\prbndsqinv>c_2\frac{\sqrt{M_4}}{\sigma_{\min}(G)}$ we have
\begin{\Ieee}{LLL}
\label{eq:predloss_delta_vari}
\left |\tr\left(\Delta(I-\ch)^{-1}G\right)\right|\leq c\frac{d}{\gamma B}\norm{\Delta}\sqrt{\kappa(G)}\Ieeen
\end{\Ieee}
for some absolute constant $c>0$.
\end{lemma}
\begin{proof}
We have 
\begin{\Ieee}{LLL}
\label{eq:predloss_delta_vari_1}
\left |\tr\left(\Delta(I-\ch)^{-1}G\right)\right| &=& \left |\tr\left(G^{1/2}\Delta G^{-1/2}G^{1/2}(I-\ch)^{-1}G^{1/2}\right)\right| \\
&\leq & d\norm{G^{1/2}\Delta G^{-1/2}}\norm{G^{1/2}(I-\ch)^{-1}G^{1/2}}\\
&\leq & d\sqrt{\kappa(G)}\norm{\Delta}\norm{G^{1/2}(I-\ch)^{-1}G^{1/2}}\Ieeen
\end{\Ieee}

From the proof of lemma~\ref{lem:predloss_iden_vari}, we know that 
\begin{\Ieee}{LLL}
\label{eq:predloss_delta_vari_2}
\norm{G^{1/2}(I-\ch)^{-1}G^{1/2}}\leq c\frac{1}{\gamma B}\Ieeen
\end{\Ieee}
for $T$ satisfying the condition the statement of the lemma. 

Hence: 
\begin{\Ieee}{LLL}
\label{eq:predloss_delta_vari_3}
\left |\tr\left(\Delta(I-\ch)^{-1}G\right)\right| \leq c\sqrt{\kappa(G)}\norm{\Delta}\frac{d}{\gamma B}\Ieeen
\end{\Ieee}

\end{proof}
Our goal is to bound $\tr(\Vt{t-1}(I-\ch)^{-1}G)$. From proposition~\ref{prop:1} we can decompose $\Vt{t-1}$ as: 
\begin{equation}
\label{eq:Vt1_decomp}
\Vt{t-1}=\gamma\tr(\Sigma)I  + (\Vt{t-1}-\gamma \tr(\Sigma)I),
\end{equation}
and hence,
\begin{equation}
\label{eq:predloss_tr_decomp}
\tr(\Vt{t-1}(I-\ch)^{-1}G)=\gamma\tr(\Sigma)\tr((I-\ch)^{-1}G)+\tr\left((\Vt{t-1}-\gamma\tr(\Sigma))(I-\ch)^{-1}G\right).
\end{equation}
To bound the second term in \eqref{eq:predloss_tr_decomp} we want to use lemma~\ref{lem:pred_loss_delta_vari}. Hence we need to bound the norm of $\Vt{t-1}-\gamma\tr(\Sigma)$.

\begin{lemma}
\label{lem:Vt1_minus_I}
Let $\gamma \leq \min\left\{\frac{c}{4RB},\frac{1}{2R}\right\}$ for $0<c<1$. Then there are constants $c_1,c_2,c_3>0$ such that for $\prbndsqinv>c_1\frac{\sqrt{M_4}}{\sigma_{\min}(G)}$ we have
\begin{\Ieee}{LLL}
\label{eq:Vt1_minus_I}
\norm{\Vt{t-1}-\gamma\tr(\Sigma)}\leq c_2\gamma d\sigma_{\max}\left[\frac{1}{B}+\left(1-c_3\gamma B\sigma_{\min}(G)\right)^{t}\right] \Ieeen
\end{\Ieee}
for some constant $c_1>0$.
\end{lemma}
\begin{proof}
From proposition~\ref{prop:1} we have
\begin{\Ieee}{LLL}
\label{eq:Vt1_minus_I_1}
\norm{\Vt{t-1}-\gamma\tr(\Sigma)I}&\leq &\gamma\tr(\Sigma) \frac{\gamma R}{1-\gamma R}+\\
&&c_1\gamma\tr(\Sigma)\norm{\Ex{\prodHtttr{t}{t}\prodHtt{t}{t}\indt{0}{t-1}}}\\
&&+c_2\gamma d\sigma_{\max}(\Sigma)T^2\prbndsq.\Ieeen
\end{\Ieee}
From lemma~\ref{lem:7} equation~\eqref{eq:lem7_3} we can show that
\begin{\Ieee}{LLL}
\label{eq:prod_HH^T_bound}
\norm{\Ex{\prodHtttr{t}{t}\prodHtt{t}{t}\indt{0}{t-1}}}\leq \left(1-c_3\gamma B\sigma_{\min}(G)\right)^{t}.\Ieeen
\end{\Ieee}

Hence
\begin{\Ieee}{LLL}
\label{eq:Vt1_minus_I_2}
&&\norm{\Vt{t-1}-\gamma\tr(\Sigma)I} \leq  c_4\gamma d\sigma_{\max}(\Sigma)\left[ \frac{\gamma R}{1-\gamma R}+\left(1-c_3\gamma B\sigma_{\min}(G)\right)^{t}\right]\\
&\leq & c_5\gamma d\sigma_{\max}\left[\gamma R+\left(1-c_3\gamma B\sigma_{\min}(G)\right)^{t}\right]\leq c_6\gamma d\sigma_{\max}\left[\frac{1}{B}+\left(1-c_3\gamma B\sigma_{\min}(G)\right)^{t}\right].\Ieeen
\end{\Ieee}

\end{proof}

Now we have all required ingredients for the main theorem of this section
\begin{theorem}
\label{thm:predloss_vari}
Let $\gamma \leq \min\left\{\frac{c}{4RB},\frac{1}{2R}\right\}$ for $0<c<1$. Then there are constants $c_1,c_2,c_3,c_4>0$ such that for $\prbndsqinv>c_1\frac{\sqrt{M_4}}{\sigma_{\min}(G)}$   the variance part of the expected prediction loss $\clt^v$ (defined in \eqref{eq:pred_var}) for $a=\theta N$ is bounded as
\begin{\Ieee}{LLL}
\label{eq:predloss_vari}
\clt^v &\leq & c_1 \frac{d\tr(\Sigma)}{NB(1-\theta)}+c_2 \frac{d^2\sigma_{\max}(\Sigma)}{NB(1-\theta)}\frac{\sqrt{\kappa(G)}}{B}+
c_3\frac{d^2\sigma_{\max}(\Sigma)}{(NB)^2(1-\theta)^2}\sqrt{\kappa(G)}\frac{1}{\gamma\sigma_{\min}(G)} \\
&& +c_4\gamma^2 R d\sigma_{\max}(\Sigma)T^2\prbndsq \tr(G)\Ieeen
\end{\Ieee}
 
\end{theorem}
\begin{proof}
From \eqref{eq:pred_var} and proposition~\ref{prop:avg_var1} equation \eqref{eq:avg_var_new1a} we have 
\begin{\Ieee}{LLL}
\clt^v &\leq &\frac{2}{(N-a)^2}\sum_{t=a+1}^N \tr\left(\Vt{t-1}(I-\ch)^{-1}G\right) \Ieeen\label{eq:predloss_vari_1a}\\
&&+ \frac{2}{(N-a)^2}\sum_{t=a+1}^N \tr\left(\Vt{t-1}(I-\ch)^{-1}\ch^{N-t+1}G\right) \Ieeen\label{eq:predloss_vari_1b}\\
&&+ c\delta \tr(G)\Ieeen\label{eq:predloss_vari_1c}
\end{\Ieee}
where $\delta=\gamma^2 T^2 R d\sigma_{\max}(\Sigma)\prbndsq$ as defined in \eqref{eq:avg_var_new2}

For the first term \eqref{eq:predloss_vari_1a} we have from \eqref{eq:predloss_tr_decomp}, lemma~\ref{lem:predloss_iden_vari}, lemma~\ref{lem:pred_loss_delta_vari} and lemma~\ref{lem:Vt1_minus_I}
\begin{\Ieee}{LLL}
\tr\left(\Vt{t-1}(I-\ch)^{-1}G\right)& \leq & c_1\gamma\tr(\Sigma)\frac{d}{\gamma B}+\\
&& c_2\frac{d}{\gamma B}\sqrt{\kappa(G)}\gamma d\sigma_{\max}(\Sigma)\left[\frac{1}{B}+\left(1-c_3\gamma B\sigma_{\min}(G)\right)^{t}\right]\\
&=& c_1 \frac{d\tr(\Sigma)}{B}+c_2 \frac{d^2\sigma_{\max}(\Sigma)}{B}\frac{\sqrt{\kappa(G)}}{B}+\\
&&c_4\frac{d^2\sigma_{\max}(\Sigma)}{B}\sqrt{\kappa(G)}\left(1-c_3\gamma B\sigma_{\min}(G)\right)^{t}\Ieeen\label{eq:predloss_vari_2}
\end{\Ieee}

Therefore
\begin{\Ieee}{LLL}
\label{eq:predloss_vari_3}
\frac{2}{(N-a)^2}\sum_{t=a+1}^N \tr\left(\Vt{t-1}(I-\ch)^{-1}G\right) &\leq & c_1 \frac{d\tr(\Sigma)}{NB(1-\theta)}+c_2 \frac{d^2\sigma_{\max}(\Sigma)}{NB(1-\theta)}\frac{\sqrt{\kappa(G)}}{B}+\\
&&c_5\frac{d^2\sigma_{\max}(\Sigma)}{N^2 B(1-\theta)^2}\sqrt{\kappa(G)}\frac{\left(1-c_3\gamma B\sigma_{\min}(G)\right)^{a+1}}{\gamma B\sigma_{\min}(G)}\Ieeen
\end{\Ieee}

Similarly, for the second term \eqref{eq:predloss_vari_1b}, from corollary~\ref{coro:var_last_iter}, lemma~\ref{lem:pred_loss_delta_vari}, lemma~\ref{lem:7} and the fact that $(I-\ch)^{-1}$ and $\ch^{N-t+1}$ commute, we get
\begin{\Ieee}{LLL}
\label{eq:predloss_vari_4}
&&\left|\tr\left(\Vt{t-1}(I-\ch)^{-1}\ch^{N-t+1}G\right)\right|\leq c_1\frac{d}{\gamma B}\sqrt{\kappa}\norms{\Vt{t-1}}\norms{\ch^{N-t+1}}\\
&\leq & c_2\frac{d}{\gamma B}\sqrt{\kappa(G)}\gamma d \sigma_{\max}(\Sigma)\left(1-c_3\gamma B\sigma_{\min}(G)\right)^{(N-t+1)}\\
&=& c_2\frac{d^2\sigma_{\max}(\Sigma)}{B}\sqrt{\kappa(G)}\left(1-c_3\gamma B\sigma_{\min}(G)\right)^{(N-t+1)}\Ieeen
\end{\Ieee}

Therefore
\begin{\Ieee}{LLL}
\label{eq:predloss_vari_5}
\left|\frac{2}{(N-a)^2}\sum_{t=a+1}^N \tr\left(\Vt{t-1}(I-\ch)^{-1}\ch^{N-t+1}G\right)\right|\leq c\frac{d^2\sigma_{\max}(\Sigma)}{N^2 B(1-\theta)^2}\sqrt{\kappa(G)}\frac{1}{\gamma B\sigma_{\min}(G)}\Ieeen
\end{\Ieee}

Hence we obtain, 
\begin{\Ieee}{LLL}
\label{eq:predloss_vari_6}
\clt^v &\leq & c_1 \frac{d\tr(\Sigma)}{NB(1-\theta)}+c_2 \frac{d^2\sigma_{\max}(\Sigma)}{NB(1-\theta)}\frac{\sqrt{\kappa(G)}}{B}+\\
&&c_3\frac{d^2\sigma_{\max}(\Sigma)}{N^2 B^2(1-\theta)^2}\sqrt{\kappa(G)}\frac{1}{\gamma\sigma_{\min}(G)} + c_4\gamma^2 R d\sigma_{\max}(\Sigma)T^2\prbndsq\tr(G).\Ieeen
\end{\Ieee}
\end{proof}

\subsection{Bias of prediction error}
\label{sec:pred_bias}

In this section we will focus on analyzing the (tail-averaged) bias part of the expected prediction loss from the coupled process
\begin{\Ieee}{LLL}
\label{eq:pred_bias}
\clt^b=\tr\left(G^{1/2}\Ex{\gram{\left(\Anoat{a}{N}\right)}\indt{0}{N-1}}G^{1/2}\right)\Ieeen
\end{\Ieee}
where $T=N(B+u)$ and $a=\theta N$ for $0<\theta<1$.

\begin{theorem}
\label{thm:pred_bias}
Let $\gamma R B\leq \frac{c}{6}$ for some $0<c<1$ and $B$ such that $\gamma R\leq \frac{1}{2}$. There exist constants $c_1,c_2, c_3,c_4>0$ such that if $T$ satisfies $\prbndsqinv>c_1\frac{\sqrt{M_4}}{\sigma_{\min}(G)}$ then for $a=\theta N$ with $0<\theta<1$ we have
\begin{\Ieee}{LLL}
\label{eq:pred_bias_1}
\clt^b \leq &c_2\frac{1}{NB(1-\theta)}\frac{ \tr(G)}{\gamma\sigma_{\min}(G) }e^{-c_3 NB \gamma \sigma_{\min}(G)\theta}\norm{A_0-\A}^2 \Ieeen
\end{\Ieee}

\end{theorem}
\begin{proof}
Proof follows directly from \eqref{eq:pred_bias} and theorem~\ref{th:4}.
\end{proof}
\subsection{Overall Prediction Error}
Combining theorem~\ref{thm:predloss_vari} and theorem~\ref{thm:pred_bias} along with lemma~\ref{lem:coupling_AA^t} we obtain the main theorem on prediction error of $\sgdber$

\begin{theorem}
\label{thm:predloss_main}
Let $R,B,u,\alpha$ be chosen as in section~\ref{sec:main_results}. Let $\gamma = \frac{c}{4RB}\leq \frac{1}{2R}$ for $0<c<1$. Then there are constants $c_1,c_2,c_3,c_4>0$ such that for $\prbndsqinv>c_1\frac{\sqrt{M_4}}{\sigma_{\min}(G)}$   the expected prediction loss $\cl$ (defined in \eqref{eq:pred_loss_def2}) is bounded as
\begin{\Ieee}{LLL}
\label{eq:predloss_main}
\Ex{\losspred(\Ana;\A,\mu) \ind{0}{N-1}}&\leq & c_2\left[ \frac{d\tr(\Sigma)}{B(N-a)}+ \frac{d^2\sigma_{\max}(\Sigma)}{B(N-a)}\frac{\sqrt{\kappa(G)}}{B}\right]+\\
&& c_3\left[\frac{d^2\sigma_{\max}(\Sigma)}{B^2(N-a)^2}\sqrt{\kappa(G)}\frac{1}{\gamma\sigma_{\min}(G)}+\right.\\
&&\left. \frac{1}{B(N-a)} d\kappa(G)RB e^{-c_4  \frac{\sigma_{\min}(G)}{R}a}\norm{A_0-\A}^2+\right.\\
&& \left.\left(  \frac{T^3}{B^3}\norm{\A^u}+\frac{ d\sigma_{\max}(\Sigma)}{R} \frac{T^2}{B^2}\prbndsq\right)\tr(G)\right]\\
\Ieeen
\end{\Ieee}

Hence, if $\norm{\A}<c_0<1$ then choosing $a\geq C\frac{R\log T}{\sigma_{\min}(G)}$ such that $B(N-a)=\Theta(T)$ and $B,u$ as in section~\ref{sec:main_results} we get
\begin{\Ieee}{LLL}
\label{eq:predloss_main_1}
\Ex{\losspred(\Ana;\A,\mu) \ind{0}{N-1}}&\leq & c_2 \frac{d\tr(\Sigma)}{T}+o\left(\frac{1}{T}\right)
\Ieeen
\end{\Ieee}

\end{theorem}

\section{Proof of Proposition~\ref{prop:1}}
\label{sec:proof_prop_1}
\begin{proof}[Proof of Proposition~\ref{prop:1}]

First note that
\begin{\Ieee}{LLL}
\gram{\Attdiff{t-1}{b}}&=&\sum_{r=1}^t \sum_{j=0}^{B-1}\dgt(t,r,j)+\sum_{r_1,r_2=1}^t\sum_{j_1,j_2=0}^{B-1}\crot(t,r_1,j_1,r_2,j_2)\Ieeen\label{eq:prop1_0}
\end{\Ieee}

where
\begin{\Ieee}{RLL}
\dgt(t,r,j)&=& 4\gamma^2\norm{\Nt{t-r}{-j}}^2\cdot \\
&& \prodHtttr{t}{r-1}\Htttr{t-r}{j+1}{B-1} \Xtt{t-r}{-j}\Xtttr{t-r}{-j} \Htt{t-r}{j+1}{B-1} \prodHtt{t}{r-1}\Ieeen\label{eq:diag_term}\\
\crot(t,r_1,j_1,r_2,j_2)&=& 4\gamma^2\left(\Nt{t-r_1}{-j_1}\Xtttr{t-r_1}{-j_1} \Htt{t-r_1}{j_1+1}{B-1}\prod_{s=r_1-1}^{1}\Htt{t-s}{0}{B-1}\right)^{\top}\cdot \\
&& \left(\Nt{t-r_2}{-j_2}\Xtttr{t-r_2}{-j_2} \Htt{t-r_2}{j_2+1}{B-1}\prod_{s=r_2-1}^{1}\Htt{t-s}{0}{B-1}\right)\Ieeen\label{eq:cross_term}\\
\end{\Ieee}
denote the diagonal and cross terms respectively. 

We begin by noting the following two facts about $\Attdiff{t-1}{b}$:
\begin{itemize}

\item It has zero mean
\begin{\Ieee}{LLL}
\Ex{\Attdiff{t-1}{B}}=0\Ieeen \label{eq:prop1_1}
\end{\Ieee}

\item Let $(r_1,j_1)\neq (r_2,j_2)$. Then
\begin{\Ieee}{LLL}
\Ex{\crot(t,r_1,j_1,r_2,j_2)}=0 \Ieeen\label{eq:prop1_2}
\end{\Ieee}
\end{itemize}

This follows because, assuming $r_1>r_2$, the term $\Nt{t-r_1}{-j_1}\Xtttr{t-r_1}{-j_1} \Htt{t-r_1}{j_1+1}{B-1}$ is independent of everything else in that expression, and that $\Nt{t-r_1}{-j_1}$ is independent of $\Xtttr{t-r_1}{-j_1} \Htt{t-r_1}{j_1+1}{B-1}$. A similar argument can be made for the case when $r_1=r_2$ but $j_1\neq j_2$.\\

But we are interested in expectation on the event $\cdt^{0, t-1}$. 

We will bound the expectation of cross terms in the following lemma.

\begin{lemma}
\label{lem:cross_term}
We have
\begin{\Ieee}{LLL}
\norm{\Ex{\sum_{r_1,r_2}\sum_{j_1,j_2}\crot(t,r_1,j_1,r_2,j_2)}\indt{0}{t-1}}\leq 8(Bt)^2 \gamma^2 R\tr(\Sigma)\prbndsq \Ieeen
\end{\Ieee}

\end{lemma}

\begin{proof}
Let 

Consider a single cross term: $\crot(t,r_1,j_1,r_2,j_2)$ and without loss of generality, assume that either $r_1> r_2$ or $r_1=r_2$ but $j_1 < j_2$. In either case, we note that $\Nt{t-r_1}{-j_1}$ is unconditionally independent of all other terms present in  $\crot(t,r_1,j_1,r_2,j_2)$. The main problem here is to bound the expectation over the event $\cdt^{0,t-1}$. For the sake of convenience, only in this proof, we will define the following notation:

$$\crot(t,r_1,j_1,r_2,j_2) = E_1\Nt{t-r_1,\top}{-j_1}\Nt{t-r_2}{-j_2}E_2 $$
Where $E_1$ and $E_2$ are random matrices defined according to the definition of $\crot(t,r_1,j_1,r_2,j_2)$ and are unconditionally independent of $\Nt{t-r_1,\top}{-j_1}$. Let $\mathcal{F}_{E} = \sigma(E_1,E_2,\Nt{t-r_2}{-j_2})$. Note that when conditioned on the event $\cdt^{0,t-1}$, we must have the event $ \mathcal{M} := \{\|E_1\| \leq 4\gamma^2 \sqrt{R}\}\cap\{\|E_2\| \leq \sqrt{R}\}$ almost surely. Therefore, we conclude:
\begin{align}
\Ex{\crot(t,r_1,j_1,r_2,j_2)\indt{0}{t-1}}&= \Ex{\crot(t,r_1,j_1,r_2,j_2)\indt{0}{t-1}1\left[\mathcal{M}\right] } \nonumber \\
&= \mathbb{E}\left[1\left[\mathcal{M}\right]E_1\mathbb{E}\left[\Nt{t-r_1,\top}{-j_1}\indt{0}{t-1}\biggr| \mathcal{F}_{E}\right]\Nt{t-r_2}{-j_2}E_2\right] \nonumber \\
&\leq \mathbb{E}\left[1\left[\mathcal{M}\right]\norm{E_1}\norm{\mathbb{E}\left[\Nt{t-r_1,\top}{-j_1}\indt{0}{t-1}\biggr| \mathcal{F}_{E}\right]}\norm{\Nt{t-r_2}{-j_2}}\norm{E_2}\right] \nonumber \\
&\leq 4\gamma^2R \mathbb{E}\left[\norm{\mathbb{E}\left[\Nt{t-r_1,\top}{-j_1}\indt{0}{t-1}\biggr| \mathcal{F}_{E}\right]}\norm{\Nt{t-r_2}{-j_2}}\right] \label{eq:first_cross_bound} 
\end{align}

In the third step, we have used the fact that under the event $\mathcal{M}$, the norms $\|E_1\|, \|E_2\|$ are bounded. We will now bound $\mathbb{E}\left[\Nt{t-r_1,\top}{-j_1}\indt{0}{t-1}\biggr| \mathcal{F}_{E}\right]$. Clearly, due to the unconditional independence, we must have:

\begin{align}
&\mathbb{E}\left[\Nt{t-r_1,\top}{-j_1}\biggr| \mathcal{F}_{E}\right] = 0 \nonumber\\
&\implies \mathbb{E}\left[\Nt{t-r_1,\top}{-j_1}\indt{0}{t-1}\biggr| \mathcal{F}_{E}\right] = -\mathbb{E}\left[\Nt{t-r_1,\top}{-j_1}\indtc{0}{t-1}\biggr| \mathcal{F}_{E}\right] \nonumber \\
&\implies \norm{\mathbb{E}\left[\Nt{t-r_1,\top}{-j_1}\indt{0}{t-1}\biggr| \mathcal{F}_{E}\right]} \leq \sqrt{\tr{\Sigma}} \sqrt{\mathbb{P}\left(\cdt^{0,t-1,C}\biggr| \mathcal{F}_{E}\right)}
\end{align}
In the last step, we have used Cauchy Schwarz inequality and the fact that $\Nt{t-r_1,\top}{-j_1}$ is independent of $\mathcal{F}_E$. We combine the Equation above with Equation~\eqref{eq:first_cross_bound} and apply Jensen's inequality once again to conclude:
\begin{equation}
\norm{\Ex{\crot(t,r_1,j_1,r_2,j_2)\indt{0}{t-1}}} \leq 4\gamma^2R \tr(\Sigma)\sqrt{\Pb{\cdt^{0,t-1,C}}} \leq 4\gamma^2 R \frac{\tr(\Sigma)}{T^{\alpha/2}}
\end{equation}
 In the last step, we have used Lemma~\ref{lem:2} to bound $\mathbb{P}\left(\cdt^{0,t-1,C}\right)$. Summing over all the indices $(r_1,j_1,r_2,j_2)$, we conclude the statement of the lemma.

\end{proof}

%
%

\begin{lemma}
\label{lem:diag_term}

We have:
\begin{\Ieee}{LLL}
\label{eq:diag_1a}
\Ex{\sum_{r=1}^t\sum_{j=0}^{B-1}\dgt(t,r,j)\indt{0}{t-1}}\preceq 4\gamma^2\tr(\Sigma)\Ex{\sum_{r=1}^{t}\sum_{j=0}^{B-1}\prodHtttr{t}{r-1}\Htttr{t-r}{j+1}{B-1} \Xtt{t-r}{-j}\cdot \right. \\
\left.\Xtttr{t-r}{-j} \Htt{t-r}{j+1}{B-1} \prodHtt{t}{r-1}\indt{0}{t-1}}+\delta_{\dg} I \Ieeen
\end{\Ieee}
and 

\begin{\Ieee}{LLL}
\label{eq:diag_1b}
\Ex{\sum_{r=1}^t\sum_{j=0}^{B-1}\dgt(t,r,j)\indt{0}{ t-1}}\succeq 4\gamma^2\tr(\Sigma)\Ex{\sum_{r=1}^{t}\sum_{j=0}^{B-1}\prodHtttr{t}{r-1}\Htttr{t-r}{j+1}{B-1} \Xtt{t-r}{-j}\cdot \right. \\
\left.\Xtttr{t-r}{-j} \Htt{t-r}{j+1}{B-1} \prodHtt{t}{r-1}\indt{0}{t-1}}-\delta_{\dg} I \Ieeen
\end{\Ieee}

where 
\begin{\Ieee}{LLL}
\label{eq:diag_2}
\delta_{\dg}\equiv \delta_{\dg}(T,\Sigma,R,\mu_4) =4\gamma^2 (Bt) R \sqrt{\mu_4} \prbndsq\Ieeen
\end{\Ieee}

\end{lemma}
\begin{proof}
The evaluation of expectations is clear when there is no indicator $\indt{0}{ t-1}$ within the expectation. We will now deal with it just like in the proof of Lemma~\ref{lem:cross_term}. Consider $\dgt(t,r,j)$. For the sake of convenience, only in this proof, we will use the following notation:

$$\dgt(t,r,j) = 4\gamma^2\norm{\Nt{t-r}{-j}}^2 E \,.$$
Where the random PSD matrix $E$ is unconditionally independent of $\Nt{t-r}{-j}$. Let $\mathcal{M} = \{\|E\|\leq R\}$. Conditioned on the event $\cdt^{0,t-1}$, the event $\mathcal{M}$ holds almost surely. Let $\mathcal{F}_{E} = \sigma(E)$. 

Now consider:
\begin{align}
\Ex{\dgt(t,r,j)\indt{0}{t-1}} &= \Ex{\dgt(t,r,j)\indt{0}{t-1}1\left[\mathcal{M}\right]}  \nonumber  \\
 &= 4\gamma^2\Ex{\norm{\Nt{t-r}{-j}}^2 E \indt{0}{t-1}1\left[\mathcal{M}\right]}  \nonumber \\
&= 4\gamma^2\Ex{\mathbb{E}\left[\norm{\Nt{t-r}{-j}}^2\indt{0}{t-1}|\mathcal{F}_{E}\right] E 1\left[\mathcal{M}\right]} \label{eq:diag_bound_1}
\end{align} 

It can be easily shown via similar techniques used in Lemma~\ref{lem:cross_term} that:
$$\tr(\Sigma) - \sqrt{\mu_4} \sqrt{\mathbb{P}\left(\cdt^{0,t-1,C}\bigr|\mathcal{F}_E\right)} \leq \mathbb{E}\left[\norm{\Nt{t-r}{-j}}^2\indt{0}{t-1}|\mathcal{F}_{E}\right] \leq \tr(\Sigma)$$
Using this in Equation~\eqref{eq:diag_bound_1}, we conclude:

\begin{align}
\Ex{\dgt(t,r,j)\indt{0}{t-1}} &\preceq 4\gamma^2 \tr(\Sigma) \mathbb{E}\left[E 1\left[\mathcal{M}\right]\right] \nonumber \\
&=  4\gamma^2 \tr(\Sigma) \mathbb{E}\left[E 1\left[\mathcal{M}\right]\indt{0}{ t-1} + E 1\left[\mathcal{M}\right]\indtc{0}{ t-1}\right] \nonumber \\
&=  4\gamma^2 \tr(\Sigma) \mathbb{E}\left[E \indt{0}{ t-1} + E 1\left[\mathcal{M}\right]\indtc{0}{ t-1}\right] \nonumber \\
&\preceq 4\gamma^2 \tr{\Sigma}\mathbb{E}\left[E \indh{0}{ t-1}\right] + 4\gamma^2 \tr(\Sigma)R\frac{I}{T^{\alpha}} \label{eq:diag_upper}
\end{align}
In the third step, we have used the fact that $\cdt^{0,t-1} \subseteq \mathcal{M}$. In the last step we have used the fact that $E$ is PSD and over the event $\mathcal{M}$, $E \preceq R I$. We have used Lemma~\ref{lem:2} to bound $\mathbb{P}(\cdt^{0,t-1,C})$. Using a similar technique as above, we can show that:

\begin{align}\label{eq:diag_lower}
\Ex{\dgt(t,r,j)\indt{0}{t-1}} &\succeq 4\gamma^2 \tr{\Sigma}\mathbb{E}\left[E \indt{0}{ t-1}\right] - 4\gamma^2 \frac{\sqrt{\mu_4}R}{T^{\alpha/2}}I
\end{align}
Note that $\frac{\sqrt{\mu_4}R}{T^{\alpha/2}} \geq \frac{\tr(\Sigma)R}{T^{\alpha}}$. Summing over $r,j$ and combining Equations~\eqref{eq:diag_lower} and~\eqref{eq:diag_upper}, we conclude the result.

\end{proof}

For convenience, define  $K^{s} := \sum_{j=0}^{B-1} \Htttr{s}{j+1}{B-1} \Xtt{s}{-j}\Xtttr{s}{-j} \Htt{s}{j+1}{B-1}$

\begin{claim}\label{claim:single_buffer_isometry}
Suppose $\gamma < \frac{1}{R}$. Under the event $\cdt^{0,t-1}$, for every $s \leq t-1$ we must have:
 $$ \frac{I-\Htttr{s}{0}{B-1}\Htt{s}{0}{B-1}}{4\gamma} \preceq K^s \preceq \frac{I-\Htttr{s}{0}{B-1}\Htt{s}{0}{B-1}}{\gammah}$$
Where $\gammah = 4\gamma(1-\gamma R)$
\end{claim}
\begin{proof}In the entire proof, we suppose that the event $\cdt^{0,t-1}$ holds. Consider:
\begin{align}
&\Htttr{s}{j}{B-1}\Htt{s}{j}{B-1}+4\gamma\Htttr{s}{j+1}{B-1} \Xtt{s}{-j}\Xtttr{s}{-j} \Htt{s}{j+1}{B-1} \nonumber \\
&= \Htttr{s}{j+1}{B-1}\left(I-\left(4\gamma - 4\gamma^2\norms{\Xtt{s}{-j}}^2\right)\Xtt{s}{-j}\Xtttr{s}{-j} \right)\Htt{s}{j+1}{B-1}+4\gamma\Htttr{s}{j+1}{B-1} \Xtt{s}{-j}\Xtttr{s}{-j} \Htt{s}{j+1}{B-1} \nonumber \\
&= \Htttr{s}{j+1}{B-1}\left(I + 4\gamma^2\norms{\Xtt{s}{-j}}^2\Xtt{s}{-j}\Xtttr{s}{-j}\right)\Htt{s}{j+1}{B-1} \nonumber \\
&\succeq \Htttr{s}{j+1}{B-1}\Htt{s}{j+1}{B-1} \label{eq:main_diag_recursion}
\end{align}
Using the recursion in Equation~\eqref{eq:main_diag_recursion}, we show that:

$$\Htttr{s}{0}{B-1}\Htt{s}{0}{B-1} + 4\gamma K^s \succeq I \,.$$
This establishes the lower bound. To establish the upper bound, we consider
$$ \Htttr{s}{j}{B-1}\Htt{s}{j}{B-1}+\gammah\Htttr{s}{j+1}{B-1} \Xtt{s}{-j}\Xtttr{s}{-j} \Htt{s}{j+1}{B-1}\,.$$

Following similar technique used to establish Equation~\eqref{eq:main_diag_recursion}, using the fact that under the event $\cdt^{0,t-1}$ we have $\norms{\Xtt{s}{-j}}^2 \leq R$ we show that:
$$ \Htttr{s}{j}{B-1}\Htt{s}{j}{B-1}+\gammah\Htttr{s}{j+1}{B-1} \Xtt{s}{-j}\Xtttr{s}{-j} \Htt{s}{j+1}{B-1} \preceq \Htttr{s}{j+1}{B-1}\Htt{s}{j+1}{B-1}\,.$$

Using a similar recursion as before, we establish that:
$$\Htttr{s}{0}{B-1}\Htt{s}{0}{B-1} + \gammah K^s \preceq I \,.$$
\end{proof}

We are now ready to bound the first term in \eqref{eq:diag_1a}:
\begin{\Ieee}{LLL}
 \Ex{\sum_{r=1}^{t}\prodHtttr{t}{r-1}K^{t-r} \prodHtt{t}{r-1}\indt{0}{t-1}}\Ieeen\label{eq:prop1_3}
\end{\Ieee}

It is easy to show via. telescoping sum argument that:

\begin{equation}
 \sum_{r=1}^{t}\prodHtttr{t}{r-1}\left(I - \Htttr{t-r}{0}{B-1}\Htt{t-r}{0}{B-1} \right) \prodHtt{t}{r-1} = I - \prodHtttr{t}{t}\prodHtt{t}{t}
\end{equation}
 We then use Claim~\ref{claim:single_buffer_isometry} to show that under the event $\cdt^{0,t-1}$, we must have:
 
 \begin{equation}\label{eq:diag_lower_isometry}
\frac{I - \prodHtttr{t}{t}\prodHtt{t}{t}}{4\gamma}  \preceq \sum_{r=1}^{t}\prodHtttr{t}{r-1}K^{t-r} \prodHtt{t}{r-1} 
\end{equation}
And:
\begin{equation}\label{eq:diag_upper_isometry}
\sum_{r=1}^{t}\prodHtttr{t}{r-1}K^{t-r} \prodHtt{t}{r-1} \preceq \frac{I - \prodHtttr{t}{t}\prodHtt{t}{t}}{\gammah}
\end{equation}


Finally, combining Lemma~\ref{lem:cross_term}, Lemma~\ref{lem:diag_term}, claim~\ref{claim:single_buffer_isometry}, Equations~\eqref{eq:diag_lower_isometry},~\eqref{eq:diag_upper_isometry} and the bound on $\mu_4$ (stated after assumption~\ref{as:stationarity} in section~\ref{sec:prob}) along with $\gammah=4\gamma(1-\gamma R)$ we get the statement of the proposition.

\end{proof}

\section{Proof of Proposition~\ref{prop:avg_var1}}
\label{sec:prop_avg_var1}
Before delving into the proof, we note some useful results below.
\begin{lemma}
\label{lem:5}
For any random matrix $B\in \mathbb{R}^{d\times d}$ we have that
\begin{equation}
\label{eq:lem5}
\Ex{B^\top}\Ex{B}\preceq \Ex{B^{\top} B}
\end{equation}
Hence 
\begin{equation}
\label{eq:lem5_1}
\norm{\Ex{B}}\leq \sqrt{\norm{\Ex{B^\top B}}}
\end{equation}
\end{lemma}
\begin{proof}
Note that for any vector $x\in\mathbb{R}^d$ we have 
\begin{\Ieee}{LLL}
\label{eq:lem5_2}
x^\top \Ex{B^\top}\Ex{B}x=\norm{\Ex{Bx}}^2\leq \Ex{\norm{Bx}^2}=x^\top\Ex{B^{\top} B}x \Ieeen
\end{\Ieee}
\end{proof}

\begin{lemma}
\label{lem:7}
Let $\gamma R B\leq \frac{c}{6}$ for $0<c<1$. The there are constants $c_1,c_2>0$ such that for $\prbndsqinv>c_1\frac{\sqrt{M_4}}{\sigma_{\min}(G)}$ we have
\begin{\Ieee}{LLL}
\label{eq:lem7}
\norm{\ch}\leq \sqrt{1-c_2\gamma B\sigma_{\min}(G)} \leq 1-\frac{c_2}{2}\gamma B\sigma_{\min}(G)\Ieeen
\end{\Ieee}
with $1-c_2\gamma B\sigma_{\min}(G)>0$. 

\end{lemma}
\begin{proof}
Note that $\ch$ can be written as $\ch=\Ex{\Htt{0}{0}{B-1}1[\cdt^0_{-0}]}$. 
First we use Lemma \ref{lem:5} to get 
\begin{equation}
\label{eq:lem7_1}
\norm{\ch}\leq \sqrt{\norm{\Ex{\Htttr{0}{0}{B-1}\Htt{0}{0}{B-1}1[\cdt^0_{-0}]}}}
\end{equation}

Then, from Lemma~\ref{lem:almost_sure_contraction} we can show that there are constants $c_1, c_2>0$ such that 
\begin{\Ieee}{LLL}
\label{eq:lem7_2}
\norm{\Ex{\Htttr{0}{0}{B-1}\Htt{0}{0}{B-1}1[\cdt^0_{-0}]}}\leq \left(1-c_1\gamma B\sigma_{\min}(G)+c_2\gamma B\sqrt{M_4}\prbndsq\right)\Ieeen
\end{\Ieee}

Now choosing $T$ such that $\prbndsqinv> \frac{c_2\sqrt{M_4}}{2c_1 \sigma_{\min}(G)}$ we get
\begin{\Ieee}{LLL}
\label{eq:lem7_3}
\norm{\Ex{\Htttr{0}{0}{B-1}\Htt{0}{0}{B-1}1[\cdt^0_{-0}]}}\leq \left(1-c_3\gamma B\sigma_{\min}(G)\right)\Ieeen
\end{\Ieee}
where $c_3$ is such that the RHS in \eqref{eq:lem7_3} is positive. 
Hence the claim follows.

\end{proof}

\begin{proof}[Proof of Proposition~\ref{prop:avg_var1}]
We will prove the proposition only for $a=0$. The arguments for general $a$ are exactly the same. 

For simplicity, we denote 
\begin{equation}
\label{eq:An_notation}
\Anvt\equiv \Anvat{0}{N}
\end{equation}

From recursion \eqref{eq:sgd_expreplay2} we have the following relation between $\Attdiff{t_2-1}{B}$ and $\Attdiff{t_1-1}{B}$ for $t_2>t_1$
\begin{\Ieee}{LLL}
\label{eq:avg_var_new4}
\Attdiff{t_2-1}{B} &=& \Attdiff{t_1-1}{B}\prodHtt{t_2}{t_2-t_1}+\\
&&2\gamma\sum_{r=1}^{t_2-t_1}\sum_{j=0}^{B-1}\Nt{t_2-r}
{-j}\Xtttr{t_2-r}{-j}\Htt{t_2-r}{j+1}{B-1}\prodHtt{t_2}{r-1}.\Ieeen
\end{\Ieee}

Hence we have
\begin{\Ieee}{LLL}
\label{eq:avg_var_new5}
\Attdiff{t_1-1}{B}^{\top}\Attdiff{t_2-1}{B} =\Attdiff{t_1-1}{B}^{\top} \Attdiff{t_1-1}{B}\prodHtt{t_2}{t_2-t_1}+\\
2\gamma \Attdiff{t_1-1}{B}^{\top} \sum_{r=1}^{t_2-t_1}\sum_{j=0}^{B-1}\Nt{t_2-r}
{-j}\Xtttr{t_2-r}{-j}\Htt{t_2-r}{j+1}{B-1}\prodHtt{t_2}{r-1}.\Ieeen
\end{\Ieee}
The second term in \eqref{eq:avg_var_new5} is bounded in claim~\ref{claim:5}

%
%

The first term in \eqref{eq:avg_var_new5} can be analyzed using independence as follows.
\begin{\Ieee}{LLL}
\label{eq:avg_var_new6e}
\Ex{\Attdiff{t_1-1}{B}^{\top} \Attdiff{t_1-1}{B}\indt{0}{t_1-1}\left(\prod_{s=t_2-t_1}^1\Htt{t_2-s}{0}{B-1}\right)\indt{t_1}{N-1}} \\
=\Vt{t_1-1} \Ex{\left(\prod_{s=t_2-t_1}^1\Htt{t_2-s}{0}{B-1}\right)\indt{t_1}{N-1}}\\
=\Vt{t_1-1}\Ex{\left(\prod_{s=t_2-t_1}^1\Htt{t_2-s}{0}{B-1}\right)\indt{t_1}{t_2-1}}\Ex{\indt{t_2}{N-1}}\\
=\Vt{t_1-1}\left(\prod_{s=t_2-t_1}^1 \Ex{\Htt{t_2-s}{0}{B-1}\indt{t_1}{t_2-1}}\right)\Ex{\indt{t_2}{N-1}}=\Vt{t_1-1}\ch^{t_2-t_1}\Ex{\indt{t_2}{N-1}}\\
=\Vt{t_1-1}\ch^{t_2-t_1}-\Vt{t_1-1}\ch^{t_2-t_1}\Ex{\indtc{t_2}{N-1}}.\Ieeen
\end{\Ieee}

Note that, {\small
\begin{\Ieee}{LLL}
\label{eq:avg_var_new6b}
\Attdiff{t_1-1}{B}^{\top} \Attdiff{t_1-1}{B}\preceq 4\gamma^2 (Bt_1)\sum_{r=1}^{t_1}\sum_{j=0}^{B-1}\norm{\Nt{t_1-r}{-j}}^2 \cdot\\
 \left(\prod_{s=1}^{r-1}\Htttr{t_1-s}{0}{B-1}\right)\Htttr{t_1-r}{j+1}{B-1}\Xtt{t_1-r}{-j}\Xtttr{t_1-r}{-j} \Htt{t_1-r}{j+1}{B-1}\left(\prod_{s=r-1}^{1}\Htt{t_1-s}{0}{B-1}\right).\Ieeen 
\end{\Ieee}}

From equation~\eqref{eq:avg_var_new6b}, we have: 
\begin{equation}
\label{eq:avg_var_6f}
\norm{\Vt{t_1-1}}\leq c\gamma^2 (Bt_1)^2 Rd\sigma_{\max}, 
\end{equation}
and further, $\norm{\ch}<1$ from Lemma~\ref{lem:7}. Hence, 
\begin{\Ieee}{LLL}
\label{eq:avg_var_new6g}
\norm{\Vt{t_1-1}\ch^{t_2-t_1}\Ex{\indtc{t_2}{N-1}}} &\leq & \norm{\Vt{t_1-1}\ch^{t_2-t_1}} \prbnd\leq c\gamma^2 (Bt_1)^2 Rd\sigma_{\max}\prbnd. 
\end{\Ieee}
For brevity, given a matrix $Q\in \mathbb{R}^{d\times d}$, let, 
\begin{equation}
\label{eq:gram_matrix}
\sym{Q}=Q+Q^{\top}. 
\end{equation}
Combining everything so far, we have, for $t_2>t_1$: {\small 
\begin{\Ieee}{LLL}
\label{eq:avg_var_new7a}
&&\sym{\Ex{\Attdiff{t_1-1}{B}^{\top}\Attdiff{t_2-1}{B}\indt{0}{N-1}}}\\ 
&\preceq & \sym{\Vt{t_1-1}\ch^{t_2-t_1}} + c_1\gamma^2 (Bt_1)^2 Rd\sigma_{\max}\prbnd I +\\
&&\left(c_3\gamma^2 B^2 t_1 t_2 R d\sigma_{\max} \prbndsq\right)I\Ieeen
\end{\Ieee}}
Since $Bt_2\leq T$  we get: 
\begin{multline}
\label{eq:avg_var_new7b}
\sym{\Ex{\Attdiff{t_1-1}{B}^{\top}\Attdiff{t_2-1}{B}\indt{0}{N-1}}} \preceq \sym{\Vt{t_1-1}\ch^{t_2-t_1}}+ \\
 c_3\gamma^2 T^2 Rd\sigma_{\max}\prbndsq I.
\end{multline}

Therefore we have, 
\begin{align*}
&&\frac{1}{N^2}\sum_{t_1\neq t_2}\Ex{\Attdiff{t_1-1}{B}^{\top}\Attdiff{t_2-1}{B}} \preceq  \frac{1}{N^2}\sum_{t_1=1}^{N-1}\sym{\Vt{t_1-1}\left(\sum_{t_2>t_1}\ch^{t_2-t_1}\right)} \\
&& + c_3\gamma^2 T^2 Rd\sigma_{\max}\prbndsq I.
\end{align*}
Next observe that,
\begin{\Ieee}{LLL}\
\label{eq:avg_var_new9}
&& \frac{1}{N^2}\sum_{t=1}^N \Vt{t-1} +\frac{1}{N^2}\sum_{t_1=1}^{N-1}\sym{\Vt{t_1-1}\left(\sum_{t_2>t_1}\ch^{t_2-t_1}\right)}\\
 & =& \frac{1}{N^2}\sum_{t=1}^N \Vt{t-1} +\frac{1}{N^2}\sum_{t_1=1}^{N-1}\sym{\Vt{t_1-1}\left(\sum_{s=1}^{N-t_1}\ch^{s}\right)} \\
&\preceq &\frac{1}{N^2}\sum_{t=1}^{N}\sym{\Vt{t-1}\left(\sum_{s=0}^{N-t}\ch^{s}\right)}.
\end{\Ieee}
Hence, substituting in \eqref{eq:An1}, we obtain: 
\begin{\Ieee}{LLL}
\Ex{\gram{\left(\hat{A}^v_N\right)}\indt{0}{N-1}}&&\preceq  \frac{1}{N^2}\sum_{t=1}^{N}\sym{\Vt{t-1}\left(\sum_{s=0}^{N-t}\ch^{s}\right)} + \Ieeen\label{eq:avg_var_new10a}\\
&& c_3\gamma^2 T^2 Rd\sigma_{\max}\prbndsq I.\Ieeen\label{eq:avg_var_new10e}
\end{\Ieee}

From Equations~\eqref{eq:avg_var_new10a}-\eqref{eq:avg_var_new10e} we obtain \eqref{eq:avg_var_new1}. 

Now $\sum_{s=0}^{N-t}\ch^s = (I-\ch)^{-1}(I-\ch^{N-t+1})$ since from Lemma~\ref{lem:7} we know that $\norm{\ch}<1$ for large $T$. Thus we get \eqref{eq:avg_var_new1a}.

\end{proof}

\subsection{Claims}

\begin{claim}
\label{claim:5}
 For $\gamma\leq \frac{1}{2R}$ we have
\begin{\Ieee}{LLL}
\label{eq:prop3_3}
&&\norm{\Ex{2\gamma \Attdiff{t_1-1}{B}^{\top} \sum_{r=1}^{t_2-t_1}\sum_{j=0}^{B-1}\Nt{t_2-r}
{-j}\Xtttr{t_2-r}{-j}\Htt{t_2-r}{j+1}{B-1}\prodHtt{t_2}{r-1}}\indt{0}{N-1}}\\
&\leq & c_1\gamma^2 B^2 t_1 t_2 R d\sigma_{\max} \prbndsq\Ieeen
\end{\Ieee}
for some constant $c_1>0$.
\end{claim}

\begin{proof}

The proof is similar to the proof of Lemma~\ref{lem:cross_term}.

\end{proof}

\section{Proof of Theorem~\ref{th:3} }

\label{sec:proof_thm_last_bias}
\begin{proof}[Proof of Theorem~\ref{th:3}]
We start with the following
\begin{\Ieee}{LLL}
\label{eq:th3_1}
\gram{\Atto{t-1}{b}}& = &\prodHtttr{t}{t}(A_0-\A)^{\top}(A_0-A)\prodHtt{t}{t}\\
& \preceq & \norm{A_0-\A}^2 \prodHtttr{t}{t}\prodHtt{t}{t}\Ieeen
\end{\Ieee}

%
%

From Lemma~\ref{lem:almost_sure_contraction} we can show that there are constants $c_1, c_2>0$ such that 
\begin{\Ieee}{LLL}
\label{eq:claim7_2}
\norm{\Ex{\prodHtttr{t}{t}\prodHtt{t}{t}\indt{0}{t-1}}}\\
\leq \left(1-c_1\gamma B\sigma_{\min}(G)+c_2\gamma B\sqrt{M_4}\prbndsq\right)^{t}.\Ieeen
\end{\Ieee}
Now choosing $T$ such that $\prbndsqinv> \frac{c_2\sqrt{M_4}}{2c_1 \sigma_{\min}(G)}$ we get,
\begin{\Ieee}{LLL}
\label{eq:claim7_2a}
\norm{\Ex{\prodHthtr{t}{t}\prodHth{t}{t}}}\leq \left(1-c_3\gamma B\sigma_{\min}(G)\right)^{t}.\Ieeen
\end{\Ieee}

Thus we get the theorem.

\end{proof}

\section{Proof of Theorem~\ref{th:4}}
\label{sec:proof_prop_avg_bias}
\begin{proof}[Proof of Theorem~\ref{th:4}]
We use the following inequality that is obtained from Lemma \ref{lem:5}
\begin{\Ieee}{LLL}
\label{eq:th4_1}
\gram{\Anoat{a}{N}}\preceq \frac{1}{N-a}\sum_{t=a+1}^{N}\gram{\Atto{t-1}{B}}\Ieeen
\end{\Ieee}

Therefore
\begin{\Ieee}{LLL}
\label{eq:th4_1a}
&&\Ex{\gram{\Anoat{a}{N}}\indt{0}{N-1}}\\
& \preceq & \frac{1}{N-a}\sum_{t=a+1}^{N}\Ex{\gram{\Atto{t-1}{B}}\indt{0}{N-1}}\\
&\preceq & \frac{1}{N-a}\sum_{t=a+1}^{N}\Ex{\gram{\Atto{t-1}{B}}\indt{0}{t-1}}\Ieeen
\end{\Ieee}

Now using theorem \ref{th:3}, we get
\begin{\Ieee}{LLL}
\label{eq:th4_2}
\Ex{\gram{\Anoat{a}{N}}\indt{0}{N-1}} \preceq \\
\left(\frac{1}{N-a}\frac{\left(1-c_1\gamma B\sigma_{\min}(G)\right)^{a+1}}{c_1\gamma B\sigma_{\min}(G)} \right)\norm{A_0-\A}^2 I \Ieeen
\end{\Ieee}

Hence using $1-x\leq e^{-x}$ we get
\begin{\Ieee}{LLL}
\label{eq:th4_3}
\norm{\Ex{\gram{\Anoat{a}{N}}\indt{0}{N-1}}}\\
\leq c\frac{1}{B(N-a)}\frac{e^{-c B \gamma \sigma_{\min}(G)a}}{\gamma \sigma_{\min}(G)} \norm{A_0-\A}^2  \Ieeen
\end{\Ieee}
\end{proof}


\section{Operator Norm Inequalities}
\label{sec:op_norm}

In this section, we develop the concentration inequalities necessary to obtain bounds on $\lossop$. Consider Equation~\eqref{eq:Atv}

\begin{equation}
\Attdiff{t-1}{B} = 2\gamma\sum_{r=1}^{t}\sum_{j=0}^{B-1}\Nt{t-r}{-j}\Xtttr{t-r}{-j} \Htt{t-r}{j+1}{B-1}\prod_{s=r-1}^{1}\Htt{t-s}{0}{B-1}
\end{equation}

Splitting the sum into $r = 1$ and $r = 2,\dots, t$, it is easy to show the following recursion:

\begin{equation}\label{eq:iterate_recursion}
\Attdiff{t-1}{B} = 2\gamma \sum_{j=0}^{B-1}\Nt{t-1}{-j}\Xtttr{t-1}{-j} \Htt{t-1}{j+1}{B-1} + \Attdiff{t-2}{B}\Htt{t-1}{0}{B-1}
\end{equation}

We will consider the matrix $\Delta_{t-1} :=  2\gamma \sum_{j=0}^{B-1}\Nt{t-1}{-j}\Xtttr{t-1}{-j} \Htt{t-1}{j+1}{B-1}$. Recall the sequence of events $\cdt^{t-1}_{-j}$  for $j = 0,1,\dots, B-1$ as defined in Section~\ref{subsec:notations}. We will pick $R$ as in Section~\ref{sec:main_results} so that $\mathbb{P}(\cdt^{t-1}_{-0})$ is close to $1$. 

For the sake of clarity, we drop the dependence on $t$ while stating and proving some of the technical results since the events and random variables considered there are identically distributed for every $t$. That is, consider $\cdt_{-j}$ instead of $\cdt^{t-1}_{-j}$ and 
$$\Delta :=  2\gamma \sum_{j=0}^{B-1}\eta_{-j}\tilde{X}^{\top}_{-j} \Htt{}{j+1}{B-1}$$

We will bound the exponential moment generating function of $\Delta$:

\begin{lemma}\label{lem:conditional_chernoff}
Suppose Assumption~\ref{as:noise_concentration} holds and that $\gamma R < 1$. Let $\lambda \in \mathbb{R}$ and $x,y \in \mathbb{R}^d$ are arbitrary. 
Then, we have:
\begin{enumerate}
\item  \begin{\Ieee}{LLL}
\mathbb{E}\left[ \exp( \gamma \lambda^2 C_{\mu} \langle x, \Sigma x \rangle \langle y, \Htt{\top}{0}{B-1}\Htt{}{0}{B-1} y\rangle + \lambda\langle x, \Delta y\rangle)| \cdt_{-0}\right]\\
 \leq \frac{\exp\left(\gamma\lambda^2 C_{\mu}\langle x, \Sigma x \rangle \|y\|^2  \right)}{\mathbb{P}(\cdt_{-0})} 
\end{\Ieee}
\item $$\mathbb{E}\left[ \exp(\lambda\langle x, \Delta y\rangle)| \cdt_{-0}\right] \leq \frac{\exp\left(\gamma\lambda^2 C_{\mu}\langle x, \Sigma x \rangle \|y\|^2  \right)}{\mathbb{P}(\cdt_{-0})}  $$ 
\end{enumerate}

Where $C_{\mu}$ is as given in Assumption~\ref{as:noise_concentration}
\end{lemma}

\begin{proof}

We will just prove item 1 since item 2 follows from it trivially as $$\gamma \lambda^2 C_{\mu} \langle x, \Sigma x \rangle \langle y, \Htt{\top}{0}{B-1}\Htt{}{0}{B-1} y\rangle \geq 0\,.$$ 
 For the sake of clarity, we will take:
$$ \Xi_0 := \gamma \lambda^2 C_{\mu} \langle x, \Sigma x \rangle \langle y, \Htt{\top}{0}{B-1}\Htt{}{0}{B-1} y\rangle $$
and more generally, 
$$\Xi_k =  \gamma \lambda^2 C_{\mu} \langle x, \Sigma x \rangle \langle y, \Htt{\top}{k}{B-1}\Htt{}{k}{B-1} y\rangle $$

Consider $\Delta_{-k} := 2\gamma \sum_{j=k}^{B-1}\eta_{-j}\tilde{X}^{\top}_{-j} \Htt{}{j+1}{B-1}$. We will first prove the following claim before bounding the exponential moment:

\begin{claim}\label{claim:shaving}
Whenever $\|\Xtt{}{-k}\|^2 \leq R$ and $\gamma R < 1/2$, we have:
$$\Xi_{k} + 2\gamma^2\lambda^2 C_{\mu}\langle x, \Sigma x \rangle \langle y,\Htt{\top}{k+1}{B-1}\Xtt{}{-k} \Xtt{\top}{-k} \Htt{}{k+1}{B-1}y \rangle  \leq \Xi_{k+1}$$
\end{claim}
\begin{proof}
We use the fact that $\Htt{\top}{k}{B-1}\Htt{}{k}{B-1} = \Htt{\top}{k+1}{B-1}(I-2\gamma \Xtt{}{-k} \Xtt{\top}{-k})^2\Htt{}{k+1}{B-1}$ to conclude that:
\begin{align}
&\Xi_{k} + 2\gamma^2\lambda^2 C_{\mu}\langle x, \Sigma x \rangle \langle y,\Htt{\top}{k+1}{B-1}\Xtt{}{-k} \Xtt{\top}{-k} \Htt{}{k+1}{B-1}y \rangle  \nonumber \\
&= \gamma \lambda^2 C_{\mu} \langle x, \Sigma x\rangle \langle y, \Htt{\top}{k+1}{B-1}\left(I-2\gamma \Xtt{}{-k} \Xtt{\top}{-k} + 4\gamma^2 \|\Xtt{}{-k}\|^2\Xtt{}{-k} \Xtt{\top}{-k}  \right)\Htt{}{k+1}{B-1} y \rangle  \nonumber \\
&\leq \gamma \lambda^2 C_{\mu} \langle x, \Sigma x\rangle \langle y, \Htt{\top}{k+1}{B-1}\Htt{}{k+1}{B-1} y \rangle  = \Xi_{k+1}
\end{align}
In the second step we have used the fact that when $\gamma \|\Xtt{}{-k}\|^2 \leq 1/2$, we have that 
$$I-2\gamma \Xtt{}{-k} \Xtt{\top}{-k} + 4\gamma^2 \|\Xtt{}{-k}\|^2\Xtt{}{-k} \Xtt{\top}{-k}  \preceq I$$
\end{proof}

First note that $\Delta = 2\gamma \eta_{0}\tilde{X}^{\top}_{0} \Htt{}{1}{B-1} + \Delta_{-1} $. Now, 
\begin{align}
&\mathbb{E}\left[ \exp(\Xi_0 + \lambda\langle x, \Delta y\rangle)| \cdt_{-0}\right] = \frac{1}{\mathbb{P}(\cdt_{-0})}\mathbb{E}\left[ \exp(\Xi_0 + \lambda\langle x, \Delta y\rangle)\mathbbm{1} \left(\cdt_{-0}\right)\right]\nonumber \\
&= \frac{1}{\mathbb{P}(\cdt_{-0})}\mathbb{E}\left[ \exp\left(\Xi_0 + 2\lambda \gamma\langle x, \eta_{-0} \rangle \langle \Xtt{}{-0},\Htt{}{1}{B-1}y \rangle  +\lambda\langle x, \Delta_{-1} y\rangle\right)\mathbbm{1} \left(\cdt_{-0}\right)\right]\nonumber \\
&\leq \frac{1}{\mathbb{P}(\cdt_{-0})}\mathbb{E}\left[ \exp\left( \Xi_0 + 2\gamma^2\lambda^2 C_{\mu}\langle x, \Sigma x \rangle \langle y,\Htt{\top}{1}{B-1}\Xtt{}{-0} \Xtt{\top}{-0} \Htt{}{1}{B-1}y \rangle  +\lambda\langle x, \Delta_{-1} y\rangle\right)\mathbbm{1} \left(\cdt_{-0}\right)\right]\nonumber \\
&\leq \frac{1}{\mathbb{P}(\cdt_{-0})}\mathbb{E}\left[ \exp\left(\Xi_1  +\lambda\langle x, \Delta_{-1} y\rangle\right)\mathbbm{1} \left(\cdt_{-0}\right)\right]\nonumber \\
&\leq \frac{1}{\mathbb{P}(\cdt_{-0})}\mathbb{E}\left[ \exp\left(\Xi_1  +\lambda\langle x, \Delta_{-1} y\rangle\right)\mathbbm{1} \left(\cdt_{-1}\right)\right]\label{eq:first_step_induction}
\end{align}
In the first step we have used the definition of conditional expectation, in the third step we have used the fact that $\eta_{-0}$ is independent of $\cdt_{-0}$, $\Delta_{-1}$, $\tilde{X}^{\top}_{-0}\Htt{}{1}{B-1}$, and $\Delta_{-1}$ and have applied the sub-Gaussianity from Assumption~\ref{as:noise_concentration}.  In the fourth step, using the fact under the event $\cdt_{-0}$,  $\|\tilde{X}_{-0}\|^2 \leq R$ we have applied Claim~\ref{claim:shaving}. In the final step, we have used the fact that $\cdt_{-0}\subseteq \cdt_{-1}$.  We proceed by induction over Equation~\eqref{eq:first_step_induction} to conclude the result. 

\end{proof}

We now consider the matrix $\Htt{}{0}{B-1}$ under the event $\cdt_{-0}$. 
\begin{lemma}\label{lem:contraction}
Suppose that $\gamma RB < \frac{1}{6}$. Then, under the event $\cdt_{-0}$, we have: 

$$I - 4\gamma\left(1+\tfrac{2\gamma BR}{1-4\gamma B R}\right)\sum_{i=0}^{B-1}\Xtt{}{-i}\Xtt{\top}{-i} \preceq \Htt{\top}{0}{B-1}\Htt{}{0}{B-1}\preceq I - 4\gamma\left(1-\tfrac{2\gamma BR}{1-4\gamma B R}\right)\sum_{i=0}^{B-1}\Xtt{}{-i}\Xtt{\top}{-i}$$
\end{lemma}
\begin{proof}
By definition, we have: $\Htt{}{0}{B-1} = \prod_{j=0}^{B-1}(I-2\gamma \Xtt{}{-j}\Xtt{\top}{-j})$. Expanding out the product, we get an expression of the form:
\begin{equation}\label{eq:series_expansion}
\Htt{\top}{0}{B-1}\Htt{}{0}{B-1} = I - 4\gamma \sum_{i=0}^{B-1}\Xtt{}{-i}\Xtt{\top}{-i} + (2\gamma)^2 \sum_{i,j} \Xtt{}{-i}\Xtt{\top}{-i}\Xtt{}{-j}\Xtt{\top}{-j} + \dots
\end{equation}
Here, the summation $\sum_{i,j}$ is over all possible combinations possible when the product is expanded and $\dots$ denotes higher order terms of the form $\Xtt{}{-i_1}\Xtt{\top}{-i_1}\dots\Xtt{}{-i_k}\Xtt{\top}{-i_k}$

\begin{claim}\label{claim:AM-GM}
Assume $k \geq 2$ and $i_1,\dots,i_k \in \{0,\dots,B-1\}$. Under the event $\cdt_{-0}$, for any $x \in \mathbb{R}^d$, we have:
$$\biggr|x^{\top}\Xtt{}{-i_1}\Xtt{\top}{-i_1}\dots\Xtt{}{-i_k}\Xtt{\top}{-i_k}x\biggr| \leq \frac{R^{k-1}}{2}\left[x^{\top}\Xtt{}{-i_1}\Xtt{\top}{-i_1}x + x^{\top}\Xtt{}{-i_k}\Xtt{\top}{-i_k}x \right]$$
\end{claim}
\begin{proof}
This follows from an application of AM-GM inequality. It is clear by Cauchy-Schwarz inequality that $|\langle \Xtt{}{i_l}, \Xtt{}{i_{l+1}}\rangle | \leq R$, which implies:

$$\biggr|x^{\top}\Xtt{}{-i_1}\Xtt{\top}{-i_1}\dots\Xtt{}{-i_k}\Xtt{\top}{-i_k}x\biggr| \leq R^{k-1}\biggr|\left[ x^{\top}\Xtt{}{-i_1} \Xtt{\top}{-i_k}x\right] \biggr| \leq \frac{R^{k-1}}{2}\left[ \langle x,\Xtt{}{-i_1}\rangle^2 +\langle\Xtt{}{-i_k},x\rangle^2\right] \,.$$
Where the last inequality follows from an application of the AM-GM inequality.
\end{proof}

From Claim~\ref{claim:AM-GM}, we conclude that:
$$\sum_{i_1,\dots,i_k}\Xtt{}{-i_1}\Xtt{\top}{-i_1}\dots\Xtt{}{-i_k}\Xtt{\top}{-i_k} \preceq (2B)^{k-1}R^{k-1} \sum_{i=0}^{B-1} \Xtt{}{-i}\Xtt{\top}{-i}$$

Plugging this into Equation~\eqref{eq:series_expansion}, we have that under the event $\cdt_{-0}$:

\begin{align}
\Htt{\top}{0}{B-1}\Htt{}{0}{B-1} &\preceq I - 4\gamma \sum_{i=0}^{B-1}\sum_{i=0}^{B-1}\Xtt{}{-i}\Xtt{\top}{-i} + \sum_{k=2}^{2B} (2\gamma)^k(2B)^{k-1}R^{k-1}\sum_{i=0}^{B-1} \Xtt{}{-i}\Xtt{\top}{-i} \nonumber \\
&\preceq I - 4\gamma \sum_{i=0}^{B-1}\sum_{i=0}^{B-1}\Xtt{}{-i}\Xtt{\top}{-i} + 2\gamma \frac{4\gamma B R}{1-4\gamma B R}\sum_{i=0}^{B-1}\sum_{i=0}^{B-1}\Xtt{}{-i}\Xtt{\top}{-i}
\end{align}
Here we have used the fact that $4\gamma BR < 1$ to convert the finite sum to an infinite sum. Using the bound on $\gamma$, we conclude the upper bound. The lower bound follows with a similar proof. 

\end{proof}

\begin{lemma}\label{lem:almost_sure_contraction}
Suppose $\gamma B R < \frac{1}{6}$. Let $G := \mathbb{E}\Xtt{}{-i}\Xtt{\top}{-i}$ and $M_4 := \mathbb{E}\bigr\|\Xtt{}{-i}\bigr\|^4$. Then, we have:

\begin{\Ieee}{LLL}
\mathbb{E}\left[\Htt{\top}{0}{B-1}\Htt{}{0}{B-1}\bigr|\cdt_{-0}\right] &\preceq & I - \frac{4\gamma B}{\mathbb{P}(\cdt_{-0})}\left(1-\tfrac{2\gamma BR}{1-4\gamma B R}\right)G + \\
&&\frac{4\gamma B\sqrt{M_4(1-\mathbb{P}(\cdt_{-0}))}}{\mathbb{P}(\cdt_{-0})}\left(1-\tfrac{2\gamma BR}{1-4\gamma B R}\right)I
\end{\Ieee}

\end{lemma}

\begin{proof}
The result follows from the statement of Lemma~\ref{lem:contraction}, once we show the following inequality via Cauchy Schwarz inequality and the definition of conditional expectation: 
 $$\mathbb{E}\left[\Xtt{}{-i}\Xtt{\top}{-i}\bigr|\cdt_{-0}\right]\succeq \frac{G}{\mathbb{P}(\cdt_{-0})} - I \frac{\sqrt{\mathbb{E}\bigr\|\Xtt{}{-i}\bigr\|^4}\sqrt{1-\mathbb{P}(\cdt_{-0})}}{\mathbb{P}(\cdt_{-0})}.$$  
\end{proof}

Now we will show that $\Htt{}{0}{B-1}$ contracts any given vector with probability at-least $p_0 > 0$. For this we will refer to lemma~\ref{lem:data_subgaussianity} where it is shown that if $X\sim\pi$ then $\langle X, x\rangle$ has mean $0$ and is sub-Gaussian with variance proxy $C_{\mu}x^{\top} Gx$. Using this will show that the matrix $\Htt{}{0}{B-1}$ operating on a given vector $x$ contracts it with a high enough probability.

\begin{lemma}\label{lem:probable_contraction}

Suppose $\gamma RB < \frac{1}{8}$ and that $\mu$ obeys Assumption~\ref{as:noise_concentration}.  There exists a constant $c_0 > 0$ which depends only on $C_\mu$ such that whenever $1-\mathbb{P}(\cdt_{-0}) \leq c_0$, then for any arbitrary $x \in \mathbb{R}^2$
$$\mathbb{P}\left(\|\Htt{}{0}{B-1}x\|^2 \geq \|x\|^2- B\gamma  x^{\top}G x\bigr|\cdt_{-0}\right) \leq 1-p_0 < 1\,.$$
Where $p_0 > 0$ depends only on $C_{\mu}$.
\end{lemma}
\begin{proof}
Initially we do not condition on $\cdt_{-0}$. Consider the quantity: $Y:= \sum_{i=0}^{B-1} \langle x, \Xtt{}{-i}\rangle^2$. 

\begin{claim}\label{claim:payley_zygmund}
$$ \mathbb{P}\left( Y \geq 1/2 B x^{\top}G x \right)  \geq q_0$$
where $q_0 > 0$ depends only on sub-Gaussianity parameter $C_{\mu}$
\end{claim}
\begin{proof}
We consider the Payley-Zygmund inequality which states that for any positive random variable $Y$ with a finite second moment, we have:

$$\mathbb{P}\left(Y > \tfrac{1}{2}\mathbb{E}Y\right) \geq \frac{1}{4}\frac{(\mathbb{E}Y)^2}{\mathbb{E}Y^2}.$$
Note that $\mathbb{E}Y = B x^{\top}Gx$. The statement of the lemma follows once we lower bound the quantity $\frac{(\mathbb{E}Y)^2}{\mathbb{E}Y^2}$. Clearly, $(\mathbb{E}Y)^2 = B^2 x^{\top}Gx$. Now, \begin{align}
\mathbb{E}Y^2 &= \sum_{i,j}\mathbb{E}\langle x,X_i\rangle^2 \langle x,X_j\rangle^2 \leq  \sum_{i,j}\sqrt{\mathbb{E}\langle x,X_i\rangle^4} \sqrt{\mathbb{E}\langle x,X_j\rangle^2} = B^2 \mathbb{E} \langle x,X_i\rangle^4 \nonumber \\
&\leq B^2 c_1 C_{\mu}^2 (x^{\top}Gx)^2 \label{eq:fourth_moment_bound}
\end{align}
Here, the second step follows from Cauchy-Schwarz inequality. The third step follows from the fact that $X_i$ are all identically distributed. The fourth step follows from Lemma~\ref{lem:data_subgaussianity} and Theorem 2.1 from \citep{boucheron2013concentration}. The statement of the claim follows once we apply Payley-Zygmund inequality.
\end{proof}

Now, by definition of conditional probabililty and Claim~\ref{claim:payley_zygmund}, we have:
$$\mathbb{P}\left(\sum_{i=0}^{B-1}\langle x, \Xtt{}{-i}\rangle^2 \leq  \frac{B}{2} x^{T}Gx \biggr|\cdt_{-0}\right) \leq \frac{(1-q_0)}{\mathbb{P}(\cdt_{-0})}$$

Now the statement of the lemma follows from an application of Lemma~\ref{lem:contraction}
\end{proof}


Now we want to bound the operator norm of $\prod_{s=a}^{a+b}\Htt{s}{0}{B-1}$ with high probability under the event $\cap_{s=a}^{a+b} \cdt_{-0}^{s}$. 
\begin{lemma}\label{lem:operator_norm_bound_1}
Suppose the conditions in Lemma~\ref{lem:probable_contraction} hold. Let $\lmin{G}$ denote the smallest eigenvalue of $G$. We also assume that $\mathbb{P}(\cdt^{a,b}) > 1/2$. Conditioned on the event $\cdt^{a,b}$,
\begin{enumerate}
\item $\|\prod_{s=a}^{b}\Htt{s}{0}{B-1}\| \leq 1$ almost surely
\item  Whenever $b-a+1$ is larger than some constant which depends only on $C_{\mu}$, we have:
$$
\mathbb{P}\left(\|\prod_{s=a}^{b}\Htt{s}{0}{B-1}\| \geq 2(1-\gamma B\lmin{G})^{c_4(b-a+1)}\biggr|\cdt^{a,b}\right)
\leq \exp(-c_3 (b-a+1)+c_5d)$$
Where $c_3,c_4$ and $c_5$ are constants which depend only on $C_{\mu}$
\end{enumerate}

\end{lemma}

\begin{proof}
\noindent
\begin{enumerate}
\item The proof follows from an application of Lemma~\ref{lem:contraction}.
\item We will prove this with an $\epsilon$ net argument over the sphere in $\mathbb{R}^d$ dimensions. 

Suppose we have arbitrary $x\in \mathbb{R}^d$ such that $\|x\| = 1$. Conditioned on the event $\cdt^{a,b}$, the matrices $\Htt{s}{0}{B-1}$ are all independent for $a \leq s \leq b$. We also note that $\Htt{s}{0}{B-1}$ is independent of $\tilde{\mathcal{D}}^{t}$ for $t \neq s$.
Let $K_{v} := \prod_{s=v}^{b}\Htt{s}{0}{B-1}$. When $v \geq b+1$, we take this product to be identity. Consider the set of events $\mathcal{G}_v := \{\|\Htt{v}{0}{B-1}K_{v+1}x\|^2 \leq \|K_{v+1}x\|^2 (1-\gamma B\lmin{G} \}$. From Lemma~\ref{lem:probable_contraction}, we have that whenever $v \in (a,b)$:
\begin{equation}\label{eq:contraction_improbability}
\mathbb{P}(\mathcal{G}^{c}_v|\tilde{\mathcal{D}}^{v},\Htt{s}{0}{B-1}: s\neq v) \leq 1-p_0
\end{equation}
Where $p_0$ is given in Lemma~\ref{lem:probable_contraction}

Let $D \subseteq \{a,\dots,b\}$ such that $|D| = r$. It is also clear from item 1 and the definitions above that whenever the event $\cap_{v \in D}\mathcal{G}_v$ holds, we have:
\begin{equation}\label{eq:norm_contraction}
\|\prod_{s=a}^{b}\Htt{s}{0}{B-1}x\| \leq (1-\gamma B\lmin{G})^\frac{r}{2}\,.
\end{equation}
Therefore, whenever Equation~\eqref{eq:norm_contraction} is violated, we must have a set $D^{c} \subseteq \{a,\dots,b\}$ such that $|D^{c}| \geq b-a-r$ and the event $\cap_{v \in D^{c}} \mathcal{G}^{c}_v$ holds. We will union bound all such events indexed by $D^{c}$ to obtain an upper bound on the probability that Equation~\eqref{eq:norm_contraction} is violated. Therefore, using Equation~\eqref{eq:contraction_improbability} along with the union bound, we have:

$$\mathbb{P}\left(\|\prod_{s=a}^{b}\Htt{s}{0}{B-1}x\| \geq (1-\gamma B\lmin{G})^\frac{r}{2}\biggr|\cdt^{a,b}\right) \leq \binom{b-a+1}{b-a-r}(1-p_0)^{b-a-r}$$
Whenever $b-a+1$ is larger than some constant depending only on $C_{\mu}$, we can pick $r = c_2(b-a+1)$ for some constant $c_2 > 0$ small enough such that:
$$\mathbb{P}\left(\|\prod_{s=a}^{b}\Htt{s}{0}{B-1}x\| \geq (1-\gamma B\lmin{G})^\frac{r}{2}\biggr|\cdt^{a,b}\right) \leq \exp(-c_3(b-a+1))$$

Now, let $\mathcal{N}$ be a $1/2$-net of the sphere $\mathcal{S}^{d-1}$. Using Corollary 4.2.13 in \citep{vershynin2018high}, we can choose $|\mathcal{N}| \leq 6^d$. By Lemma 4.4.1 in \citep{vershynin2018high} we show that:

\begin{equation}\label{eq:constant_vector_contraction}
\|\prod_{s=a}^{b}\Htt{s}{0}{B-1}\| \leq 2 \sup_{x\in \mathcal{N}}\|\prod_{s=a}^{b}\Htt{s}{0}{B-1}x\|
\end{equation}

By union bounding Equation~\eqref{eq:constant_vector_contraction} for every $x \in \mathcal{N}$, we conclude that:

\begin{align}
&\mathbb{P}\left(\|\prod_{s=a}^{b}\Htt{s}{0}{B-1}\| \geq 2(1-\gamma B\lmin{G})^{c_4(b-a+1)}\biggr|\cdt^{a,b}\right)\leq |\mathcal{N}|\exp(-c_3(b-a+1)) \nonumber \\
&= \exp(-c_3 (b-a+1)+c_5d)
\end{align}
\end{enumerate}

\end{proof}

Now we will give a high probability bound for the following operator:
\begin{equation}\label{eq:geometric_stochastic_op}
F{a,N}:= \sum_{r=a}^{N-1} \prod_{s=a+1}^{r}\Htt{s}{0}{B-1}
\end{equation}
Here, we use the convention that $\prod_{s = a+1}^{a}\Htt{s}{0}{B-1} = I$

\begin{lemma}\label{lem:operator_norm_bound_2}
Suppose $c_4\gamma B\lmin{G} < \frac{1}{4}$ for the constant $c_4$ as given in Lemma~\ref{lem:operator_norm_bound_1}. Suppose all the conditions given in the statement of Lemma~\ref{lem:operator_norm_bound_1} hold. Then, for any $\delta \in (0,1)$, we have:
$$\mathbb{P}\left(\|F_{a,N}\| \geq C\left(d + \log \frac{N}{\delta} + \frac{1}{\gamma B \lmin{G}}\right)\biggr|\cdt^{a,N-1}\right) \leq \delta $$
Where $C$ is a constant which depends only on $C_{\mu}$
\end{lemma}

\begin{proof}

We consider the triangle inequality: $\norm{F_{a,N}} \leq \sum_{t=a}^{N-1} \norm{\prod_{s=a+1}^{t}\Htt{s}{0}{B-1}}$. By Lemma~\ref{lem:operator_norm_bound_1}, we have that whenever $t-a \geq \frac{c_5d}{c_3} + \frac{\log \frac{N}{\delta}}{c_3} $:

$$ \mathbb{P}\left(\|\prod_{s=a+1}^{t}\Htt{s}{0}{B-1}\| \geq 2(1-\gamma B\lmin{G})^{c_4(t-a)}\biggr|\cdt^{a,N-1}\right) \leq \frac{\delta}{N}$$

Using union bound, we show that when conditioned on $\cdt^{a,N-1}$, with probability at least $1-\delta$ the following holds:
\begin{enumerate}
\item For all $a \leq t \leq N-1$ such that $t -a \geq \frac{c_5d}{c_3} + \frac{\log \frac{N}{\delta}}{c_3}$:
$$\|\prod_{s=t}^{N}\Htt{s}{0}{B-1}\| \leq 2(1-\gamma B\lmin{G})^{c_4(t-a)} $$
\item For all $t$ such that $t -a < \frac{c_5d}{c_3} + \frac{\log \frac{N}{\delta}}{c_3}$, we have: $\|\prod_{s=t}^{N}\Htt{s}{0}{B-1}\| \leq 1.$
For this, we use the almost sure bound given in item 1 of Lemma~\ref{lem:operator_norm_bound_1}
\end{enumerate}
 
Therefore, when conditioned on $\cdt^{a,N-1}$, with probability at least $1-\delta$ we have:

\begin{align}
\|F_{a,N}\| &\leq C(d + \log \frac{N}{\delta}) + 2\sum_{j = 0}^{\infty} (1-\gamma B \lmin{G})^{c_4 j} \nonumber \\
&\leq C(d + \log \frac{N}{\delta}) + 2\sum_{j = 0}^{\infty} \exp(-c_4 j \gamma B \lmin{G}) \nonumber \\
&\leq  C(d + \log \frac{N}{\delta}) + \frac{2}{ 1 - \exp(-c_4 \gamma B \lmin{G})} \nonumber \\
&\leq  C(d +\log \frac{N}{\delta}) + \frac{2}{  c_4 \gamma B \lmin{G} - \tfrac{c_4^2 \gamma^2 B \lmin{G}}{2}} \nonumber \\
&\leq C\left(d + \log \frac{N}{\delta} + \frac{1}{\gamma B \lmin{G}}\right)
\end{align} 
In the first step, we have used the event described above to bound the operator norm via. the infinite geometric series. In the second step, we have used the inequality $(1-x)^{a} \leq \exp(-ax)$ whenever $x \in [0,1]$ and $a > 0$. In the fourth step, we have used the inequality $\exp(-x) \leq 1-x + \frac{x^2}{2}$ whenever $x\in [0,1]$. In the last step, we have absorbed constants into a single constant $C$
\end{proof}

We will now consider the averaged iterate of the coupled process as defined in Equation~\eqref{eq:tail_variance} with $a=0$.
\begin{equation}\label{eq:average_definition}
\Anovt := \frac{1}{N}\sum_{t=1}^{N}\Attdiff{t-1}{B} 
\end{equation}
We recall the definition of $\Delta_{t-1}$ from the beginning of the Section~\ref{sec:op_norm} and the recursion shown in Equation~\eqref{eq:iterate_recursion}.  We combine these with Equation~\eqref{eq:average_definition} to show:

\begin{equation} \label{eq:average_unfurled}
\Anovt = \frac{1}{N}\sum_{t=1}^{N} \Delta_{t-1} F_{t-1,N}
\end{equation}
Where $F_{a,N}$ is as defined in Equation~\eqref{eq:geometric_stochastic_op}. Using the results in Lemma~\ref{lem:conditional_chernoff} and a similar proof technique we show the following theorem. We define the following event as considered in Lemma~\eqref{lem:operator_norm_bound_2}:
 $$\tilde{\mathcal{M}}^{t-1} := \left\{\|F_{t-1,N}\| \leq C\left(d + \log \frac{N}{\delta}+ \tfrac{1}{\gamma B \lmin{G}}\right)\right\}$$

Define the event $\tilde{\mathcal{M}}^{0,N-1} = \cap_{t=0}^{N-1}\tilde{\mathcal{M}}^t$ and recall the definition of the event $\cdt^{0,N-1}$.
 
\begin{theorem}\label{thm:average_iterate_op_bound}
 We suppose that the conditions in Lemmas~\ref{lem:conditional_chernoff},~\ref{lem:operator_norm_bound_2} and~\ref{lem:contraction} hold. We also assume that $\mathbb{P}(\tilde{\mathcal{M}}^{0,N-1}\cap \cdt^{0,N-1}) \geq \frac{1}{2}$. Define $\alpha := C(d + \log \frac{N}{\delta} +\tfrac{1}{\gamma B \lmin{G}} )$ as in the definition of the event $\tilde{\mathcal{M}}^t$ 
$$\mathbb{P}\left(\|\Anovt\|> \beta\biggr| \tilde{\mathcal{M}}^{0,N-1} \cap \cdt^{0,N-1}\right) \leq \exp\left(c_1 d -\frac{\beta^2N}{16\gamma C_{\mu} \lmax{\Sigma}(1+2\alpha)} \right) \,.$$

\end{theorem}
\begin{proof}

Recall the events $\tilde{\mathcal{D}}^{t,N-1}$ and define $\tilde{\mathcal{M}}^{t,N-1} := \cap_{s=t}^{N-1} \tilde{\mathcal{M}}^{t}$. We recall that $\Delta_{t-1}$ is independent of $F_{t-1,N}$ and $\tilde{\mathcal{D}}^{t,N-1}$. 
 Now consider arbitrary $x,y \in \mathbb{R}^d$ such that $\|x\| = \|y\| = 1$. Define $\Gamma_{t-1,N-1} := \frac{1}{N}\sum_{s=t}^{N} \Delta_{s-1} F_{s-1,N}$. For any $\lambda > 0$, consider the following exponential moment:

\begin{align}
&\mathbb{E} \left[\exp\left(\lambda \langle x, (\Anovt)y\rangle \right)\biggr|\tilde{\mathcal{M}}^{0,N-1} \cap \cdt^{0,N-1}\right]\nonumber\\ &= \frac{\mathbb{E} \left[\exp\left(\lambda \langle x, (\Anovt)y\rangle \right)\mathbbm{1}\left(\tilde{\mathcal{M}}^{0,N-1} \cap \cdt^{0,N-1}\right)\right]}{\mathbb{P}\left(\tilde{\mathcal{M}}^{0,N-1} \cap \cdt^{0,N-1}\right)} \nonumber\\
&=\frac{\mathbb{E} \left[\exp\left(\tfrac{\lambda}{N}\langle x, \Delta_0 F_{0,N}y\rangle + \lambda\langle x, \Gamma_{1,N-1}y\rangle  \right)\mathbbm{1}\left(\tilde{\mathcal{M}}^{0,N-1} \cap \cdt^{0,N-1}\right)\right]}{\mathbb{P}\left(\tilde{\mathcal{M}}^{0,N-1} \cap \cdt^{0,N-1}\right)} \label{eq:average_chernoff_1}
\end{align}
 Here, we note that $\Delta_0$ is independent of $\tilde{\mathcal{M}}^{0,N-1}$, $F_{0,N}$ and $\tilde{D}^{1,N-1}$. We integrate out $\Delta_0$ in Equation~\eqref{eq:average_chernoff_1} using item 2 of Lemma~\ref{lem:conditional_chernoff} by using the fact that $\tilde{\mathcal{D}}^{0,N-1} = \tilde{\mathcal{D}}^{1,N-1}\cap \tilde{\mathcal{D}}_{-0}^{0}$ to show:

\begin{align}
&\mathbb{E} \left[\exp\left(\lambda \langle x, (\Anovt)y\rangle \right)\biggr|\tilde{\mathcal{M}}^{0,N-1} \cap \cdt^{0,N-1}\right] \nonumber \\
 &\leq
 \frac{\mathbb{E} \left[\exp\left( \gamma \frac{\lambda^2 C_{\mu}}{N^2} \langle x,\Sigma x\rangle \|F_{0,N}y\|^2+ \lambda\langle x, \Gamma_{1,N-1}y\rangle  \right)\mathbbm{1}\left( \tilde{\mathcal{M}}^{0,N-1} \cap \tilde{\mathcal{D}}^{1,N-1}\right)\right]}{\mathbb{P}\left(\tilde{\mathcal{M}}^{0,N-1} \cap \cdt^{0,N-1}\right)} 
\label{eq:average_chernoff_2}
\end{align}
We use the fact that $F_{0,N} = I + \Htt{1}{0}{B-1} F_{1,N}$ to conclude:
$\|F_{0,N} y\|^2 = \|y\|^2 + 2\langle y, \Htt{1}{0}{B-1} F_{1,N} y\rangle + \langle y, F^{T}_{1,N} \Htttr{1}{0}{B-1}\Htt{1}{0}{B-1} F_{1,N}y\rangle $.  Under the event $\tilde{\mathcal{M}}^{0,N-1}\cap \tilde{\mathcal{D}}^{1,N-1}$, we have: $\|\Htt{1}{0}{B-1}\| \leq 1$ and $\|F_{1,N}\| \leq \alpha$. Therefore,
$\|F_{0,N} y\|^2 \leq \|y\|^2(1+2\alpha) + \langle y, F^{T}_{1,N} \Htttr{1}{0}{B-1}\Htt{1}{0}{B-1} F_{1,N}y\rangle$. Using this in Equation~\eqref{eq:average_chernoff_2}, we conclude:
{\small
\begin{align}
&\mathbb{P}\left(\tilde{\mathcal{M}}^{0,N-1} \cap \cdt^{0,N-1}\right)\mathbb{E} \left[\exp\left(\lambda \langle x, (\Anovt)y\rangle \right)\biggr|\tilde{\mathcal{M}}^{0,N-1} \cap \cdt^{0,N-1}\right] \nonumber \\
 &\leq
\mathbb{E} \left[\exp\left( \Omega+\lambda\langle x, \Gamma_{1,N-1}y\rangle  \right)\mathbbm{1}\left( \tilde{\mathcal{M}}^{0,N-1} \cap \tilde{\mathcal{D}}^{1,N-1}\right)\right] \nonumber \\
 &\leq \mathbb{E} \left[\exp\left( \Omega+ \lambda\langle x, \Gamma_{1,N-1}y\rangle  \right)\mathbbm{1}\left( \tilde{\mathcal{M}}^{1,N-1} \cap \tilde{\mathcal{D}}^{1,N-1}\right)\right], 
\label{eq:average_chernoff_3}
\end{align}}
where $\Omega := \gamma \frac{\lambda^2 C_{\mu}}{N^2} \langle x,\Sigma x\rangle (1+2\alpha)\|y\|^2 + \gamma \frac{\lambda^2 C_{\mu}}{N^2} \langle x,\Sigma x\rangle \langle y, F^{T}_{1,N} \Htttr{1}{0}{B-1}\Htt{1}{0}{B-1} F_{1,N}y\rangle $. In the last step we have used the fact that $ \tilde{\mathcal{M}}^{0,N-1} \cap \tilde{\mathcal{D}}^{1,N-1}\subseteq \tilde{\mathcal{M}}^{1,N-1} \cap \tilde{\mathcal{D}}^{1,N-1}  $. We continue just like before but use item 1 of Lemma~\ref{lem:conditional_chernoff} instead of item 2 to keep peeling terms of the form $\langle x,\Delta_{t-1} F_{t-1,N} y\rangle$ to conclude:
\begin{align}
\mathbb{E} \left[\exp\left(\lambda \langle x, (\Anovt)y\rangle \right)\biggr|\tilde{\mathcal{M}}^{0,N-1} \cap \cdt^{0,N-1}\right] &\leq 2\exp\left( \gamma \frac{\lambda^2 C_{\mu}}{N} \langle x,\Sigma x\rangle (1+2\alpha)\|y\|^2\right) \nonumber \\
&\leq  2\exp\left( \gamma \frac{\lambda^2 C_{\mu}}{N} \lmax{\Sigma} (1+2\alpha)\right) \label{eq:full_average_chernoff}
\end{align}
Where $\lmax{\Sigma}$ is the maximum eigenvalue of the covariance matrix $\Sigma$. Here we have used the assumption that $\mathbb{P}\left(\tilde{\mathcal{M}}^{0,N-1} \cap \cdt^{0,N-1}\right) \geq \frac{1}{2}$ and the fact that $\|x\| = \|y\| =1$. We apply Chernoff bound to $\langle x,(\Anovt)y \rangle $ using Equation~\eqref{eq:full_average_chernoff} to conclude that for any $\beta, \lambda \in \mathbb{R}^{+}$
\begin{equation}
\mathbb{P}\left( \langle x,(\Anovt)y \rangle  > \beta \biggr|\tilde{\mathcal{M}}^{0,N-1} \cap \cdt^{0,N-1}  \right) \leq  2\exp\left( \gamma \frac{\lambda^2 C_{\mu}}{N} \lmax{\Sigma}\rangle (1+2\alpha) - \beta \lambda\right)
\end{equation}
Choose $\lambda = \frac{N\beta}{2\gamma C_{\mu}\lmax{\Sigma}(1+2\alpha)}$ to conclude:
$$\mathbb{P}\left( \langle x,(\Anovt)y \rangle  > \beta \biggr|\tilde{\mathcal{M}}^{0,N-1} \cap \cdt^{0,N-1}  \right) \leq  2\exp\left(-\frac{\beta^2N}{4\gamma C_{\mu} \lmax{\Sigma}(1+2\alpha)}\right)$$

We now apply an $\epsilon$ net argument just like in Lemma~\ref{lem:operator_norm_bound_1}. Suppose $\mathcal{N}$ is a $1/4$-net of the sphere in $\mathbb{R}^{d}$. By Corollary 4.2.13 in \citep{vershynin2018high}, we can choose $|\mathcal{N}| \leq 12^d$. By Exercise 4.4.3 in \citep{vershynin2018high}, we conclude that:
$$\|\Anovt\| \leq 2 \sup_{x,y \in \mathcal{N}} \langle x,(\Anovt)y \rangle. $$
Therefore, 
\begin{align}
&\mathbb{P}\left( \|\Anovt\|  > \beta \biggr|\tilde{\mathcal{M}}^{0,N-1} \cap \cdt^{0,N-1}  \right)  \nonumber \\
&\leq \mathbb{P}\left( \sup_{x,y \in \mathcal{N}} \langle x,(\Anovt)y \rangle  > \frac{\beta}{2} \biggr|\tilde{\mathcal{M}}^{0,N-1} \cap \cdt^{0,N-1}  \right) \nonumber \\
 &\leq  |\mathcal{N}|^2 \sup_{x,y \in \mathcal{N}} \mathbb{P}\left( \langle x,(\Anovt)y \rangle  > \frac{\beta}{2} \biggr|\tilde{\mathcal{M}}^{0,N-1} \cap \cdt^{0,N-1}  \right) \nonumber \\
&\leq 2(12)^{2d} \exp\left(-\frac{\beta^2N}{16\gamma C_{\mu} \lmax{\Sigma}(1+2\alpha)}\right) \leq \exp\left(c_1 d -\frac{\beta^2N}{16\gamma C_{\mu} \lmax{\Sigma}(1+2\alpha)} \right)
\end{align}

\end{proof}

\section{Lower Bounds}
\label{sec:lower_bounds}
%

Consider the notations as defined in Section~\ref{sec:main_results}. The idea behind the proof is to consider an appropriate Bayesian error lower bound to the minimax error. To construct such a prior distribution, we consider binary tuples $M = (M_{ij} \text{ for } i,j \in [d], i< j) \in \{0,1\}^{d(d-1)/2}$ and $ \epsilon \in (0,\frac{1}{4d}) $. We construct the symmetric matrix corresponding to $M$, denoted by $A(M)$ as:

\begin{equation}
A(M)_{ij} = \begin{cases}
\frac{1}{2} \text{ if } i = j \\
\frac{1}{4d} - \epsilon M_{ij} \text{ if } i < j
\end{cases}
\end{equation}

 For the sake of clarity, we denote $\losspred(\cdot;A(M),\nn(0,\sigma^2I))$ by $\losspred(\cdot;M)$. We use $\pi_M$ to denote the stationary distribution of $\var(A(M),\nn(0,\sigma^2I))$ and the data co-variance matrix at stationarity to be $G_M := \mathbb{E}_{X\sim \pi_M} XX^{\top}$. By $(Z_t)\sim M$, we mean $(Z_1,\dots,Z_T) \sim \var(A(M),\nn(0,\sigma^2I))$. We will first list some useful results in the following Lemmas:

\begin{lemma}
\label{lem:loss_identity}
Suppose Assumption~\ref{as:norm_condition} holds for $\var(\A,\mu)$ and let its stationary distribution be $\pi$. Let $G := \mathbb{E}_{X\sim\pi}XX^{\top}$. Then, 
$$\losspred(A) - \losspred(\A) = \tr\left[(A-\A)^{\top}(A-\A)G\right]$$
\end{lemma}

\begin{lemma}\label{lem:covar_ineq_lb}
For every $M \in \{0,1\}^{d(d-1)/2}$ we have:
$$  \sigma^2 I \preceq G_M \preceq 3\sigma^2I $$
\end{lemma}
\begin{proof}
First we note by Gershgorin circle theorem that $\|A(M)\| \leq \frac{3}{4}$.
Given a stationary sequence $(Z_{0},\dots,Z_T) \sim M$ and the corresponding noise sequence $\eta_0,\dots, \eta_T \sim \nn(0,\sigma^2 I)$ i.i.d, we have by stationarity definition:
 $Z_{t+1} = A(M) Z_t + \eta_t$ and  $Z_{t+1}, Z_t$ are both stationary. Therefore:
$$G_M = \mathbb{E}Z_{t+1}Z_{t+1}^{\top} = A(M)\mathbb{E}Z_tZ_t^{\top} A(M)^{\top} + \mathbb{E} \eta_t \eta_t^{\top} = A(M)G_M A(M)^{\top} + \sigma^{2}I\,.$$

From this we conclude that $G_M \succeq \sigma^2 I$. Now, expanding the recursion above, we have:
\begin{align}
G_M &= \sigma^2 \sum_{i=0}^{\infty} A(M)^i (A(M)^{\top})^i \preceq \sigma^2 \sum_{i=0}^{\infty} \left(\frac{9}{16}\right)^i I =  \frac{16\sigma^2}{7} I
\end{align}  
In the second step we have the fact that $\|A(M)\| \leq \frac{3}{4}$ to show that $A(M)^i (A(M)^{\top})^i \preceq \left(\frac{9}{16}\right)^i I $
\end{proof}

Suppose $M$ and $M^{\prime}$ are such that their Hamming distance is $1$ (i.e, $A(M)$ and $A(M^{\prime})$ differ in exactly two places). We want to bound the total variation distance between the corresponding stationary sequences $(Z_0,Z_1,\dots,Z_T) \sim \var(A(M),\nn(0,\sigma^2I))$ and $(Z^{\prime}_0,Z^{\prime}_1,\dots,Z^{\prime}_T) \sim \var(A(M^{\prime}),\nn(0,\sigma^2I))$.

\begin{lemma}
\label{lem:KL_inequality}
Let the quantities be as defined above. 
For some universal constant $c$, whenever $\epsilon < c\min(\frac{1}{\sqrt{T}},\frac{1}{d})$, we have:
$$TV\left((Z_0,\dots,Z_T),(Z^{\prime}_0,\dots,Z^{\prime}_T)\right) \leq \frac{1}{2}$$

By the existence of maximal coupling (see Chapter I, Theorem 5.2 in \citep{lindvall2002lectures}), we conclude that we can define $(Z_0,\dots,Z_T)$ and $(Z^{\prime}_0,\dots, Z^{\prime}_T)$ on a common probability space such that:
$$\mathbb{P}((Z_0,\dots,Z_T) =(Z^{\prime}_0,\dots, Z^{\prime}_T)) \geq \frac{1}{2}$$

\end{lemma}
\begin{proof}
We will first bound the KL divergence between the two distributions and infer the bound on TV distance from Pinsker's inequality. Consider $p_{M,T}$ and $p_{M^{\prime},T}$ to be the respective probability density functions of $(Z_0,\dots,Z_T) \sim M$ and $(Z^{\prime}_0,\dots, Z^{\prime}_T)\sim M^{\prime}$ respectively. In this proof, we will use $Z_{t,-}$ to denote the tuple $(Z_0,\dots,Z_{t})$. Now, by definition of KL divergence, we have:
\begin{align}
\kl{p_{M,T}}{p_{M^{\prime},T}} &= \mathbb{E}_{Z\sim p_{M,T}} \log \frac{p_{M,T}(Z_0,\dots,Z_T)}{p_{M^{\prime},T}(Z_0,\dots,Z_T)} \nonumber\\
&= \mathbb{E}_{Z\sim p_{M,T}} \log \frac{p_{M,T}(Z_T|Z_{T-1,-})}{p_{M^{\prime},T}(Z_T|Z_{T-1,-})} + \mathbb{E}_{Z\sim p_{M,T}}\log \frac{p_{M,T-1}(Z_0,\dots,Z_{T-1})}{p_{M^{\prime},T-1}(Z_0,\dots,Z_{T-1})} \nonumber \\
&= \mathbb{E}_{Z\sim p_{M,T}} \log \frac{p_{M,T}(Z_T|Z_{T-1,-})}{p_{M^{\prime},T}(Z_T|Z_{T-1,-})} + \kl{p_{M,T-1}}{p_{M^{\prime},T-1}} \nonumber \\
&= \mathbb{E}_{Z\sim p_{M,T}} \log \frac{p_{M,T}(Z_T|Z_{T-1})}{p_{M^{\prime},T}(Z_T|Z_{T-1})} + \kl{p_{M,T-1}}{p_{M^{\prime},T-1}} \label{eq:KL_recursion}
\end{align}
The first 3 steps above follow from the definition of KL divergence and conditional density. In the last step we have used the Markov property of the sequence $Z_0,\dots,Z_T$ which in this case shows that the law of $Z_T| Z_{T-1}$ is the same as the law $Z_T| Z_{T-1,-}$. Using Equation~\eqref{eq:KL_recursion} recursively and noting that $(Z_t,Z_{t-1})$ are identically distributed for every $t \in \{1,\dots,T\}$, we conclude:
 \begin{equation}\label{eq:kl_identity_full_chain}
\kl{p_{M,T}}{p_{M^{\prime},T}} = T\mathbb{E}_{(Z_0,Z_1)\sim p_{M,1}} \log \frac{p_{M,1}(Z_1|Z_{0})}{p_{M^{\prime},1}(Z_1|Z_{0})} + \kl{\pi_M}{\pi_{M^{\prime}}}
\end{equation}

We will first bound $\mathbb{E}_{(Z_0,Z_1)\sim p_{M,1}} \log \frac{p_{M,1}(Z_1|Z_{0})}{p_{M^{\prime},1}(Z_1|Z_{0})}$. Conditioned on $Z_0$, the law of $Z_1$ under the model $M$ is $\nn(A(M)Z_0, \sigma^2 I)$. Similarly, the conditional law of $Z_1$ under the model $M^{\prime}$ is $\nn(A(M^{\prime})Z_0, \sigma^2 I)$. Therefore, a simple calculation shows that:

\begin{align}
\mathbb{E}_{(Z_0,Z_1)\sim p_{M,1}} \log \frac{p_{M,1}(Z_1|Z_{0})}{p_{M^{\prime},1}(Z_1|Z_{0})} &= \mathbb{E}_{Z_0 \sim \pi_{M}}\frac{\|\left(A(M)-A(M^{\prime})\right)Z_0\|^2}{2\sigma^2} \nonumber \\
&= \mathbb{E}_{Z_0 \sim \pi_{M}} \tr\left(\left(A(M)-A(M^{\prime})\right)^{\top}\left(A(M)-A(M^{\prime})\right) \frac{Z_0 Z_0^{\top}}{2\sigma^2}\right) \nonumber \\
&=  \frac{1}{2\sigma^2}\tr\left(\left(A(M)-A(M^{\prime})\right)^{\top}\left(A(M)-A(M^{\prime})\right) G_M\right) \nonumber \\
&\leq \frac{3}{2}\tr\left(\left(A(M)-A(M^{\prime})\right)^{\top}\left(A(M)-A(M^{\prime})\right)\right) \nonumber \\
&= \frac{3}{2}\|A(M)-A(M^{\prime})\|_{\mathsf{F}}^2 = 3 \epsilon^2. \label{eq:time_dependent_part_kl}
\end{align}
In the first step, we have used standard KL formula for Gaussians with different mean but same variance. In the third step we have used the fact that $Z_0\sim \pi_{M}$. In the fourth step, we have used the upper bound on $G_M$ from Lemma~\ref{lem:covar_ineq_lb}. In the last step we have used the definition of $A(M)$ and the fact that the Hamming distance between $M$ and $M^{\prime}$ is $1$. Now we consider: $\kl{\pi_M}{\pi_{M^{\prime}}}$

Clearly, $\pi_M = \nn(0, G_M)$. By standard formula for KL divergence between Gaussians, 
\begin{equation}\label{eq:kl_identity_gaussian}
\kl{\pi_M}{\pi_{M^{\prime}}} =  \frac{1}{2}\left[\tr(G_{M^{\prime}}^{-1}G_M) - d + \log \frac{\mathsf{det}G_{M^{\prime}}}{\mathsf{det}G_M}\right]. 
\end{equation}
First we consider $\tr(G_{M^{\prime}}^{-1}G_M)$. Clearly, $G_M = \sigma^2 (I - A(M)^2)^{-1}$ and $G_{M^{\prime}} = \sigma^2 (I - A(M^{\prime})^2)^{-1}$. Therefore, $ G_{M^{\prime}}^{-1} = G_{M}^{-1} + \frac{A(M)^2 - A(M^{\prime})^2}{\sigma^2}$. We have:
\begin{align}
\tr(G_{M^{\prime}}^{-1}G_M) &= \tr(I) + \tr\left(\frac{A(M)^2 - A(M^{\prime})^2}{\sigma^2} G_{M}\right) \leq d + d \bigr\|\tfrac{A(M)^2 - A(M^{\prime})^2}{\sigma^2} G_{M}\bigr\| \nonumber \\
&\leq d + d \frac{\|G_M\|}{\sigma^2} \|A(M)^2 - A(M^{\prime})^2\| \leq d + 3d \|A(M)^2 - A(M^{\prime})^2\| \nonumber \\
&= d + 3d\| (A(M)-A(M^{\prime}))A(M) + A(M^{\prime})(A(M)-A(M^{\prime}))\| \nonumber \\
&\leq d + 3d\left[\| A(M)-A(M^{\prime})\|\|A(M)\| + \|A(M^{\prime})\|\|A(M)-A(M^{\prime}\|\right] \nonumber \\
&\leq d + \frac{9}{2}d\epsilon. \label{eq:trace_part_kl}
\end{align}
In the second step we have used the fact that $tr(B) \leq d\|B\|$. In the future steps, we have made use of the sub-multiplicativity of the operator norm and the upper bound on $\|G_M\|$ given by Lemma~\ref{lem:covar_ineq_lb}. We have also used the fact that by Gershgorin theorem $\|A(M)\| \leq \frac{3}{4}$ and $\|A(M)- A(M^{\prime})\| = \epsilon$. 

Next, we will bound $\log \frac{\mathsf{det}G_{M^{\prime}}}{\mathsf{det}G_M}$. Suppose $\mu_1 \geq \dots \geq \mu_d$ be the eigenvalues of $A(M)$ and $\mu^{\prime}_1 \geq \dots \geq \mu^{\prime}_d$ be the eigenvalues of $A(M^{\prime})$. We conclude that:
$$\log \frac{\mathsf{det}G_{M^{\prime}}}{\mathsf{det}G_M} = \sum_{i=1}^{d} \log \left(\frac{1-\mu_i^2}{ 1- (\mu_i^{\prime})^2}\right).$$
Now, $\|A(M) - A(M^{\prime})\| \leq \epsilon$. Therefore, we conclude by Weyl inequalities that $|\mu_i - \mu^{\prime}_i| \leq \epsilon$. By Gershgorin circle theorem, we also conclude that $\frac{1}{4} \leq \mu^{\prime}_i \leq \frac{3}{4}$

Plugging this into the equation above, we have: 

\begin{align}
\log \frac{\mathsf{det}G_{M^{\prime}}}{\mathsf{det}G_M} &= \sum_{i=1}^{d} \log \left(\frac{1-\mu_i^2}{ 1- (\mu_i^{\prime})^2}\right) \leq \sum_{i=1}^{d} \log \left(\frac{1-(\mu^{\prime}_i-\epsilon)^2}{ 1- (\mu_i^{\prime})^2}\right) = \sum_{i=1}^{d} \log \left(1 + \frac{2\mu_i^{\prime} -\epsilon^2}{1-(\mu_i^{\prime})^{2}}\right)\nonumber \\
&\leq \sum_{i=1}^{d} \log \left(1 +4\epsilon\right)\leq 4\epsilon d \label{eq:determinant_part_kl}
\end{align}
Combining Equations~\eqref{eq:trace_part_kl} and~\eqref{eq:determinant_part_kl} along with Equation~\eqref{eq:kl_identity_gaussian} we conclude:
$$\kl{\pi_M}{\pi_{M^{\prime}}}  \leq 5\epsilon d.$$
Using this along with Equations~\eqref{eq:time_dependent_part_kl} and~\eqref{eq:kl_identity_full_chain}, we conclude:
\begin{equation}
\kl{p_{M,T}}{p_{M^{\prime},T}} = 3\epsilon^2T + 5\epsilon d. 
\end{equation}
From this we conclude that when $\epsilon$ is as given in the statement of the lemma, we have:
\begin{equation}
\kl{p_{M,T}}{p_{M^{\prime},T}} \leq \frac{1}{8}. 
\end{equation}
By Pinsker's inequality, which states that $\mathsf{TV} \leq \sqrt{2 \mathsf{KL}}$, we conclude the result of the lemma. 
\end{proof}

\begin{proof}[Theorem~\ref{thm:main_lower_bound}]

We first note that when we choose $\sigma^2$ such that $d\sigma^2 = \beta$, we have $$\var(A(M),\mathcal{N}(0,\sigma^2I)) \in \mathcal{M}$$ for every $M \in \{0,1\}^{d(d-1)/2}$. We pick $\epsilon = c \min(\frac{1}{\sqrt{T}},\frac{1}{d})$ so that Lemma~\ref{lem:KL_inequality} is satisfied.

 We draw $M$ randomly from the uniform measure over $\{0,1\}^{d(d-1)/2}$ and lower bound the minimax error by Bayesian error. 

\begin{equation}\label{eq:bayesian_lb}
\lossmm(\mathcal{M}) \geq \inf_{f\in \mathcal{F}}\mathbb{E}_{M} \mathbb{E}_{(Z_t)\sim M}\losspred(f(Z_0,\dots,Z_T);M) - \losspred(A(M);M)
\end{equation}

We will now uniformly lower bound $\mathbb{E}_{M} \mathbb{E}_{(Z_t)\sim M}\losspred(f(Z_0,\dots,Z_T);M) - \losspred(A(M);M)$ for every fixed choice of $f\in \mathcal{F}$ to conclude the statement of the theorem from Equation~\eqref{eq:bayesian_lb}. Henceforth, we will denote $f(Z_0,\dots,Z_T)$ by $\ahat(M)$ whenever $(Z_t) \sim M$. 
By Lemma~\ref{lem:loss_identity}, we conclude that:
$$\losspred(\ahat(M);M)- \losspred(A(M);M)= \tr\left[(\ahat(M)-A(M))^{\top}(\ahat(M)-A(M))G_M\right].$$
$(\ahat(M)-A(M))^{\top}(\ahat(M)-A(M))$ is a PSD matrix and by Lemma~\ref{lem:covar_ineq_lb}, $G_M \geq \sigma^2I$ for every $M$. Therefore, we conclude that with probability $1$ we have:

\begin{align}
\losspred(\ahat(M);M)- &\losspred(A(M);M)\geq \sigma^2\tr\left[(\ahat(M)-A(M))^{\top}(\ahat(M)-A(M))\right] \nonumber \\
&= \sigma^2\|\ahat(M) - A(M)\|^2_{\mathsf{F}} \geq 2\sigma^2\sum_{\substack{i,j \in [d]\\ i < j}} (\ahat(M)_{ij}-A(M)_{ij})^2. \label{eq:intermediate_lb_ineq_1}
\end{align}
Therefore, we conclude that:
\begin{equation}\label{eq:intermediate_lb_ineq_2}
\mathbb{E}_M \mathbb{E}_{Z_t\sim M}\losspred(\ahat(M);M)- \losspred(A(M);M) \geq 2\sum_{\substack{i,j \in [d]\\ i < j}} \mathbb{E}_{M}\mathbb{E}_{(Z_t)\sim M}(\ahat(M)_{ij}-A(M)_{ij})^2. 
\end{equation}
We will now lower bound every term in the summation in the RHS of Equation~\eqref{eq:intermediate_lb_ineq_2}. Fix $(i,j)$. Let $M_{\sim ij}$ denote all the co-ordinates of $M$ other than $(i,j)$. We define $M^{+},M^{-} \in \{0,1\}^{d(d-1)/2}$ so that $M^{+}_{\sim ij} = M_{\sim ij}$ and $M^{+}_{ij} = 1$. Similarly, let $M^{-}_{\sim ij} = M_{\sim ij}$ and $M^{-}_{ij} = 0$. Therefore, we have:

\begin{align}
 \mathbb{E}_{M}\mathbb{E}_{(Z_t)\sim M}(\ahat(M)_{ij}-A(M)_{ij})^2 &= \frac{1}{2}\mathbb{E}_{M_{\sim ij}} \mathbb{E}_{(Z_t)\sim M^+}(\ahat(M^+)_{ij}-A(M^+)_{ij})^2 \nonumber \\ &\quad + \frac{1}{2}\mathbb{E}_{M_{\sim ij}}\mathbb{E}_{(Z_t)\sim M^-}(\ahat(M^-)_{ij}-A(M^-)_{ij})^2. \label{eq:intermediate_lb_ineq_3}
\end{align}
Now, $M^+$ and $M^-$ differ in exactly one co-ordinate. We invoke Lemma~\ref{lem:KL_inequality} to show that there exists a coupling between $(Z_t^{+})\sim M^+$ and $Z_t^{-} \sim M^-$ such that $\mathbb{P}(Z_t^{+} = Z_t^{-}) \geq \frac{1}{2}$. Call this event $\Gamma$ (we ignore the dependence on $M_{\sim ij}$ for the sake of clarity). In this event, we must have $\ahat(M^{+}) = \ahat(M^-)$ since our estimator $f \in \mathcal{F}$ is a measurable function of the data. For any fixed $M_{\sim ij}$, we have:
\begin{align}
 &\mathbb{E}_{(Z_t)\sim M^+}(\ahat(M^+)_{ij}-A(M^+)_{ij})^2 +\mathbb{E}_{(Z_t)\sim M^-}(\ahat(M^-)_{ij}-A(M^-)_{ij})^2 \nonumber\\ 
&\geq  \mathbb{E}_{(Z_t)}\mathbbm{1}(\Gamma)\left[(\ahat(M^+)_{ij}-A(M^+)_{ij})^2 + (\ahat(M^+)_{ij}-A(M^-)_{ij})^2\right] \nonumber\\
 &\geq \mathbb{P}(\Gamma)(A(M^-)_{ij}-A(M^+)_{ij})^2 \geq \frac{1}{2}(A(M^-)_{ij}-A(M^+)_{ij})^2 = \frac{\epsilon^2}{2}. \label{eq:intermediate_lb_ineq_4}
\end{align} 
In the second line we have used the fact that under event $\Gamma$, $\ahat(M^+) = \ahat(M^{-})$. In the third line, we have used the inequality $(x-y)^2 + (x-z)^2 \geq \frac{1}{2} (y-z)^2$. In the fourth line, we have used the fact that $\mathbb{P}(\Gamma) \geq 1/2$. Using Equation~\eqref{eq:intermediate_lb_ineq_4} along with Equations~\eqref{eq:intermediate_lb_ineq_3} and~\eqref{eq:intermediate_lb_ineq_2}, we conclude that for every estimator $f \in \mathcal{F}$ the following holds:
$$\mathbb{E}_M \mathbb{E}_{Z_t\sim M}[\losspred(\ahat(M);M)- \losspred(A(M);M) ]\geq \frac{d(d-1)\epsilon^2 \sigma^2}{4}.$$ 
Using above equation with Equation~\eqref{eq:bayesian_lb}, we conclude the statement of the theorem. 
\end{proof}

\begin{remark}
We can show a similar lower bound by considering a discrete prior over the space of orthogonal matrices. In particular taking $\A$ to be an orthogonal matrix scaled by $\rho$, we can endow the orthogonal (or special orthogonal) group with metric induced by the Frobenius norm. Then from \citep[Proposition 7]{szarek1982nets}, we can construct an $\epsilon$-cover of cardinality $d^{\frac{d(d-1)}{2}}$. But then from the proof of \citep[Proposition 3]{cai2013sparse}, for $\alpha\in(0,1)$, there exists a \emph{local packing} of the space with packing distance $\alpha \epsilon$ and cardinality at least $c^{d(d-1)/2}$ where $c>1$. Further the diameter of this local packing is at most $2\epsilon$ (in Frobenius norm). Now using standard arguments from Fano's inequality (c.f.\citep[Proposition 3]{cai2013sparse}) or Birge's inequality (c.f.\citep[Lemma F.1]{simchowitz2018learning}) we can get a similar lower bound on the prediction error as Theorem~\ref{thm:main_lower_bound} but with explicit dependence on $\rho$. 
\end{remark}
\section{Techincal Proofs}

\label{sec:technical_proofs}

\subsection{Proof of Lemma~\ref{lem:bounded_iterates}}
\begin{proof}
Consider the $\sgdber$ iteration: 
\begin{align}
\At{t-1}{i+1} &=  \At{t-1}{i} - 2 \gamma (\At{t-1}{i}\Xt{t-1}{-i}-\Xt{t-1}{-(i+1)})\Xt{t-1,\top}{-i} \nonumber \\
&= \At{t-1}{i} (I-2\gamma\Xt{t-1}{-i}\Xt{t-1,\top}{-i} ) + 2\gamma \Xt{t-1}{-(i-1))}\Xt{t-1,\top}{-(i+1)}  
\end{align}
Observe that for our choice of $\gamma$ and under the event $\cd^{0,N-1}$, we have $\|(I-2\gamma\Xt{t-1}{-i}\Xt{t-1,\top}{-i} ) \| \leq 1$ and $\|\Xt{t-1}{-(i+1)}\Xt{t-1,\top}{-i} \| \leq R$. Therefore, triangle inequality implies:
$$\|\At{t-1}{i+1}\| \leq \|\At{t-1}{i} \|+ 2\gamma R$$
We conclude the bound in the Lemma.

\end{proof}

\subsection{Proof of Lemma~\ref{lem:coupled_iterate_replacement}}

\begin{proof}
We again consider the evolution equation: $\Xtt{t-1}{-i}$
\begin{align}
\At{t-1}{i+1} &= \At{t-1}{i} - 2 \gamma (\At{t-1}{i}\Xt{t-1}{-i}-\Xt{t-1}{-(i+1)})\Xttr{t-1}{-i} \nonumber \\
&= \At{t-1}{i} - 2\gamma(\At{t-1}{i}\Xtt{t-1}{-i}-\Xtt{t-1}{-(i+1)})\Xtttr{t-1}{-i} + \Delta_{t,i}
\end{align}
Where  
$$\Delta_{t,i} = 2\gamma\At{t-1}{i}\left(\Xtt{t-1}{-i} \Xtttr{t-1}{-i}  - \Xt{t-1}{-i} \Xttr{t-1}{-i} \right) + 2\gamma\left(\Xt{t-1}{-(i+1)}\Xttr{t-1}{-i} - \Xtt{t-1}{-(i+1)}\Xtttr{t-1}{-i}\right) $$
 Using Lemmas~\ref{lem:bounded_iterates} and~\ref{lem:1}, we conclude that:
$$\|\Delta_{t,i}\| \leq (16\gamma^2R^2T + 8\gamma R)\norm{\A^u} $$
Using the recursion for $\tilde{A}_{i}^{t}$, we conclude:
\begin{align}
\At{t-1}{i+1} -\Att{t-1}{i+1}  &= (\At{t-1}{i}  - \Att{t-1}{i})\Ppt{t}{i} + \Delta_{t,i} \nonumber \\
\implies \norm{\At{t-1}{i+1} -\Att{t-1}{i+1}} &\leq \norm{\At{t-1}{i}  - \Att{t-1}{i}}\norm{\Ppt{t}{i}} + (16\gamma^2R^2T + 8\gamma R)\norm{\A^u}  \nonumber \\
\implies \norm{\At{t-1}{i+1} -\Att{t-1}{i+1}} &\leq \norm{\At{t-1}{i}  - \Att{t-1}{i}} +  (16\gamma^2R^2T + 8\gamma R)\norm{\A^u} \label{eq:coupling_distance_recursion}
\end{align}
In the last step we have used the fact that under the event $\cdh^{0,N-1}$, we must have $\norm{\Ppt{t}{i}}\leq 1$. We conclude the statement of the lemma from Equation~\eqref{eq:coupling_distance_recursion}. 
\end{proof}

\subsection{Proof of Lemma~\ref{lem:coupling_AA^t}}

\begin{proof}
First we have
\begin{\Ieee}{LLL}
\label{eq:coupling_AA^t_1}
\Ex{\gram{\left(\At{t-1}{j}-\A\right)}\ind{0}{t-1}}&\preceq & \Ex{\gram{\left(\At{t-1}{j}-\A\right)}\indh{0}{t-1}}\\
&&+4\gamma^2 (Bt)^2 R\sqrt{\mu_4}\prbndsq I\\
&\preceq & \Ex{\gram{\left(\At{t-1}{j}-\A\right)}\indh{0}{t-1}}\\
&&+ c\gamma^2 d\sigma_{\max}(\Sigma) R T^2\prbndsq I\Ieeen
\end{\Ieee}

Next, we have 
\begin{\Ieee}{LLL}
\label{eq:coupling_AA^t_2}
\norm{\gram{\left(\At{t-1}{j}-\A\right)}-\gram{\left(\Att{t-1}{j}-\A\right)} }\\
\leq \norm{\At{t-1}{j}-\Att{t-1}{j}} \left(\norm{\left(\At{t-1}{j}-\A\right)}+\norm{\left(\Att{t-1}{j}-\A\right)}\right)\\
\leq \norm{\At{t-1}{j}-\Att{t-1}{j}} \left(2\norm{\A}+\norm{\At{t-1}{j}}+\norm{\Att{t-1}{j}}\right)\Ieeen
\end{\Ieee}
Thus on the event $\cdh^{0,t-1}$, using lemma~\ref{lem:coupled_iterate_replacement} and lemma~\ref{lem:bounded_iterates} we get
\begin{\Ieee}{LLL}
\label{eq:coupling_AA^t_3}
\norm{\gram{\left(\At{t-1}{j}-\A\right)}-\gram{\left(\Att{t-1}{j}-\A\right)} }\\
\leq c(\gamma^2 R^2 T^2 +\gamma R T)(\gamma R T+\norm{\A}+\norm{A_0})\norm{\A^u}
\leq c\gamma^3 R^3 T^3\norm{\A^u}\Ieeen
\end{\Ieee}
for some constant $c$. (We have suppressed the dependence on $A_0$ and $\A$ since they are constants and $\gamma R T$ grows with $T$).

The proof follows by combining \eqref{eq:coupling_AA^t_1} and \eqref{eq:coupling_AA^t_3}. 

The proof of \eqref{eq:coupling_AA^t_b} follows similarly. 
\end{proof}
\section{Prediction error for sparse systems}
\label{sec:sparse_system}
In this section we consider the $\var(\A,\mu)$ model with sparse $\A$ whose sparsity pattern is known. We will present a modification of $\sgdber$ that takes into account the sparsity pattern information. Formally, let $S_l=\{k:\Aa_{l,k}\neq 0\}$ be support or sparsity pattern of row $l$ of $\A$. Further let $s_l=|S_l|$ denote the sparsity of row $j$. We assume that $S_l$ is known for each $1\leq l\leq d$. The claim is that the excess expected prediction loss is of order $\frac{\sum_l s_l\sigma_l^2}{T}$. We will present only a sketch of the proof highlighting the main steps. Detailed calculations follow similarly as in sections~\ref{sec:bias_var_analysis} and \ref{sec:pred_loss}.\\

The modification of the $\sgdber$ algorithm to use the sparsity pattern is as follows. Let $\atr_l$ denote row $l$ of $\A$. The algorithmic iterates are given by $(A^{t-1}_j)$ where row $l$ is $a^{t-1,\top}_{j,l}$. Let $a^0_{0,l}=0\in\mathbb{R}^d$. Let $\{e_l:1\leq l\leq d\}$ denote the standard basis of $\mathbb{R}^d$.  Let $P_{S_l}:\mathbb{R}^d\to \mathbb{R}^d$ denote the (self adjoint) orthogonal projection operator onto the subspace spanned by $\{e_l:l\in S_l\}$. Then update for row $l$ is given by 
\begin{equation}
\label{eq:sparse_sgd1}
a^{t-1,\top}_{j+1,l}=\left[a^{t-1,\top}_{j,l}-2\gamma(a^{t-1,\top}_{j,l}\Xt{t-1}{-j}-\langle e_l, X^{t-1}_{-(j-1)}\rangle)\Xttr{t-1}{-j}\right]P_{S_l}
\end{equation}
and $a^{t}_{0,l}=a^{t-1}_{B,l}$. Since each iterate above has sparsity pattern $S_l$ by construction, we can rewrite the above as
\begin{equation}
\label{eq:sparse_sgd2}
a^{t-1,\top}_{j+1,l}=a^{t-1,\top}_{j,l}-2\gamma(a^{t-1,\top}_{j,l}\Xt{t-1}{-j}-\langle e_l, X^{t-1}_{-(j-1)}\rangle)\left(P_{S_l}\Xt{t-1}{-j}\right)^{\top}
\end{equation}

Notice that $ a^{t-1,\top}_{j,l}\Xt{t-1}{-j}=a^{t-1,\top}_{j,l}P_{S_l}\Xt{t-1}{-j}$ and 
$$\langle e_l, X^{t-1}_{-(j-1)}\rangle =\atr_l \Xt{t-1}{-j}+\eta^{t-1}_{-j,l} $$
 Thus
\begin{\Ieee}{LLL}
\label{eq:sparse_sgd3}
\left(a^{t-1}_{j+1,l}-\a_l\right)^{\top}=\left(a^{t-1}_{j,l}-\a_l\right)^{\top}\left(P_{S_l}-2\gamma\left(P_{S_l}\Xt{t-1}{-j}\right)\left(P_{S_l}\Xt{t-1}{-j}\right)^{\top}\right)+2\gamma \eta^{t-1}_{-j,l}\left(P_{S_l}\Xt{t-1}{-j}\right)^{\top}\\
\Ieeen
\end{\Ieee}

For a vector $v\in \mathbb{R}^d$, let $v_{S_l}\in\mathbb{R}^{s_l}$ be the vector corresponding to the support $S_l$ i.e. entries in $v_{S_l}$ correspond to the entries in $v$ whose indices are in $S_l$. So we can rewrite \eqref{eq:sparse_sgd3} completely in $\mathbb{R}^{s_l}$ as
\begin{\Ieee}{LLL}
\label{eq:sparse_sgd4}
\left(a^{t-1}_{j+1,l}-\a_l\right)^{\top}_{S_l}=\left(a^{t-1}_{j,l}-\a_l\right)^{\top}_{S_l}\left(I_{s_l}-2\gamma\left(\Xt{t-1}{-j}\right)_{S_l}\left(\Xt{t-1}{-j}\right)^{\top}_{S_l}\right)+2\gamma \eta^{t-1}_{-j,l}\left(\Xt{t-1}{-j}\right)^{\top}_{S_l}\\
\Ieeen
\end{\Ieee}
where $I_{s_l}$ is the identity matrix of dimension $s_l$. 

Our goal is to bound the expected prediction error for this modified $\sgdber$. To that end, we will make some important observations. 
\begin{enumerate}[label=(\arabic*)]
\item Since we focus on prediction error, the entire analysis can be carried out row by row. To see this, if $\hat A$ is any estimator, the 
$$\losspred(\hat A;\A,\mu)-\tr(\Sigma)=\tr(G(\hat A-\A)^{\top}(\hat A-A))=\sum_{l=1}^d \tr(G(\hat a_l-\a_l)(\hat a_l-\a_l)^{\top})$$
where $\hat a^{\top}_l$ is the row $l$ of $\hat A$. 
\item If $\hat a_l$ and $\a_l$ have sparsity pattern $S_l$ then
\begin{\Ieee}{LLL}
\tr(G(\hat a_l-\a_l)(\hat a_l-\a_l)^{\top})&=&\tr(P_{S_l}GP_{S_l}(\hat a_l-\a_l)(\hat a_l-\a_l)^{\top})\\
&=&\tr(G_{S_l}(\hat a_l-\a_l)_{S_l}(\hat a_l-\a_l)^{\top}_{S_l})
\end{\Ieee}
where $G_{S_l}\in \mathbb{R}^{s_l\times s_l}$ is the submatrix of $G$ obtained by picking rows and columns corresponding to indices in $S_l$. 
\item Under the stationary measure, we have $\Ex{\left(P_{S_l}\Xt{t-1}{-j}\right)\left(P_{S_l}\Xt{t-1}{-j}\right)^{\top}}=P_{S_l}GP_{S_l}$. Thus, with high probability $\norm{P_{S_l}\Xt{t-1}{-j}}^2\leq c s_l\sigma_{\max}(G)\log T$. 

\item Letting $s_0=\max_{l}s_l$, we can set $R=c s_0\sigma_{\max}(G)\log T$ and use step size $\gamma=O(1/RB)$.

\item We can perform the same bias-variance decomposition as described in section~\ref{sec:bias_variance} to obtain $a^{t-1,v}_{B,l}$ and $a^{t-1,b}_{B,l}$. 

\item From previous observations, the variance of last iterate corresponding to row $l$ turns out to be
$$\gamma\sigma_l^2(1-o(1))I_{s_l}\preceq \Ex{\left(a^{t-1,v}_{B,l}\right)_{S_l}\left(a^{t-1,v}_{B,l}\right)^{\top}_{S_l}}\preceq \frac{\gamma}{1-\gamma R} \sigma_l^2(1+o(1))I_{s_l} $$
where $\sigma_l^2=\Sigma_{l,l}$. 

\item Similarly, the variance of the average iterate $\Ex{(\hat a^{v}_{0,N,l})(\hat a^{v}_{0,N,l})^{\top}}$ corresponding to row $l$ can be bounded upto leading order by 
$$\frac{1}{N^2}\sum_{t=1}^N \left[V_{t-1,l}(I_{s_l}-\ch_{S_l})^{-1}+(I_{s_l}-\ch^{\top}_{S_l})^{-1}V_{t-1,l}\right] $$
where $V_{t-1,l}=\Ex{\left(a^{t-1,v}_{B,l}\right)_{S_l}\left(a^{t-1,v}_{B,l}\right)^{\top}_{S_l}}$ and (with abuse of notation) $\ch_{S_l}$ is defined as
$$\ch_{S_l}=\Ex{\prod_{j=0}^{B-1}\left(I_{s_l}-2\gamma (\Xtt{0}{-j})_{S_l} (\Xtt{0}{-j})_{S_l}^{\top}\right)1\left[\cap_{j=0}^{B-1}\left\{\norm{(\Xtt{0}{-j})_{S_l}}^2\leq R\right\}\right]}$$ 
where $\Xtt{0}{0}\distas{}\pi$.

\item  Now, similar to lemma~\ref{lem:H_plus_HT_bound} we can bound $\ch_{S_l}+\ch_{S_l}^{\top}$ by $2(I_{s_l}-c\gamma B G_{s_l})$ upto leading order. 

\item Thus similar to lemma~\ref{lem:predloss_iden_vari} we obtain 
$$\tr(G_{S_l}(I-\ch_{S_l})^{-1})\leq c\frac{s_l}{\gamma B}$$

\item Finally as in section~\ref{sec:pred_var} we can bound the variance of prediction error of row $l$ upto leading order by
$$\tr(G\Ex{(\hat a^{v}_{0,N,l})(\hat a^{v}_{0,N,l})^{\top}})\lesssim \frac{\sigma_l^2 s_l}{T}$$
Thus summing over $l$ we get
$$\tr\left(G\Ex{(\hat A^{v}_{0,N})(\hat A^{v}_{0,N})^{\top}}\right)\lesssim \frac{\sum_l \sigma_l^2 s_l}{T}$$

\item Bias can also be analyzed in a similar way and it will be of strictly lower order (using suitable tail-averaging). 

\item Thus the excess prediction loss  is given bounded as
$$ \Ex{\losspred(\hat A_{N/2,N};\A,\mu)}-\tr(\Sigma)\lesssim \frac{\sum_l \sigma_l^2 s_l}{T} $$
\end{enumerate}

So the modified $\sgdber$ algorithm effectively utilizes the low dimensional structure in $\A$.





\end{document}